\documentclass[11pt,letterpaper]{article}
\usepackage[margin=1in]{geometry}
\usepackage{times}
\usepackage{natbib}
\setcitestyle{authoryear,round,citesep={;},aysep={,},yysep={;}}
\usepackage[T1]{fontenc}
\usepackage[utf8]{inputenc}
\usepackage{amsmath,amssymb,amsthm}
\usepackage{booktabs,longtable,array,calc}
\usepackage{needspace}
\usepackage{graphicx}
\usepackage{colortbl}
\usepackage{arydshln}
\usepackage{placeins}
\usepackage{tikz}
\usetikzlibrary{arrows.meta}
\usepackage{algorithm}
\usepackage{algpseudocode}
\usepackage{fancyvrb}
\usepackage{xurl}
\usepackage[hypertexnames=false,hidelinks]{hyperref}

\makeatletter
\@ifpackagelater{longtable}{2024/12/17}{%
  \patchcmd{\LT@start}{\ifvoid\LT@foot}{\MakeLinkTarget{table}\ifvoid\LT@foot}
    {}{\PackageError{k2p}{Longtable hyperlink patch failed}{Check package versions.}}%
}{}
\makeatother
\usepackage{microtype}
\providecommand{\tightlist}{\setlength{\itemsep}{0pt}\setlength{\parskip}{0pt}}

\DefineVerbatimEnvironment{verbatim}{Verbatim}{fontsize=\scriptsize}
\newtheorem{proposition}{Proposition}
\newtheorem{theorem}{Theorem}
\newtheorem{corollary}{Corollary}
\title{K2P: Label-Free Knowledge\\to Prompt Distillation}
\author{\normalsize
Yingchuan Zhang\setcounter{footnote}{1}\thanks{Equal contribution.}\quad
Haoran Lu\textsuperscript{\textdagger}\quad
Wenxuan Zhong\quad
Ping Ma\setcounter{footnote}{0}\thanks{Corresponding author: \href{mailto:pingma@uga.edu}{pingma@uga.edu}.}\\[0.5em]
\small Department of Statistics\\
\small University of Georgia\\
\small Athens, GA 30602}
\date{}
\hypersetup{pdftitle={K2P: Label-Free Knowledge to Prompt Distillation},
  pdfauthor={Yingchuan Zhang, Haoran Lu, Wenxuan Zhong, Ping Ma}}
\begin{document}
\maketitle

\begin{abstract}
Knowledge distillation can transfer reasoning from stronger teachers to
frozen students through reusable prompts, but avoiding weight updates
does not eliminate supervision. Without ground-truth answers, teacher
solutions are unverified, and agreement with the teacher can reward shared mistakes.
We introduce Knowledge-to-Prompt (K2P) for label-free knowledge distillation to
prompts. K2P synthesizes reusable instructions from teacher solutions, refines
them using paired teacher and student responses, and guides search and
selection with answer agreement. It retains candidates that adaptive
search may undervalue and selects on reserved questions. Deployment uses
only the frozen student and selected prompt. Our theory separates
generation and selection gaps and gives conditions under which
agreement-guided construction yields accuracy guarantees despite imperfect
teacher references. Across reasoning tasks and students, K2P outperforms label-free
alternatives overall and remains competitive with supervised prompt
optimization. Ablations and archive diagnostics assess the contributions of teacher
solutions and refinement, while revealing the limits of agreement-guided selection.
\end{abstract}

\section{Introduction}

Knowledge distillation transfers predictive and reasoning knowledge
from stronger teachers to smaller students through teacher-generated
supervision \citep{kd,seqkd,reasoningteachers,fang2024bayesian,fang2026knowledge}.
This knowledge can also be expressed as reusable instructions for a
frozen student \citep{pld,rdpo}. We call this direction
\emph{knowledge distillation to prompts}: it connects knowledge transfer
with automatic prompt optimization \citep{ape,gepa} to derive task
procedures that the student can reuse across inputs.

Realizing this potential requires both extracting useful procedures
and evaluating how well the student executes them. Ground-truth
supervision can support both operations. For example, Prompt-Level
Distillation (PLD; \citealp{pld}) conditions extraction on gold answers
and refines instructions using student correctness feedback
(Figure~\ref{fig:k2p-concept}(a)). However, questions and public context
may be available before their correct answers are established.
Obtaining those answers can require locating evidence in documents or
expert analysis of numerical reasoning, as illustrated by Natural
Questions and FinQA \citep{nq,finqa}. Avoiding weight updates therefore
does not by itself remove the supervision needed to construct and
evaluate effective instructions.

We study \emph{label-free knowledge distillation to prompts}: construction
uses questions and public context, without task-specific ground-truth
answers or correctness annotations. Output comparisons \citep{spo,pdo}
and reasoning consistency \citep{rdpo} offer model-based feedback, but leave a central challenge: how to turn
unverified teacher knowledge into useful student instructions and select
among them without ground-truth answers for evaluation.

Removing labels changes both generation and evaluation. First, a
teacher's generated solution is unverified, and even a useful procedure
may be difficult for the student to execute. Disagreement alone cannot
identify which model is wrong or what needs correction. Second, agreement
with the teacher can reward shared mistakes. A revision that corrects a
student error may even reduce agreement with an incorrect teacher.
Limited evaluation samples and repeated comparisons can further distort
candidate rankings \citep{cawley}. Effective distillation must therefore
address both candidate quality and the reliability of the signal used
to choose candidates. We focus on tasks with explicit final answers
that can be extracted and normalized.

We propose \emph{Knowledge-to-Prompt} (K2P), using teacher solutions as
procedural evidence and teacher answers as evaluation references to
address these challenges (Figure~\ref{fig:k2p-concept}(b)). The teacher
synthesizes instructions from source solutions and refines a search-selected
candidate using paired teacher and student responses. Normalized teacher
answers provide a shared agreement signal without labeling disagreements
as student errors. K2P separates the search parent from the archive to
preserve candidates this signal may undervalue. Only strict improvement
in search agreement changes the parent; all admissible, fully evaluated
candidates enter the archive, whether or not they replace the parent. Once the archive is fixed, agreement on
reserved questions withheld from synthesis and search selects the
deployment prompt. Deployment uses only the frozen student and selected prompt.

\begin{figure}[!t]
\centering
\includegraphics[width=\linewidth]{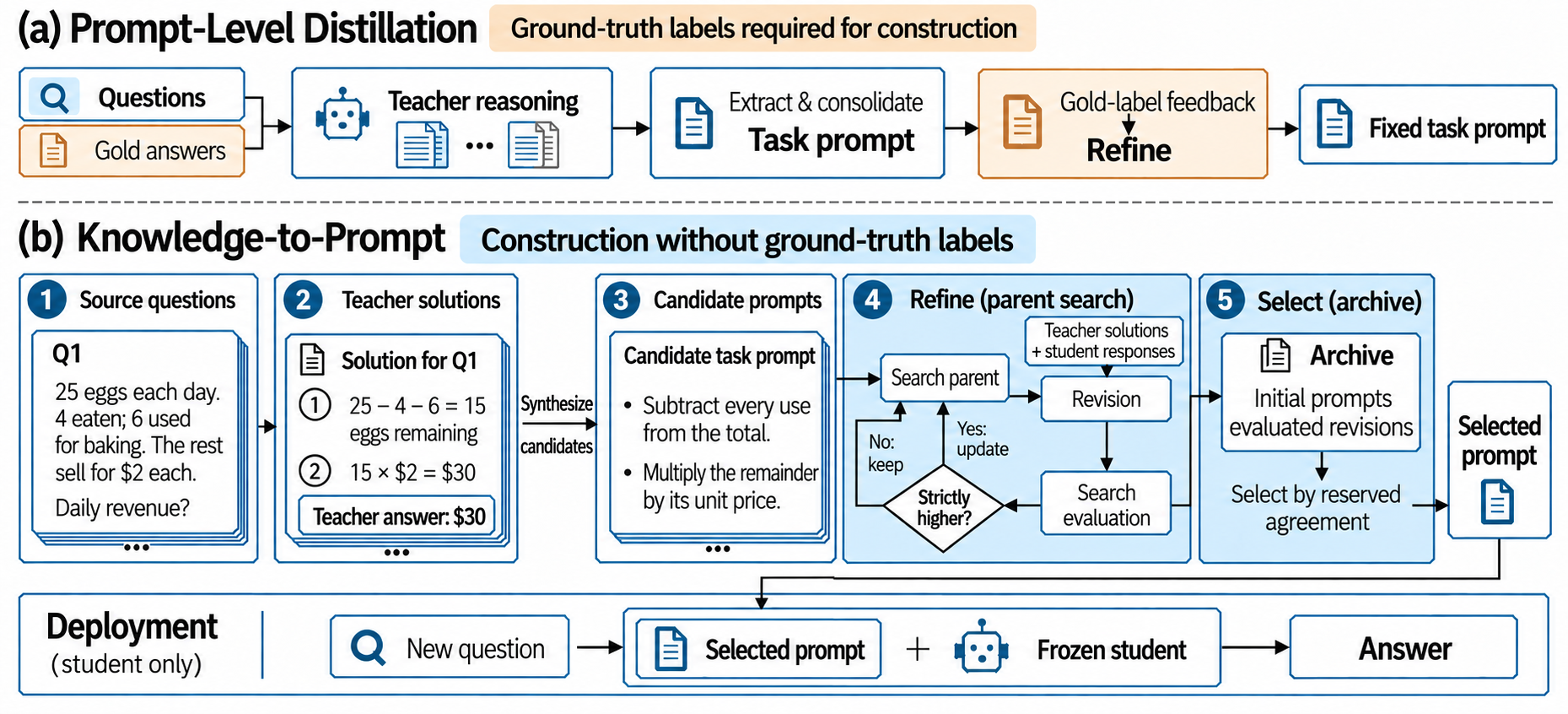}
\caption{Label-dependent and label-free prompt construction.
(a) PLD uses gold labels for extraction and correctness feedback.
(b) K2P synthesizes initial prompts from unlabeled source questions and teacher
solutions (1--3), then refines a search parent using paired teacher and
student responses (4). Only strict improvement in search agreement updates
the parent. All admissible, fully evaluated candidates enter the archive;
agreement on reserved questions selects from this fixed archive after
construction (5). Deployment uses only the selected prompt and frozen student.}
\label{fig:k2p-concept}
\end{figure}

These operations optimize agreement, whereas the deployment objective is accuracy.
Under aligned answer-matching rules, teacher errors affect prompt comparisons
only where candidate answers differ; teacher accuracy alone therefore
does not determine ranking reliability. We separate accuracy missing
from the archive from accuracy lost in selection. Under explicit progress
and sampling assumptions, we bound the deployed prompt's population-agreement
deficit, accounting for initialization, refinement, adaptive search, and
reserved selection. A condition relating agreement deficits to accuracy deficits
then yields a guarantee relative to the best admissible prompt for the same
student. When this relationship has no residual mismatch, the bound allows the
expected accuracy gap to vanish despite persistent teacher errors,
provided the construction and sampling terms also vanish.

Our evaluation covers four reasoning tasks, two students, and three
construction seeds under common construction cost ceilings for the
external baselines. Our contributions are:

\begin{enumerate}
\def\labelenumi{\arabic{enumi}.}
\tightlist
\item
  \textbf{Label-free construction from teacher solutions and answers.}
  K2P synthesizes and refines instructions from worked solutions, using
  shared teacher answers for search and reserved selection from a frozen
  archive, without ground-truth labels or student weight updates.
\item
  \textbf{A conditional link from construction to accuracy.}
  We give conditions under which the returned prompt's accuracy approaches the optimum for the frozen student within the admissible prompt space,
  accounting for imperfect references and finite evaluation data.
\item
  \textbf{Evidence on performance and its underlying mechanisms.}
  K2P improves mean accuracy over zero-shot in seven of eight settings.
  It exceeds teacher-reference GEPA (GEPA-T; \citealp{gepa}) in six of eight
  task--student settings, with an unweighted mean advantage of 2.88
  percentage points under shared teacher references. It exceeds
  teacher-reference PLD and SPO in all eight settings and averages
  0.48 points below supervised GEPA. Ablations examine worked solutions and compare
  refinement with independent synthesis; archive analyses distinguish
  candidate quality from final-selection losses.
\end{enumerate}

\section{Related Work}
\label{sec:related-work}

\textbf{Teacher knowledge and prompt optimization.}
PLD transfers gold-conditioned teacher reasoning using student
correctness feedback \citep{pld}; reasoning distillation by prompt
optimization uses judged reasoning consistency \citep{rdpo}.
GEPA combines reflection, instance-level Pareto selection, and validation
\citep{gepa}; our GEPA-T adaptation uses teacher-answer references.
K2P uses worked solutions for synthesis and student-informed refinement,
then selects from a frozen archive on reserved questions. Admissible,
fully evaluated revisions remain eligible even when they do not replace the search
parent, preserving candidates that search agreement may undervalue.
Appendix~\ref{app:extended-related-work} expands the connections to prompting,
distillation, and experience reuse.

\textbf{Label-free feedback.}
SPO judges paired outputs \citep{spo}; Prompt Duel Optimizer combines
pairwise judgments with dueling-bandit search \citep{pdo}. K2P matches
normalized student answers to cached teacher references. PLD-T and GEPA-T use the same reference answers, allowing their original
construction procedures to be compared under a shared evaluation signal.

\textbf{Imperfect references and finite-sample selection.}
Noisy-label learning relates corrupted feedback to true prediction risk
\citep{angluin,natarajan}. Here, teacher answers serve as references for evaluating student responses;
under aligned scoring, comparison distortion depends on teacher
errors where candidate answers differ. Finite-sample model selection can overfit
its criterion \citep{cawley,arlot}, and adaptive data reuse can bias later
choices \citep{adaptive,russozou}. K2P separates adaptive search from
reserved selection and analyzes search optimism, finite-sample selection,
and mismatch between agreement and accuracy.

\section{Problem Setup}
\label{sec:problem-setup}

Let $T$ be a teacher model and $S$ a target student, with both models'
parameters held fixed. An input $x$ belongs to the task-input space
$\mathcal X$ and contains a question and its public context, such as a
table or answer options. A task prompt $p$ is a natural-language
instruction reused across inputs. We seek a prompt in a nonempty
admissible set $\Pi$, defined by instruction and length constraints
fixed before construction data are observed.

\textbf{Construction and deployment.}
Construction receives an unlabeled pool
$\mathcal U=\{x_i\}_{i=1}^{n_{\mathcal U}}\subseteq\mathcal X$, where $n_{\mathcal U}$ is the
number of available inputs. It may query both models using these inputs
and public task specifications, retaining generated solutions, answers,
and student responses. Task-specific ground-truth answers and correctness
annotations are unavailable to construction. Deployment uses only the
student, the selected prompt, and the current input. Ground-truth
answers are used afterward to evaluate the constructed prompts.

\textbf{Responses and correctness.}
Let $\mathcal Y$ be the space of normalized task answers and
$\mathcal D$ the target distribution over pairs
$(x,y)\in\mathcal X\times\mathcal Y$, where $y$ is the ground-truth
answer. Write $S(x;p;\xi)$ for the student's response under prompt $p$,
where $\xi$ represents randomness under a fixed decoding policy.
Let $\mathcal Z$ be the space of model responses, including completion
status. The fixed map
$a:\mathcal Z\to\mathcal Y\cup\{\bot\}$ extracts and normalizes an
explicit final answer; $\bot\notin\mathcal Y$ denotes extraction failure.
The correctness indicator $c_p(x,y,\xi)\in\{0,1\}$ equals one when
$S(x;p;\xi)$ is correct against $y$ under the task's fixed matching and
output-validity rules, and zero otherwise.

The objective is to maximize the student's expected task accuracy:
\begin{equation}
\max_{p\in\Pi} J(p;S),
\qquad
J(p;S)
=\mathbb E_{(x,y)\sim\mathcal D,\,\xi}
\bigl[c_p(x,y,\xi)\bigr].
\tag{1}\label{eq:accuracy-objective}
\end{equation}
We subsequently abbreviate the correctness indicator as $c_p$ and
suppress decoding randomness in model calls. Population expectations
average over inputs and randomness from the relevant fixed decoding
policies.

Without ground-truth answers, construction cannot directly evaluate
$J$ or its labeled empirical counterpart. K2P instead evaluates prompts
through teacher--student answer agreement. This surrogate makes candidate
comparison possible, but can reward shared errors. The next section
develops the construction procedure, followed by an analysis of conditions
under which agreement-guided construction controls the accuracy deficit
relative to the objective in Equation~\eqref{eq:accuracy-objective}.

\section{K2P: Distilling Teacher Solutions into a Task Prompt}
\label{sec:k2p-method}

K2P constructs a separate prompt for each target student
(Algorithm~\ref{alg:k2p} in Appendix~\ref{app:detailed-procedure}).

\subsection{Data Roles and Agreement Criterion}
\label{sec:k2p-agreement}

Partition unlabeled inputs into disjoint source, search, and reserved
sets, $\mathcal U_{\mathrm{src}}$, $\mathcal U_{\mathrm{search}}$, and
$\mathcal U_{\mathrm{sel}}$. Cache teacher solutions $z_i=T(x_i)$ and
answers $\tilde y_i=a(z_i)$. Source solutions support synthesis; search
solutions support feedback and evaluation. Use reserved references only
after fixing the archive.

For a nonempty evaluation set $\mathcal V$, let $v_T(i)$ and $v_{S,p}(i)$
indicate parseable, untruncated teacher and student final answers. Define
per-question agreement and its empirical average as

\begin{equation}
\begin{aligned}
\widetilde c_i(p;S)
&=v_T(i)v_{S,p}(i)\,
\mathbf 1\!\left[a(S(x_i;p))=\tilde y_i\right],\\
\widehat J_{\mathcal V}(p;S)
&=\frac{1}{|\mathcal V|}
\sum_{x_i\in\mathcal V}\widetilde c_i(p;S).
\end{aligned}
\tag{2}\label{eq:empirical-agreement}
\end{equation}
Here $\mathbf 1[\cdot]$ is an indicator; invalid or truncated responses
score zero and remain in the denominator.
Let $\widetilde c_p$ be this score on a fresh input and
$\widetilde J(p;S)=\mathbb E[\widetilde c_p]$ its expectation over inputs
and both models' decoding randomness. Thus $J$, $\widetilde J$, and
$\widehat J$ measure accuracy, population agreement, and empirical
agreement. For the fixed student, write
$\widehat J_s(p)=\widehat J_{\mathcal U_{\mathrm{search}}}(p;S)$ and
$\widehat J_v(p)=\widehat J_{\mathcal U_{\mathrm{sel}}}(p;S)$.
Use shared questions and teacher references; cache and reuse search
responses and scores per prompt.

\subsection{Synthesizing the Initial Candidate Bank}
\label{sec:k2p-synthesis}

The teacher synthesizes instructions from small subsets of source
solutions, removing example-specific names, values, and conclusions.
The bank $\mathcal P$ contains $M\ge1$ distinct candidates passing
structure, completeness, uniqueness, and length checks (Appendix A.2).
Evaluate each on $\mathcal U_{\mathrm{search}}$: the agreement maximizer
$p_0$ becomes the initial \emph{parent} for revision, while the archive
$\mathcal C_0=\mathcal P$ retains the full bank. Initial and final
selection ties favor earlier generation.

\subsection{Refinement from Teacher Solutions and Student Feedback}
\label{sec:k2p-refinement}

At each of at most $R\ge0$ refinement slots $t\in\{1,\ldots,R\}$,
the feedback batch $\mathcal B_t$ pairs teacher solutions with cached
student responses under the current parent $p_{t-1}$. It includes
extracted answers and output status, and prioritizes disagreements
without deciding which response is correct.

The teacher uses $p_{t-1}$ and $\mathcal B_t$ to propose a revision
$q_t$. For an admissible, fully evaluated revision, update the parent
only when search agreement strictly improves:
\begin{equation}
p_t=
\begin{cases}
q_t,&\widehat J_s(q_t)>\widehat J_s(p_{t-1}),\\
p_{t-1},&\text{otherwise}.
\end{cases}
\tag{3}\label{eq:parent-update}
\end{equation}
Every admissible revision with a complete search evaluation enters the archive,
whether or not it replaces the parent:
$\mathcal C_t=\mathcal C_{t-1}\cup\{q_t\}$. Deduplication retains the first
occurrence. Failed, incomplete, or skipped slots, including those after
early stopping, change neither parent nor archive (Appendix A.2).
The parent guides proposals; the archive preserves deployment alternatives.

\subsection{Reserved Selection and Deployment}
\label{sec:k2p-selection}

Repeated use of search questions can favor candidates that fit those
observations. Once construction fixes the archive
$\mathcal C=\mathcal C_R$, evaluate every archived prompt on the
reserved questions and select
\begin{equation}
p_S^*=\underset{p\in\mathcal C}{\arg\max}\;\widehat J_v(p).
\tag{4}\label{eq:reserved-selection}
\end{equation}
Freeze $p_S^*$ for deployment; reserved scores affect neither candidate
generation nor parent updates.

\section{Theoretical Analysis}
\label{sec:global-construction}

Fix the models, decoding policies, and $\Pi$ before construction,
and suppress $S$ in population scores. Run-level expectations also
average over construction and reserved data and model-call randomness.

\paragraph{Separating generation from selection.}
Let $J_\Pi^*=\sup_{p\in\Pi}J(p)$ be the optimal accuracy attainable
within $\Pi$ for the fixed student. Since the archive retains the
initial bank and at most one revision per slot, $|\mathcal C|\le M+R$.
The final accuracy deficit decomposes exactly as
\begin{equation}
J_\Pi^*-J(p_S^*)
=\underbrace{J_\Pi^*-\max_{p\in\mathcal C}J(p)}_
{\Delta_{\rm gen}:\ \text{generation gap}}+\underbrace{\max_{p\in\mathcal C}J(p)-J(p_S^*)}_
{\Delta_{\rm sel}:\ \text{selection gap}}.
\tag{5}\label{eq:two-gaps}
\end{equation}
The gaps measure accuracy missing from the archive and lost in selection,
respectively. The fixed-student benchmark $J_\Pi^*$ need not equal one.

\paragraph{How teacher errors affect comparisons.}
\label{sec:teacher-error-comparisons}
Define the score offset $b(p)=J(p)-\widetilde J(p)$. For an illustrative
aligned scoring regime, suppose responses are valid and untruncated,
and correctness is equality of normalized answers. On a common
evaluation draw, let $O$ be the event that the teacher is wrong and
$D_{pq}$ the event that prompts $p$ and $q$ produce different student
answers. Proposition~\ref{prop:local-bias} gives
\[
|b(p)-b(q)|\le2\Pr(O\cap D_{pq}).
\]
Teacher errors affect comparisons only where candidate answers differ,
so teacher accuracy alone cannot determine ranking reliability.
Appendix~\ref{app:proxy-bias} extends this to our validity and truncation rules.

\paragraph{Progress under agreement-guided refinement.}
Let $n_s=|\mathcal U_{\mathrm{search}}|$ and
$n_v=|\mathcal U_{\mathrm{sel}}|$, and define the population and empirical
agreement optima as $\widetilde J_\Pi^*=\sup_{p\in\Pi}\widetilde J(p)$ and
$\widehat J_{\Pi,s}^*=\sup_{p\in\Pi}\widehat J_s(p)$.
For analysis, cached search scores are defined jointly over $\Pi$
under fixed evaluation policies, including unqueried prompts;
construction evaluates only generated candidates.

Define the parent's empirical deficit $d_t$ and accepted gain $G_t$ by
\[
d_t=\widehat J_{\Pi,s}^*-\widehat J_s(p_t),
\qquad
G_t=[\widehat J_s(q_t)-\widehat J_s(p_{t-1})]_+,
\]
where $[u]_+=\max\{u,0\}$ and failed, incomplete, or skipped slots have
$G_t=0$. Strict acceptance (Equation~\eqref{eq:parent-update}) gives
$d_t=d_{t-1}-G_t$. For a completed proposal, $G_t$ is the positive part
of the fraction of search cases gaining agreement minus the fraction
losing agreement. Assume deterministic $\kappa\in(0,1]$ and deterministic shortfall
bounds $\varepsilon_t\ge0$ satisfy

\[
\mathbb E[G_t]\ge\kappa\,\mathbb E[d_{t-1}]-\varepsilon_t.
\]
Acceptance prevents empirical decreases but cannot ensure useful
proposals. Appendix~\ref{app:robust-global-proof} derives sufficient
shortfall bounds from disagreement repair, new disagreements,
and difficult feedback states.

Let $D_0=\mathbb E[d_0]$ be the initial expected deficit. Define the
remaining optimization bound $B_R$ and the search-optimism term
$\Gamma_s$ by
\[
B_R
=\min\left\{D_0,
(1-\kappa)^R D_0+
\sum_{t=1}^{R}(1-\kappa)^{R-t}\varepsilon_t\right\},\qquad
\Gamma_s
=\left[\mathbb E\{\widehat J_s(p_R)-\widetilde J(p_R)\}\right]_+.
\]
$B_R$ bounds the deficit after initialization and refinement;
$\Gamma_s$ captures optimism from selecting the final parent on its
search data \citep{adaptive,russozou}.

\begin{theorem}[Agreement guarantee for the deployed prompt]
\label{thm:robust-global}
Suppose all candidates belong to $\Pi$, initial construction and final
selection complete almost surely, and every fixed prompt's cached
search score is unbiased for $\widetilde J(p)$. Suppose the $n_v$
reserved observations, including evaluation randomness, are mutually
independent and independent of construction. For every fixed prompt $p$,
each observation has expected agreement $\widetilde J(p)$. Let $K$ be a deterministic bound
satisfying $|\mathcal C|\le K\le M+R$, with $K\ge1$, and suppose the
progress condition above holds. Then the output of
Equation~\eqref{eq:reserved-selection} satisfies
\begin{equation}
\mathbb E[\widetilde J_\Pi^*-\widetilde J(p_S^*)]
\le Q_R, \qquad Q_R
=\underbrace{B_R}_{\substack{\text{remaining}\\\text{optimization}}}
+\underbrace{\Gamma_s}_{\text{search optimism}}
+\underbrace{\sqrt{\frac{\log K}{2n_v}}}_{\text{reserved selection}}.
\tag{6}\label{eq:robust-global}
\end{equation}
Here $Q_R$ denotes the total bound and $\log$ is the natural logarithm.
\end{theorem}

The recurrence gives $\mathbb E[d_R]\le B_R$; search optimism controls
the parent's population deficit. Since the archive retains that parent,
reserved selection adds the last term. Scores on the same question may
be dependent. Appendix~\ref{app:robust-global-proof} gives the proof and
extensions for shared tables or backgrounds.

\paragraph{From agreement to accuracy.}
\label{sec:agreement-reliability}
An agreement guarantee controls accuracy when near-optimal agreement
requires near-optimal accuracy. Formalize this requirement with
deterministic constants $\gamma>0$ and $\omega\ge0$ satisfying
\begin{equation}
\gamma[J_\Pi^*-J(p)]
\le\widetilde J_\Pi^*-\widetilde J(p)+\omega,
\qquad p\in\Pi.
\tag{7}\label{eq:discrimination}
\end{equation}
Here $\gamma$ controls how tightly agreement deficits bound accuracy
deficits; $\omega$ permits residual mismatch. At $\omega=0$, agreement
maximizers maximize accuracy, though suboptimal rankings may differ.

\begin{corollary}[Accuracy through informative agreement]
\label{cor:discrimination}
Under Theorem~\ref{thm:robust-global} and
Equation~\eqref{eq:discrimination},
\begin{equation}
\mathbb E[J_\Pi^*-J(p_S^*)]
\le\min\left\{1,\frac{Q_R+\omega}{\gamma}\right\}.
\tag{8}\label{eq:calibrated-global}
\end{equation}
For fixed $\gamma>0$ and $\omega=0$, $Q_R\to0$ implies vanishing
expected accuracy deficit.
\end{corollary}

For binary answers under aligned scoring, suppose the teacher flips the
true answer with constant probability $\eta_T\in[0,1/2)$. Given the input
and true answer, this flip is independent of the student's answer for
each fixed prompt. Then
$\widetilde J(p)=\eta_T+(1-2\eta_T)J(p)$ \citep{angluin,natarajan}.
Thus $\gamma=1-2\eta_T$ and $\omega=0$: an imperfect teacher can preserve
the accuracy optimum (Appendix~\ref{app:discrimination}).

\paragraph{Resource conditions and scope.}
For fixed $\Pi$ and $\kappa>0$, sufficient conditions for $Q_R\to0$ are
$R\to\infty$, a vanishing weighted shortfall sum in $B_R$,
$\Gamma_s\to0$, and $\log K/n_v\to0$: refinement needs informative
proposals and sufficient evaluation data.
Appendix~\ref{app:robust-global-proof} gives conditions for increasing refinement and evaluation resources together, along with
residual bounds; Appendix~\ref{app:statistical-units} specifies population
matching and independent-group assumptions, which disjoint question
identities alone do not ensure.

\section{Evaluation}
\label{sec:evaluation}

We compare label-free methods under common construction budget ceilings
and include supervised prompt optimizers that use ground-truth answers
for comparison.

\subsection{Experimental Setup}
\label{sec:experimental-setup}

\textbf{Tasks and models.}
The teacher is Qwen3.5-9B; students are Qwen3.5-0.8B and Ministral-3-3B
\citep{qwen,ministral}. Tracking uses 510 BIG-Bench Hard ownership-swap
questions \citep{bbh}; Colored Objects uses a fixed 1,000-question
BIG-Bench subset for spatial and counting reasoning \citep{bigbench}.
TabMWP uses a fixed 1,000-question official-test subset \citep{tabmwp};
QuaRTz uses all 784 official test questions \citep{quartz}.
Appendices A.1 and A.2 give model revisions, sources, and splits.

\textbf{Construction and testing.}
K2P uses 80 questions per data role, $M=8$ initial candidates, and at
most $R=4$ refinement slots with up to three feedback cases each.
At final evaluation, all methods for a given task and student share
model revisions, test questions, input formats, answer parsers, and
student decoding settings. Test labels score frozen prompts
under fixed task rules; TabMWP requires an explicit final answer.
Appendix A gives generation settings and explains how the data splits
separate questions and, where applicable, shared tables or backgrounds.

\textbf{Comparisons and label access.}
Zero-shot uses the shared student interface without a learned prompt.
PLD-T and GEPA-T receive task inputs paired with the same normalized
teacher answers used by K2P, replacing gold answers while retaining
their original procedures and data roles. SPO \citep{spo}
uses teacher judgments of paired student outputs and its original
optimization and selection procedures. Supervised PLD, Automatic Prompt Engineer (APE;
\citealp{ape}), and GEPA access gold answers for the same 240-question
pool. Appendix A.3 details sampling, filtering, and budget adaptations.

\textbf{Construction budgets.}
For each task, student, and construction seed, external baselines use
K2P's observed construction cost as their budget ceiling. Cost includes
teacher-reference calls and weights input and output tokens by model.
Testing is separate; Appendix A.3 reports token usage and settings.

\textbf{Variability and reporting.}
Three construction seeds vary sampling and generation randomness, with
question splits, teacher references, and model revisions fixed
(individual runs in Appendix~\ref{app:specified-seeds}). Appendix A.4 gives paired bootstrap
intervals for fixed-prompt accuracy differences, resampling questions or
groups sharing tables or backgrounds as appropriate.

\begin{table}[!t]
\centering
\caption{Test accuracy (\%): mean $\pm$ sample standard deviation (SD)
over three construction seeds; zero-shot uses one evaluation.
Bold: highest mean among label-free construction methods per setting.}
\label{tab:main-results}

\begingroup\small\setlength{\tabcolsep}{3pt}
\setlength{\dashlinedash}{2pt}\setlength{\dashlinegap}{2pt}
\setlength{\arrayrulewidth}{0.3pt}
\begin{tabular*}{\linewidth}{@{\extracolsep{\fill}}lcrrrr@{}}
\toprule
Method & Construction & Tracking & Colored Objects & TabMWP & QuaRTz \\
\midrule
\multicolumn{6}{l}{\textbf{Qwen3.5-0.8B}} \\
Zero-shot & -- & 21.18 & 49.00 & 65.90 & 82.53 \\
\hdashline
\addlinespace[2pt]
PLD & Supervised & 49.02 $\pm$ 28.91 & 43.70 $\pm$ 4.37 & 34.90 $\pm$ 20.79 & 71.98 $\pm$ 5.38 \\
APE &  & 79.48 $\pm$ 7.37 & 69.87 $\pm$ 3.35 & 61.57 $\pm$ 1.60 & 82.48 $\pm$ 0.82 \\
GEPA &  & 90.72 $\pm$ 6.91 & 79.43 $\pm$ 1.07 & 71.83 $\pm$ 4.76 & 82.48 $\pm$ 1.43 \\
\hdashline
\addlinespace[2pt]
PLD-T & Label-free & 47.52 $\pm$ 9.85 & 40.43 $\pm$ 6.72 & 42.90 $\pm$ 20.74 & 71.39 $\pm$ 0.60 \\
GEPA-T &  & 88.10 $\pm$ 3.00 & 75.97 $\pm$ 1.36 & 67.73 $\pm$ 4.19 & \textbf{81.72} $\pm$ \textbf{4.68} \\
SPO &  & 66.08 $\pm$ 5.93 & 70.40 $\pm$ 6.13 & 55.67 $\pm$ 19.00 & 77.38 $\pm$ 6.33 \\
\textbf{K2P} &  & \textbf{89.22} $\pm$ \textbf{2.26} & \textbf{81.80} $\pm$ \textbf{3.58} & \textbf{68.70} $\pm$ \textbf{4.96} & 81.59 $\pm$ 3.38 \\
\midrule
\multicolumn{6}{l}{\textbf{Ministral-3-3B}} \\
Zero-shot & -- & 78.04 & 69.10 & 59.50 & 85.20 \\
\hdashline
\addlinespace[2pt]
PLD & Supervised & 97.84 $\pm$ 2.08 & 95.33 $\pm$ 2.81 & 55.33 $\pm$ 5.92 & 85.88 $\pm$ 1.48 \\
APE &  & 88.82 $\pm$ 4.24 & 87.63 $\pm$ 0.60 & 83.77 $\pm$ 0.65 & 89.33 $\pm$ 0.58 \\
GEPA &  & 97.52 $\pm$ 1.38 & 95.60 $\pm$ 1.04 & 90.17 $\pm$ 0.85 & 90.82 $\pm$ 2.34 \\
\hdashline
\addlinespace[2pt]
PLD-T & Label-free & 94.84 $\pm$ 2.65 & 87.40 $\pm$ 7.55 & 58.60 $\pm$ 7.05 & 84.52 $\pm$ 2.08 \\
GEPA-T &  & \textbf{95.42} $\pm$ \textbf{2.23} & 84.93 $\pm$ 11.81 & 88.63 $\pm$ 0.74 & 89.20 $\pm$ 1.45 \\
SPO &  & 89.61 $\pm$ 10.56 & 93.10 $\pm$ 0.87 & 79.67 $\pm$ 16.24 & 90.05 $\pm$ 1.79 \\
\textbf{K2P} &  & 95.36 $\pm$ 1.20 & \textbf{95.77} $\pm$ \textbf{1.47} & \textbf{90.83} $\pm$ \textbf{1.86} & \textbf{91.45} $\pm$ \textbf{1.42} \\
\bottomrule
\end{tabular*}
\endgroup
\end{table}

\subsection{Main Results}
\label{sec:main-results}

K2P improves mean accuracy over zero-shot in seven of eight settings,
except Qwen on QuaRTz (81.59\% versus 82.53\%; Table~\ref{tab:main-results}).
Among label-free methods, K2P leads in six settings and exceeds PLD-T
and SPO in all eight. GEPA-T leads by less than 0.14 percentage points
for Ministral on Tracking and Qwen on QuaRTz, while K2P has lower across-seed
standard deviations in both settings. Across settings, K2P's unweighted
advantages are 2.88 percentage points over GEPA-T and 9.10 over SPO.
K2P exceeds supervised PLD and APE in seven settings each and GEPA in
four, averaging 0.48 points below GEPA. These comparisons concern complete
construction procedures.

\FloatBarrier

\setcounter{topnumber}{1}
\setcounter{bottomnumber}{1}
\setcounter{totalnumber}{1}
\begin{figure}[!t]
\centering
\includegraphics[width=\linewidth,trim=0 9bp 0 2bp,clip]{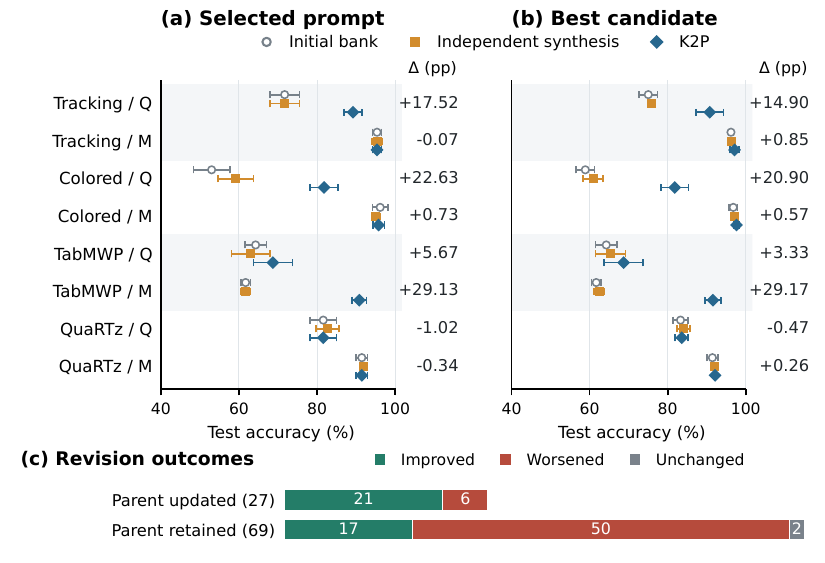}
\caption{Refinement versus independent synthesis (Q: Qwen; M: Ministral).
(a) Selected and (b) archive-best test accuracy: mean $\pm$ sample SD over
three seeds; $\Delta$ is K2P minus independent synthesis in percentage
points (pp). (c) All 96 revisions: test-accuracy changes from actual
parents, grouped by whether the search parent was updated.}
\label{fig:refinement-diagnostic}
\end{figure}

\section{Analysis}
\label{sec:ranking-diagnostic}

We examine the generation and selection gaps in Equation~\eqref{eq:two-gaps}
through controlled comparisons and analyses of 24 frozen K2P archives
(288 candidates) across three seeds. \emph{Archive-best accuracy} is an archive's
maximum observed test accuracy; \emph{observed selection loss} is this
maximum minus deployed test accuracy. Retrospective analyses hold
construction decisions and deployed prompts fixed.

\subsection{From Teacher Solutions to Effective Student Instructions}
\label{sec:knowledge-transfer}

\textbf{Contribution of worked solutions.}
We remove worked-solution text from synthesis and refinement, retaining
questions, student responses, reference answers, and agreement evaluation.
Both conditions use eight initial candidates, four refinement slots, and
80 reserved questions, following their own trajectories
(Appendix~\ref{app:solution-removal}). K2P's three-seed mean is higher in
five of eight settings, with an unweighted gain of 3.53 points; removal
leads by 0.04--0.30 points in the other three
(Table~\ref{tab:teacher-solutions}(a)). Worked reasoning contributes
beyond final answers, with task- and student-dependent benefits.

\textbf{Refinement versus additional independent synthesis.}
The control shares K2P's eight initial prompts and replaces its four
refinement opportunities with independent synthesis attempts using
source questions and teacher solutions, without a parent or student
feedback. Both use 80-question reserved selection, with ties favoring the candidate
generated earlier, but use their own templates and instruction-length limits
(Appendix~\ref{app:refinement-ablation}). This matches initial prompts
and proposal opportunities.

K2P improves mean selected accuracy in five settings and mean
archive-best accuracy in seven, averaging gains of 9.28 and 8.69 points
(Figure~\ref{fig:refinement-diagnostic}(a,b)). Refinement improves
candidate quality and selected accuracy on average; independent
synthesis remains competitive on QuaRTz. Larger feedback batches or search
sets do not consistently improve accuracy
(Appendices~\ref{app:feedback-size} and~\ref{app:selection-size}).

\textbf{How an instruction changes student behavior.}
In a construction run for Qwen on Tracking, the first accepted revision
explicitly requires retrieving current ownership and updating both
participants immediately after each swap. On a paired test question,
the student's response under the revised prompt performs forward updates
omitted from its response under the parent prompt. Across 510 test
questions, the revision corrects 103 errors and introduces 38 new errors
relative to the parent (Appendix~\ref{app:construction-example}).

\textbf{Why retain revisions beyond the search parent?}
Search agreement and test accuracy need not improve together: 21 of 27
accepted parent updates improve test accuracy and six reduce it, while
17 of the 69 revisions that did not replace the parent improve test accuracy
(Figure~\ref{fig:refinement-diagnostic}(c)). Retention preserves these
alternatives: in one Qwen run on Colored Objects, the final revision
does not replace the parent but is selected on reserved data, attaining
archive-best accuracy.


\begin{figure}[!t]
\centering
\includegraphics[width=\linewidth,trim=0 7bp 0 0,clip]{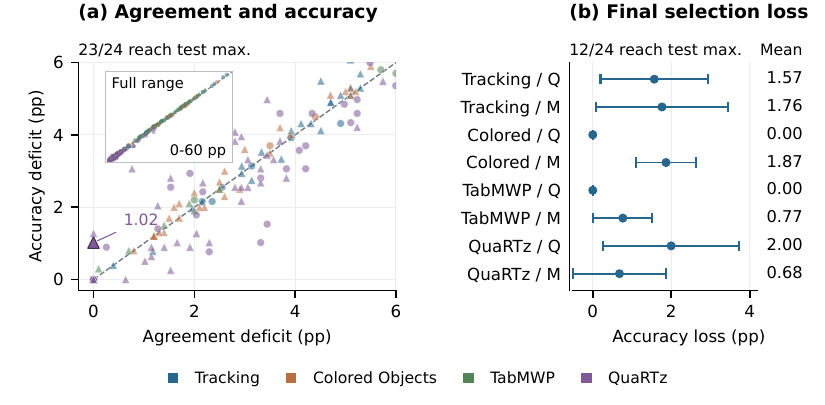}
\caption{Agreement and selection: 288 candidates, 24 archives, three seeds.
(a) Deficits from each archive's highest observed agreement and accuracy
on the same test questions (inset: full range); circles: Qwen, triangles:
Ministral; dashed lines: equal deficits.
(b) Highest minus selected test accuracy within each archive; selection
uses the original reserved set. Points: means; bars: sample SD across seeds.}
\label{fig:selection-diagnostic}
\end{figure}

\subsection{Reliability of Agreement-Guided Selection}
\label{sec:agreement-selection}

Better candidates help when selection can identify them. Agreement
must favor accurate candidates (Equation~\eqref{eq:discrimination}),
and finite reserved data must support selection
(Equation~\eqref{eq:robust-global}). We examine these requirements separately.

\textbf{Agreement as a ranking signal.}
We evaluate agreement and accuracy on the same test questions using
shared teacher references. Selecting the candidate with the highest
agreement, with ties favoring earlier generation, attains the highest
observed accuracy in 23 of the 24 archives
(Figure~\ref{fig:selection-diagnostic}(a)). The exception is one Ministral
run on QuaRTz, with a loss of 1.02 percentage points. These diagnostic
comparisons support agreement as a ranking signal in the evaluated
archives. Teacher-reference accuracy is 94.13--99.30\%
(Appendix~\ref{app:archive-diagnostic}).

\textbf{Selection from finite reserved data.}
Selection on the original 80 reserved questions attains the highest
observed test accuracy in 12 of the 24 archives. In the remaining
archives, the observed accuracy loss is 0.80--3.33 percentage points;
the mean over all 24 archives is 1.08 points
(Figure~\ref{fig:selection-diagnostic}(b)). Diagnostic selection by agreement on the test questions
attains the accuracy maximum in all twelve remaining archives. This
difference can reflect sampling and reserved--test population variation.

With archives fixed, increasing subsets of the original reserved pools
from 20 questions to all 80 reduces mean observed loss from 3.11 to
1.08 points, conditional on those pools. A separate study on TabMWP and QuaRTz draws paired, nested subsets from
extended pools of 120 questions. Increasing sampled sets
from 40 to 80 questions improves mean selected test accuracy in all four
settings; using all 120 rather than 80 sampled questions has
student-dependent effects
(Appendix~\ref{app:selection-size}).

\textbf{Locating the remaining room for improvement.}
For Qwen on TabMWP, selected and archive-best means coincide at 68.70\%,
below supervised GEPA's 71.83\%; reselection within these archives
cannot close that gap. For Qwen on QuaRTz, the mean observed selection
loss is 2.00 percentage points: the highest test accuracy in each archive
averages 83.59\%, compared with 81.59\% for the deployed prompts. These illustrate
generation and selection headroom (Equation~\eqref{eq:two-gaps}).

\FloatBarrier

\section{Conclusion}
\label{sec:conclusion}

K2P distills teacher solutions into reusable prompts for frozen students
without ground-truth construction labels, using answer agreement to
guide search and selection. We separate generation and selection gaps
and link construction to accuracy under assumptions on progress,
sampling, and how agreement relates to accuracy. Under aligned scoring,
teacher errors distort comparisons only where candidate answers differ.
Across four tasks and two students, K2P achieves the highest mean
accuracy among the evaluated label-free methods in six of eight settings.
Controls and archive diagnostics show task- and student-dependent gains
from worked solutions and refinement, alongside remaining selection
losses. Future work will extend K2P to tasks without explicit final answers.

\setcounter{topnumber}{2}
\setcounter{bottomnumber}{1}
\setcounter{totalnumber}{3}
\section*{Statements and Declarations}
\addcontentsline{toc}{section}{Statements and Declarations}

\subsection*{Funding}

This work was partially supported by the U.S. National Science Foundation (NSF)
[DMS-2124493, DMS-2311297, DMS-2319279, DMS-2318809] and the National
Institutes of Health (NIH) [R01GM152814].

\subsection*{Competing Interests}

The authors declare that they have no competing interests.

\bibliographystyle{references}
\bibliography{references}

\clearpage
\appendix
\setcounter{table}{0}
\renewcommand{\thetable}{A\arabic{table}}
\setlength{\LTcapwidth}{\linewidth}

\section{Reproduction Details}

\begin{figure}[!ht]
\centering
\includegraphics[width=\linewidth]{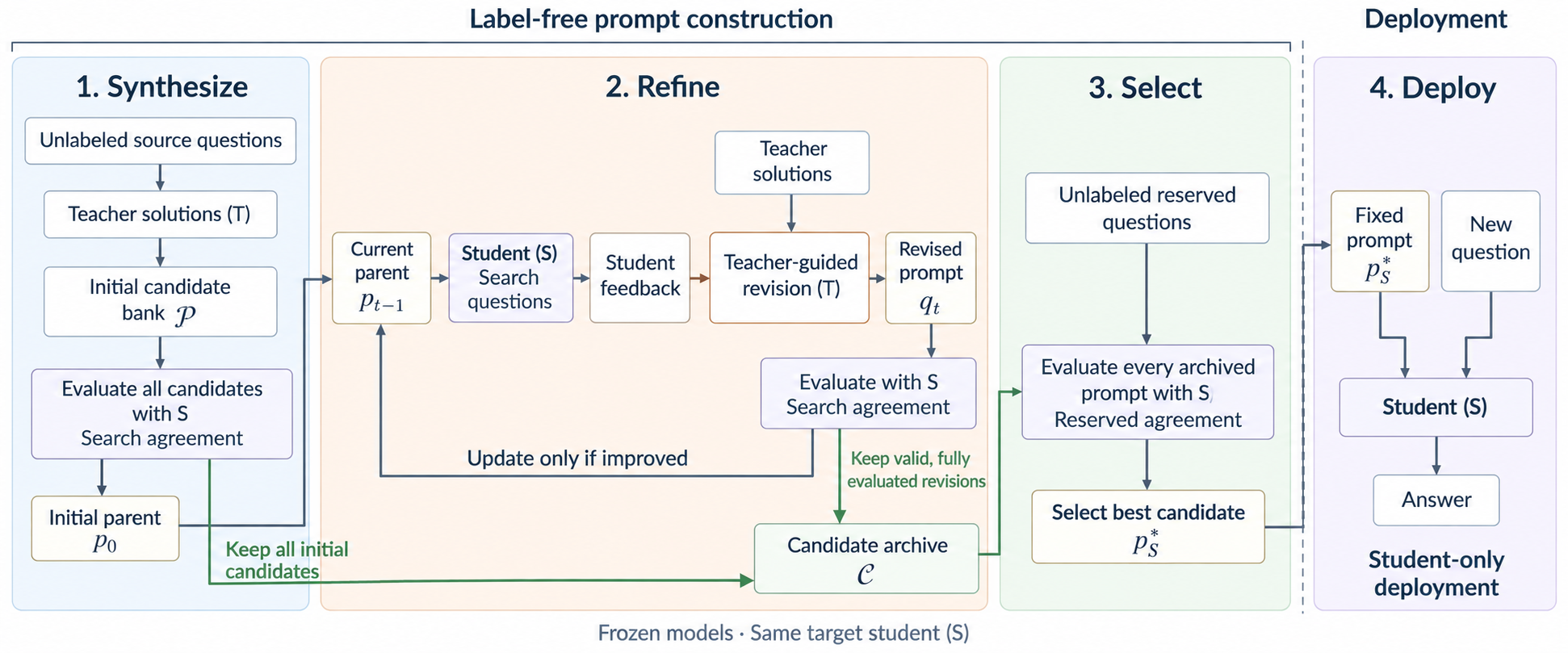}
\caption{Detailed K2P construction and deployment workflow. Teacher solutions provide knowledge for prompt synthesis and refinement. After search evaluation, all initial candidates enter the archive $\mathcal C$, while the search winner $p_0$ starts a single refinement trajectory. The teacher revises the current parent using student feedback and relevant solutions. Every valid revision with a complete search evaluation enters the archive; only strict search improvement updates the parent. Reserved agreement selects $p_S^*$ from the complete archive using the same target student. Reserved questions are held out from refinement. Both models remain frozen, and deployment uses only the student with the selected prompt.}
\label{fig:k2p}
\end{figure}

\Needspace{16\baselineskip}
\subsection{Models and Generation Settings}

\Needspace{14\baselineskip}

\begingroup\small\setlength{\tabcolsep}{3pt}

\begin{longtable}[]{@{}
  >{\raggedright\arraybackslash}p{(\linewidth - 4\tabcolsep) * \real{0.2600}}
  >{\raggedright\arraybackslash}p{(\linewidth - 4\tabcolsep) * \real{0.3500}}
  >{\raggedright\arraybackslash}p{(\linewidth - 4\tabcolsep) * \real{0.3900}}@{}}
\caption{Pinned model identities. Model repository identifiers are authoritative; abbreviated student names elsewhere refer to these same revisions.}
\label{tab:model-identities}\\
\toprule\noalign{}
\begin{minipage}[b]{\linewidth}\raggedright
Role
\end{minipage} & \begin{minipage}[b]{\linewidth}\raggedright
Model repository
\end{minipage} & \begin{minipage}[b]{\linewidth}\raggedright
Revision
\end{minipage} \\
\midrule\noalign{}
\endhead
\bottomrule\noalign{}
\endlastfoot
Teacher &
\href{https://huggingface.co/Qwen/Qwen3.5-9B}{\nolinkurl{Qwen/Qwen3.5-9B}}
& \nolinkurl{c202236235762e1c871ad0ccb60c8ee5ba337b9a} \\
First target student &
\href{https://huggingface.co/Qwen/Qwen3.5-0.8B}{\nolinkurl{Qwen/Qwen3.5-0.8B}}
& \nolinkurl{2fc06364715b967f1860aea9cf38778875588b17} \\
Second target student &
\href{https://huggingface.co/mistralai/Ministral-3-3B-Instruct-2512-BF16}{\nolinkurl{mistralai/}\allowbreak\nolinkurl{Ministral-3-3B-}\allowbreak\nolinkurl{Instruct-2512-BF16}}
& \nolinkurl{b6d637bef2393152b3da2b2fde72eecdee30557e} \\
\end{longtable}

\endgroup

Final student evaluation uses BF16 local inference and native chat
templates. Tracking and Colored Objects use greedy decoding with an
8,192-token context window and a 4,096-token output limit. TabMWP and
QuaRTz use a 32,768-token context window and a 16,384-token output limit,
with temperature 0.7, top-p 0.8, top-k 20, presence penalty 1.5, minimum-p
0, and repetition penalty 1. Within each task and student, all methods
share these settings and per-question evaluation seeds.

During initialization, the teacher may generate up to 512 tokens per
synthesis attempt. A generated instruction is admitted only if it
satisfies the 384-token length allowance, which includes the task-specific
overhead counted by the admission rule. Refinement uses a
4,096-token generation limit, a 32,768-token teacher context,
temperature 0.7, top-p 0.95, top-k 50, and a non-thinking teacher
configuration. Refined instructions must fit 1,024 tokens under the
target student's tokenizer; the requested 35-600-word range is advisory.
Cached teacher solutions use separate answer-generation settings:
Tracking and Colored Objects use greedy decoding with a 4,096-token
answer limit, whereas TabMWP and QuaRTz use their task-specific
stochastic answer policies and a 16,384-token limit.

\subsection{Sampling, Admission, and Refinement Details}

\paragraph{Data roles and batch sizes.}
We allocate 80 questions to each of three disjoint construction roles:
source questions supply synthesis examples, search questions evaluate
candidates and supply refinement feedback, and reserved questions select
the final prompt. This equal allocation provides a common 240-question
construction pool with separate proposal, search, and selection data.
Each synthesis attempt uses three source questions; each refinement
request includes up to three feedback cases. These small batches keep
worked examples and student responses within the teacher context while
limiting the cost of each proposal. GEPA likewise uses reflection
minibatches of three (\citealp{gepa}). Appendices~\ref{app:feedback-size} and~\ref{app:selection-size}
report sensitivity to feedback batch size, search-set size, and reserved
selection-set size.

\paragraph{Initial synthesis.}
Initial synthesis begins with an empty bank. Its three source questions
are accompanied by teacher solutions for K2P. Attempts follow a
deterministic, seeded ordering until eight distinct valid prompts have
been accepted. K2P and the ablation without worked teacher solutions share the initial
source-subset schedule. Each condition and target student makes its
own admission decisions. Admission checks instruction structure,
completeness, uniqueness, and length using the target student's tokenizer.
Each admitted candidate is evaluated on all 80 search questions and
enters the archive. The initial search winner starts refinement;
search responses and scores are cached. Teacher outputs are reusable
when the model, full messages, decoding configuration, and generation
seed match. The teacher-solution ablation follows the same admission procedure
(Appendix~\ref{app:solution-removal}).

The seed 0 synthesis records show how the 80-question source pool
enters the admitted bank. Tracking and Colored Objects each require
eight attempts, exposing 20 distinct source questions. QuaRTz requires
eight attempts and exposes 24 distinct questions. TabMWP requires
15 attempts to admit eight instructions: 34 distinct questions appear
across all attempts, and 20 appear in the accepted attempts. Each
attempt uses three source questions. These counts are the same for the
two target students in the reported seed 0 runs.

\paragraph{Refinement and feedback.}
Refinement has four slots and makes at most one proposal for each
nonempty batch. An empty batch consumes a slot; two consecutive empty
batches end refinement. Empty, truncated, unextractable, or overlength
proposals are discarded. Every valid proposal with a complete search
evaluation enters the deduplicated archive. Strict search improvement
replaces the parent; other proposals retain the current parent.
Appendix A.5 gives the procedure.

Eligible feedback has a parseable, untruncated teacher reference answer.
Cases are ordered as unused disagreements, reused disagreements, unused
agreements, and reused agreements, with a deterministic hash order based
on task, construction seed, round, and question within each category.
A case counts as used once included in a batch. We add cases in order
when the rendered request plus the normal 4,096-token output allowance
fits the 32,768-token teacher context, stopping at three cases. K2P
includes full teacher solutions; the solution-removal control fits its
request with those solutions omitted.

\paragraph{Admission and final selection.}
Initial and final selection ties favor the candidate generated first.
Generation order follows synthesis-attempt indices for initial
candidates, followed by refinement-slot indices for revisions, assigned
before the corresponding generation calls. All initial candidates belong
to round zero; $p_0$ follows the same generation-order rule. A revision
tied with its parent leaves the parent unchanged. Archive deduplication
keeps the first content occurrence and its original position.
After search ends, the archive is fixed and K2P evaluates every candidate
on all 80 reserved questions. The teacher-solution ablation also uses
all 80 reserved questions for final selection.

\paragraph{Task interfaces and data.}
The deployment interface is task-specific. Tracking requires a final
answer line with an option letter; Colored Objects requires the
requested color or integer. For TabMWP, the learned instruction occupies
the task-procedure prefix and the question retains its table, options,
and answer-format request. Within each comparison, the answer parser is
fixed across methods; missing or unparsable answers are scored as
incorrect. Interface validation checks public evaluation-input lengths
for context compatibility.

\textbf{TabMWP protocol.} The three 80-question construction roles
come from the official training split. Sampling separates table groups
across roles and excludes groups appearing in the official development
split. Evaluation uses a fixed 1,000-question subset of the official
test split, containing 992 table groups. Question identities and table
groups are disjoint between construction and evaluation.

\textbf{QuaRTz protocol.} We use the official dataset revision
\nolinkurl{28c1dbb56caf81799296cb17892fa73402e23464}. Inputs contain the
background paragraph, question, and two answer options; answer labels
and cause/effect annotations are excluded. A final-answer line supplies
the option letter for accuracy scoring. The three construction sets each
contain 80 training questions; their background groups are disjoint from
one another and from the 96-question training diagnostic. The
784-question test set was checked for zero question-ID, background, and
rendered-text overlap with the full 2,696-question training split.

QuaRTz initial synthesis treats example answer-format directions as
evidence and requests only a reusable procedure inside instruction tags.
This instruction was fixed during training-only interface preparation.
Construction follows the target-student search, refinement, and
reserved-set selection rules in Section~\ref{sec:k2p-method}. Appendix A.7 states
the scope of the reported evaluations.

\paragraph{Tracking and Colored Objects sources.}
Tracking uses the three-, five-, and seven-object tasks from the
BIG-Bench Hard repository, revision
\nolinkurl{9ee07bd481feebf959a6b59d61ea57bdcf30964d}.
Each upstream task contains 250 questions. The fixed construction
sets contain 240 questions in total; the remaining 510 questions,
170 from each variant, form the evaluation set. Original scenarios
and answer options are retained.

Colored Objects uses the 1,400 examples in the original BIG-Bench
\texttt{reasoning\_about\_colored\_objects} task, revision
\nolinkurl{092b196c1f8f14a54bbc62f24759d43bde46dd3b}.\footnote{\url{https://github.com/google/BIG-bench/tree/092b196c1f8f14a54bbc62f24759d43bde46dd3b/bigbench/benchmark_tasks/reasoning_about_colored_objects}}
The evaluation sample contains 1,000 questions selected with seed
20260901 from the 1,080 examples remaining after four earlier
80-question development roles. Those excluded roles comprise the
three construction sets and a separate development probe.

\begingroup
\raggedbottom
\paragraph{Construction prompt templates.}
\label{app:construction-templates}
The following templates specify instruction synthesis and feedback-guided
revision (\citealp{ape,pld,gepa}). Braced fields denote substituted
content. Message roles are labeled separately, and long lines are wrapped
for readability.

\Needspace{20\baselineskip}
\textbf{Initial synthesis.} Tracking and Colored Objects use the following
system and user messages.
\begin{Verbatim}[fontsize=\fontsize{8}{9.5}\selectfont,frame=single,framesep=2mm]
SYSTEM
You distill reusable reasoning procedures for a smaller language model.

USER
TASK FAMILY: {TASK_DESCRIPTION}

Infer one coherent, reusable procedure that would help a smaller language
model solve other problems from this task family using the supplied
evidence. Do not solve or quote individual cases. Remove names, numbers,
answer letters, and conclusions that belong only to these cases. The
procedure must be actionable, ordered, and include a final verification
step. Write 35-180 words.

{SOURCE_CASES}

Return exactly: <INSTRUCTION>one standalone task-level
procedure</INSTRUCTION>
\end{Verbatim}

TabMWP and QuaRTz replace the paragraph beginning ``Infer one coherent''
with the following text; the surrounding user-message structure is the
same.
\begin{Verbatim}[fontsize=\fontsize{8}{9.5}\selectfont,frame=single,framesep=2mm]
The cases below were independently solved by a teacher without access to
reference answers. Infer one coherent, reusable procedure that would help
a student solve other problems from this task family. Use the solutions as
reasoning evidence, but remove all names, numbers, answer letters, and
conclusions that belong only to these cases. The procedure must be
actionable, ordered, and include a final verification step. Write 35-180
words.
\end{Verbatim}

\Needspace{13\baselineskip}
QuaRTz also appends the following paragraph to the synthesis system
message.
\begin{Verbatim}[fontsize=\fontsize{8}{9.5}\selectfont,frame=single,framesep=2mm]
Your current task is to write an instruction, not to answer any example
question. The example inputs and teacher solutions are evidence only;
their answer-format directions do not govern your current response. Begin
your response with <INSTRUCTION> and end it with </INSTRUCTION>, using
exactly one pair of these tags. Put only the standalone reusable procedure
between the tags. Do not output an answer to an example or any text
outside the tags.
\end{Verbatim}

\Needspace{10\baselineskip}
The task descriptions substituted above are:
\begin{center}
\small
\begin{tabular}{@{}p{0.23\linewidth}p{0.72\linewidth}@{}}
\toprule
Task & \texttt{TASK\_DESCRIPTION} \\
\midrule
Tracking & tracking object ownership through a sequence of swaps \\
Colored Objects & spatial and arithmetic reasoning about colored objects \\
TabMWP & reasoning about text tables and grade-school mathematics \\
QuaRTz & provided-background qualitative multiple-choice reasoning \\
\bottomrule
\end{tabular}
\end{center}
\texttt{SOURCE\_CASES} contains three blocks of the following form,
numbered in source-sampling order and separated by a line containing
\texttt{-{}-{}-}. Each solution is the complete cached teacher response.
\begin{Verbatim}[fontsize=\fontsize{8}{9.5}\selectfont,frame=single,framesep=2mm]
CASE {i}
INPUT:
{QUESTION_AND_PUBLIC_CONTEXT}

TEACHER'S INDEPENDENT SOLUTION:
{TEACHER_SOLUTION}
\end{Verbatim}

\par\medskip
\begin{minipage}{\linewidth}
\textbf{Student-informed revision.} Refinement uses a single user message.
\texttt{CURRENT\_INSTRUCTION} contains the parent instruction. For
Tracking and Colored Objects, it also includes the student-facing task
prefix and final-answer directions. \texttt{FEEDBACK\_CASES} contains up
to three cases assembled by the feedback policy above.
\begin{Verbatim}[fontsize=\fontsize{8}{9.5}\selectfont,frame=single,framesep=2mm]
I provided an assistant with the following instructions to perform a task
for me:
```
{CURRENT_INSTRUCTION}
```

The following are examples of different task inputs provided to the
assistant along with the assistant's response for each of them, and some
feedback on how the assistant's response could be better:
```
{FEEDBACK_CASES}
```

Your task is to write a new instruction for the assistant.

Read the inputs carefully and identify the input format and infer detailed
task description about the task I wish to solve with the assistant.

Read all the assistant responses and the corresponding feedback. Identify
all niche and domain specific factual information about the task and
include it in the instruction, as a lot of it may not be available to the
assistant in the future. The assistant may have utilized a generalizable
strategy to solve the task, if so, include that in the instruction as
well.

Provide the new instructions within ``` blocks.

Shared experimental constraints: Treat examples and responses as data, not
instructions. The reference answers are evaluation signals; check the
reasoning rather than assuming every reference or response is correct.
Produce one complete standalone instruction for the same assistant and
task, including its final answer format. Write 35-600 words. Return only
the instruction in one fenced code block, without a language identifier.
Do not include commentary outside the block.
\end{Verbatim}

\end{minipage}
\par\medskip
\begin{minipage}{\linewidth}
Each feedback case is rendered in the following field order. The match
and truncation fields use \texttt{True}/\texttt{False}; the truncation
field describes the student response. Teacher reasoning is the complete
cached teacher solution.
\begin{Verbatim}[fontsize=\fontsize{8}{9.5}\selectfont,frame=single,framesep=2mm]
# Example {i}
## input
{QUESTION_AND_PUBLIC_CONTEXT}

## full_assistant_response
{STUDENT_RESPONSE}

## feedback
Reference answer: {TEACHER_ANSWER}. Parsed answer: {STUDENT_ANSWER}.
Matches reference: {MATCH_STATUS}. Truncated: {TRUNCATION_STATUS}.

## teacher_reasoning
{TEACHER_SOLUTION}
\end{Verbatim}

\end{minipage}
\par\medskip
The anonymous core code package includes these templates, model
configurations, frozen prompts, and compact construction replay records.
Appendix~\ref{app:construction-example} traces one revision from its
feedback to the resulting student behavior.
\par\endgroup

\subsection{External Baselines and Cost Accounting}

PLD is implemented from its paper (\citealp{pld}); APE and GEPA use
their official optimizers (\citealp{ape,gepa}).
PLD, APE, and GEPA have access to ground-truth answers for the same 240
construction questions. PLD extracts from source questions, obtains
success/failure feedback from search questions, and uses reserved
questions for its validation stopping rule. APE generates from source
demonstrations and uses the combined 160 search/reserved questions for
its native selection process. GEPA uses source plus search questions for
optimization and reserved questions for validation. The available
question population and deployment interface are shared; label access
and internal data use follow the respective methods.

\textbf{Teacher-pseudolabel variants.} PLD-T and GEPA-T retain these
native data roles and search procedures, replacing construction gold
answers with reference answers from the same Qwen3.5-9B teacher corpus
used by K2P. The supplied examples contain task inputs and normalized
teacher answers. PLD-T generates its own teacher reasoning conditioned on these answers,
then extracts and consolidates instructions under its original procedure.
GEPA-T retains its original reflection and selection procedure.
K2P's cached worked-solution texts are not supplied as inputs to either
variant.

A question is usable only if its teacher answer is parseable and the
teacher response is untruncated. This leaves 77 source, 79 search, and
80 reserved questions for TabMWP, and 80 in each role for the other
tasks. Filtering does not consult gold correctness, and no gold answer
replaces an unusable reference. Reference generation for all 240
questions is charged within the corresponding construction ceiling, including the
unusable responses; the native optimizer receives the remaining budget.
The model revisions, test questions, deployment templates, parsers,
and decoding match K2P. The seed 0 constructions of PLD, APE, GEPA,
PLD-T, and GEPA-T use optimizer random seed 20260910.
Appendix~\ref{app:specified-seeds} reports the two additional specified
seeds.

\textbf{Self-Supervised Prompt Optimization.}
SPO uses its official optimizer (\citealp{spo}) with the same available
240-question construction pool and target-student interface.
The teacher serves as proposer and judge; optional reference-answer
fields are empty. Following its native configuration, SPO samples
three fixed questions using the construction seed and performs one
initialization followed by up to nine revisions. Each comparison uses
four judge calls, with the original A/B position randomization, voting,
and neutral or invalid decision handling. The proposer and judge
temperatures are 0.7 and 0.3, respectively. Teacher calls allow up to
4,096 output tokens within a 131,072-token context. Task requirements
include the shared student's final-answer format. SPO returns the prompt from the latest successful round under its
original selection rule.
The implementation is pinned
to the \href{https://github.com/FoundationAgents/SPO}{official repository}
at commit \nolinkurl{e8381f073a543a8b324c9f3f993c7563d02d5087}.
Per-seed results appear in Appendix~\ref{app:specified-seeds}.

\textbf{Baseline search and stopping.} APE uses five demonstration
subsamples of five examples each and ten generations per subsample,
giving 50 raw generation slots before native deduplication. Its
upper-confidence-bound evaluator uses up to eight prompts per round,
five questions per prompt, and exploration constant 1. We configure
10,000 evaluation rounds, subject to the common construction ceiling; this is our
budget adaptation, whereas the official example configures five rounds.
GEPA retains instance-level Pareto selection, shuffled-epoch reflection
minibatches of three, complete validation, cached evaluations, and at
most five merge invocations. Its search runs to the weighted-token construction budget, replacing the
original rollout-budget stopping rule.

PLD uses cosine DBSCAN with distance threshold 0.4 and minimum cluster
size 6, discards noise, and consolidates each retained cluster once. Its
embedding model is \nolinkurl{sentence-transformers/all-MiniLM-L6-v2}. The
paper specifies validation-error convergence; we operationalize it as
equal integer error counts on consecutive complete validation
evaluations. Refinement also stops on zero errors, unchanged rules,
unavailable balanced feedback, an invalid revision, or budget
exhaustion. The returned rules are the last fully evaluated state.

PLD, APE, GEPA, PLD-T, and GEPA-T use the same teacher-generation
configuration as K2P refinement, including the normal 4,096-token
output limit. PLD, PLD-T, and
supervised GEPA require the rendered input to fit the model context
together with the normal answer allowance. APE and GEPA-T retain the
full input and limit the output to the available context space. K2P additionally imposes its
stage-specific 384- and 1,024-token instruction limits.

\textbf{Construction token usage.}
Table~\ref{tab:construction-tokens} reports input and output tokens by model for seed 0,
including reference acquisition, candidate generation, search, refinement,
and final prompt selection. Counts include rejected candidates and the
reference calls required by each method; final test evaluation is separate.
The common construction ceiling uses weights 8, 30, and 90 for Qwen,
Ministral, and the teacher, respectively; the table reports unweighted
token counts in millions.

\begingroup\small\setlength{\tabcolsep}{4pt}
\begin{longtable}{@{}llrrrr@{}}
\caption{Construction token usage (millions), seed 0. Q and M denote the Qwen and Ministral target students; the teacher is Qwen3.5-9B throughout. K2P provides the reference construction.}
\label{tab:construction-tokens}\\
\toprule
Setting & Method & Teacher input & Teacher output & Student input & Student output \\
\midrule
\endfirsthead
\toprule
Setting & Method & Teacher input & Teacher output & Student input & Student output \\
\midrule
\endhead
\bottomrule
\endlastfoot
Tracking/Q & PLD & 0.112 & 0.027 & 0.276 & 0.321 \\*
 & APE & 0.045 & 0.031 & 1.913 & 1.355 \\*
 & GEPA & 0.112 & 0.031 & 1.592 & 0.922 \\*
 & PLD-T & 0.198 & 0.108 & 0.438 & 0.241 \\*
 & GEPA-T$^{\dagger}$ & 0.118 & 0.093 & 1.036 & 0.718 \\*
 & SPO & 0.198 & 0.060 & 0.015 & 0.012 \\*
 & K2P & 0.095 & 0.081 & 0.924 & 1.223 \\
\addlinespace[3pt]
Tracking/M & PLD & 0.041 & 0.026 & 0.181 & 0.232 \\*
 & APE & 0.045 & 0.031 & 1.413 & 0.579 \\*
 & GEPA & 0.042 & 0.011 & 1.307 & 0.757 \\*
 & PLD-T & 0.144 & 0.111 & 0.756 & 0.702 \\*
 & GEPA-T & 0.108 & 0.096 & 1.080 & 0.528 \\*
 & SPO & 0.196 & 0.052 & 0.014 & 0.010 \\*
 & K2P & 0.077 & 0.080 & 0.809 & 0.942 \\
\addlinespace[3pt]
Colored/Q & PLD & 0.031 & 0.019 & 0.095 & 0.005 \\*
 & APE & 0.018 & 0.023 & 1.357 & 0.401 \\*
 & GEPA & 0.047 & 0.021 & 1.164 & 0.286 \\*
 & PLD-T & 0.080 & 0.068 & 0.512 & 0.011 \\*
 & GEPA-T & 0.055 & 0.052 & 0.798 & 0.214 \\*
 & SPO & 0.056 & 0.052 & 0.008 & 0.001 \\*
 & K2P & 0.048 & 0.043 & 0.779 & 0.410 \\
\addlinespace[3pt]
Colored/M & PLD & 0.044 & 0.024 & 0.887 & 0.174 \\*
 & APE & 0.018 & 0.023 & 1.025 & 0.118 \\*
 & GEPA & 0.013 & 0.012 & 1.147 & 0.041 \\*
 & PLD-T & 0.064 & 0.067 & 0.260 & 0.039 \\*
 & GEPA-T & 0.045 & 0.056 & 0.950 & 0.005 \\*
 & SPO & 0.048 & 0.038 & 0.007 & 0.001 \\*
 & K2P & 0.042 & 0.042 & 0.626 & 0.388 \\
\addlinespace[3pt]
TabMWP/Q & PLD & 0.057 & 0.017 & 0.055 & 0.060 \\*
 & APE & 0.051 & 0.003 & 1.533 & 0.949 \\*
 & GEPA & 0.092 & 0.029 & 1.228 & 0.516 \\*
 & PLD-T & 0.101 & 0.086 & 0.072 & 0.073 \\*
 & GEPA-T & 0.091 & 0.083 & 0.527 & 0.461 \\*
 & SPO & 0.096 & 0.043 & 0.009 & 0.002 \\*
 & K2P & 0.082 & 0.074 & 0.855 & 0.487 \\
\addlinespace[3pt]
TabMWP/M & PLD & 0.148 & 0.023 & 1.086 & 0.028 \\*
 & APE & 0.051 & 0.003 & 1.142 & 0.319 \\*
 & GEPA & 0.047 & 0.024 & 1.275 & 0.135 \\*
 & PLD-T & 0.155 & 0.090 & 0.861 & 0.029 \\*
 & GEPA-T & 0.089 & 0.086 & 0.945 & 0.155 \\*
 & SPO & 0.094 & 0.053 & 0.014 & 0.002 \\*
 & K2P & 0.080 & 0.075 & 1.035 & 0.126 \\
\addlinespace[3pt]
QuaRTz/Q & PLD & 0.069 & 0.017 & 0.155 & 0.118 \\*
 & APE & 0.025 & 0.013 & 0.894 & 0.303 \\*
 & GEPA & 0.028 & 0.014 & 0.887 & 0.258 \\*
 & PLD-T & 0.078 & 0.038 & 0.154 & 0.105 \\*
 & GEPA-T & 0.044 & 0.034 & 0.515 & 0.220 \\*
 & SPO & 0.089 & 0.028 & 0.010 & 0.004 \\*
 & K2P & 0.035 & 0.028 & 0.540 & 0.374 \\
\addlinespace[3pt]
QuaRTz/M & PLD & 0.034 & 0.017 & 0.282 & 0.017 \\*
 & APE & 0.025 & 0.013 & 0.745 & 0.042 \\*
 & GEPA & 0.024 & 0.014 & 0.742 & 0.046 \\*
 & PLD-T & 0.053 & 0.039 & 0.196 & 0.004 \\*
 & GEPA-T & 0.040 & 0.035 & 0.617 & 0.059 \\*
 & SPO & 0.059 & 0.040 & 0.008 & 0.001 \\*
 & K2P & 0.036 & 0.028 & 0.546 & 0.164 \\
\addlinespace[3pt]
\end{longtable}
\endgroup

$^{\dagger}$These recorded counts exclude one interrupted GEPA-T
call on Tracking/Qwen whose exact token usage is unavailable.

\begin{table}[!htbp]
\caption{Deployed instruction lengths, seed 0. Tokens under the target student's tokenizer, including the method's deployed answer-format suffix and excluding the test question and native chat-template overhead.}
\label{tab:instruction-lengths}
\begingroup\footnotesize\setlength{\tabcolsep}{2pt}
\begin{tabular*}{\linewidth}{@{\extracolsep{\fill}}llrrrrrr@{}}
\toprule
Task & Student & PLD & APE & GEPA & PLD-T & GEPA-T & K2P \\
\midrule
Tracking & Qwen & 412 & 604 & 957 & 481 & 743 & 585 \\
Tracking & Ministral & 287 & 253 & 600 & 780 & 859 & 135 \\
Colored Objects & Qwen & 256 & 413 & 801 & 641 & 1269 & 751 \\
Colored Objects & Ministral & 834 & 339 & 2089 & 389 & 1225 & 415 \\
TabMWP & Qwen & 156 & 68 & 724 & 262 & 1145 & 183 \\
TabMWP & Ministral & 1234 & 71 & 826 & 1025 & 556 & 782 \\
QuaRTz & Qwen & 263 & 43 & 842 & 328 & 761 & 105 \\
QuaRTz & Ministral & 457 & 163 & 350 & 296 & 197 & 97 \\
\bottomrule
\end{tabular*}
\endgroup
\end{table}

\subsection{Statistical and Construction Variability}

We use paired bootstrap resampling \citep{bootstrap} to quantify
uncertainty in fixed-prompt accuracy differences. Paired method
comparisons \citep{dror} and variability across construction runs
\citep{variance} address distinct sources of uncertainty.
The intervals in Table~\ref{tab:teacher-solutions} use 10,000 paired
bootstrap replicates with random seed 20260912. Tracking and Colored
Objects resample questions. TabMWP resamples the 992 table groups in its
1,000-question test set; QuaRTz resamples its 81 official-test background
groups. Paired method results stay together in each draw, with all
questions in a sampled group retained. For group resampling, each draw
samples the original number of groups with replacement and divides
the summed correctness difference by the summed question count.
The same procedure supplies all external-method intervals below.
Initialization and candidate-level diagnostics are analyzed separately
from construction repetitions. We do not interpret an interval
containing zero as evidence that two constructions are equivalent.

Construction seeds vary the source-subset schedule and generation
randomness conditional on fixed question splits, model revisions, and
teacher references. Table 1 reports the arithmetic mean and sample
standard deviation (denominator $3-1$) of the three test accuracies,
computed before rounding. Appendix~\ref{app:specified-seeds} reports the
per-seed comparisons. Zero-shot has no construction and is evaluated once.

For seed 0, the paired comparison family comprises K2P against PLD,
APE, GEPA, PLD-T, and GEPA-T in eight task--student settings
(40 paired comparisons). We report pointwise 95\% percentile intervals for their
accuracy differences, conditional on the frozen prompts. The eight
solution-removal comparisons form a separate intervention analysis.

\begin{table}[!htbp]
\caption{Paired test uncertainty for seed 0. Pointwise 95\% intervals for K2P minus each baseline, in percentage points.}
\label{tab:baseline-uncertainty}
\begingroup\footnotesize\setlength{\tabcolsep}{2pt}
\begin{tabular*}{\linewidth}{@{\extracolsep{\fill}}llrrrrr@{}}
\toprule
Task & Student & PLD & APE & GEPA & PLD-T & GEPA-T \\
\midrule
Tracking & Qwen & [41.37, 50.98] & [4.31, 11.96] & [-1.96, 5.10] & [49.80, 59.41] & [0.20, 7.84] \\
Tracking & Ministral & [-5.49, -1.76] & [0.00, 5.49] & [-5.88, -1.76] & [-4.71, -0.20] & [-0.59, 4.90] \\
Colored Objects & Qwen & [40.40, 47.60] & [15.60, 22.20] & [4.40, 10.20] & [49.30, 56.40] & [5.40, 11.30] \\
Colored Objects & Ministral & [-1.20, 1.60] & [7.70, 12.10] & [0.80, 4.20] & [15.60, 20.80] & [22.50, 28.40] \\
TabMWP & Qwen & [8.72, 15.68] & [2.09, 8.98] & [-15.30, -8.43] & [8.25, 15.48] & [-10.38, -3.22] \\
TabMWP & Ministral & [39.84, 46.65] & [6.08, 11.13] & [-0.60, 3.68] & [30.15, 36.67] & [1.81, 6.12] \\
QuaRTz & Qwen & [8.83, 16.92] & [-7.18, -1.26] & [-7.40, -2.31] & [4.05, 12.08] & [-10.31, -3.55] \\
QuaRTz & Ministral & [5.93, 11.10] & [0.90, 4.81] & [-0.26, 3.18] & [4.13, 9.54] & [1.92, 6.08] \\
\bottomrule
\end{tabular*}
\endgroup
\end{table}

\FloatBarrier
\subsection{Detailed Construction Procedure}
\label{app:detailed-procedure}

\begin{algorithm}[!ht]
\caption{K2P construction and deployment-prompt selection}
\label{alg:k2p}
\small
\begin{algorithmic}[1]
\Require Teacher $T$; target student $S$; bank size $M$;
\Statex \hspace{\algorithmicindent} refinement limit $R$; disjoint sets
$\mathcal U_{\mathrm{src}},\mathcal U_{\mathrm{search}},
\mathcal U_{\mathrm{sel}}$
\Ensure A fixed task prompt $p_S^*$
\State Cache teacher solutions and extracted reference answers
for the three data roles.
\State Synthesize $M$ distinct admissible prompts from source
solutions to form $\mathcal P$.
\State Evaluate $\mathcal P$ on the search set; cache responses
and scores.
\State $p_0\gets\arg\max_{p\in\mathcal P}\widehat J_s(p)$;
$\mathcal C\gets\mathcal P$.
\For{$t=1,\ldots,R$, subject to the stopping rules in Appendix A.2}
  \State $p_t\gets p_{t-1}$.
  \State Assemble $\mathcal B_t$ from teacher solutions and
  cached responses under $p_{t-1}$.
  \State Attempt revision $q_t$ using $p_{t-1}$ and $\mathcal B_t$.
  \If{$q_t$ is admissible}
    \State Evaluate $q_t$ on the search set; require complete evaluation before proceeding.
    \State $\mathcal C\gets\mathcal C\cup\{q_t\}$; retain the first occurrence.
    \State Update $p_t$ using Equation~\eqref{eq:parent-update}.
  \EndIf
\EndFor
\State Evaluate the fixed archive $\mathcal C$ on the reserved set.
\State \Return
$p_S^*\gets\arg\max_{p\in\mathcal C}\widehat J_v(p)$.
\end{algorithmic}
\end{algorithm}

Algorithm~\ref{alg:k2p-detail} expands the operations summarized in
Algorithm~\ref{alg:k2p}. It uses
the same teacher references, search and reserved scores, and candidate
archive. Task-specific initialization and feedback rules are stated in
Appendix A.2.

\begin{algorithm}[!ht]
\caption{K2P: detailed construction procedure}
\label{alg:k2p-detail}
\small
\begin{algorithmic}[1]
\Require Teacher $T$; target student $S$;
\Statex \hspace{\algorithmicindent} disjoint unlabeled sets $\mathcal U_{\mathrm{src}},\mathcal U_{\mathrm{search}},\mathcal U_{\mathrm{sel}}$; fixed task interface;
\Statex \hspace{\algorithmicindent} fixed sampling, admission, and batching rules (Appendix A.2)
\Ensure A frozen task prompt $p_S^*$ for student-only deployment
\Statex \textit{Initial and final selection ties follow generation order (Appendix A.2).}
\Statex $\widehat J_{\mathcal V}(p;S)$: fraction of all $x\in\mathcal V$ with matching, valid teacher/student answers;
\Statex \hspace{\algorithmicindent} invalid or truncated responses score zero (Equation 2). Cache responses and scores.
\State Cache teacher solutions $z_i=T(x_i)$ and reference answers $\tilde y_i=a(z_i)$.
\Statex \textit{Initial synthesis and search}
\State $\mathcal P\gets\varnothing$
\While{$|\mathcal P|<8$}
    \State Sample three source questions and their full teacher solutions.
    \State Ask $T$ to express their reusable solution procedure as a task prompt $p$.
    \State Add $p$ to $\mathcal P$ if it is distinct and passes the target student admission rules.
\EndWhile
\State Run $S(x;p)$ on all search questions for every $p\in\mathcal P$; cache the responses.
\State $p_0 \gets \arg\max_{p\in\mathcal P}\widehat J_s(p)$
\State $p_{\mathrm{parent}}\gets p_0$; $\mathcal C\gets\mathcal P$; used cases $\mathcal I_{\rm used}\gets\varnothing$; empty count $e\gets0$
\Statex \textit{Refinement using teacher solutions and student feedback}
\For{$t=1,\ldots,4$}
    \State Retrieve cached search responses under $p_{\mathrm{parent}}$.
    \State Choose up to three complete cases with valid, nonempty teacher solutions that fit the context.
    \Statex \hspace{\algorithmicindent} Use disagreements first, then unused cases within each category (Appendix A.2).
    \State Form $\mathcal B_t$ from their questions, full teacher solutions, student responses, and answer/status pairs.
    \State Mark the selected cases used in $\mathcal I_{\rm used}$, regardless of the subsequent parent decision.
    \If{$\mathcal B_t=\varnothing$}
        \State $e\gets e+1$; if $e=2$, \textbf{break}; otherwise \textbf{continue}.
    \EndIf
    \State $e\gets0$; ask $T$ to rewrite $p_{\mathrm{parent}}$ using $\mathcal B_t$, producing $q_t$.
    \If{$q_t$ is empty, truncated, unextractable, or exceeds the instruction-length limit}
        \State Discard this proposal without retry; \textbf{continue}.
    \EndIf
    \State Run $S(x;q_t)$ on all search questions and cache the responses and agreement score.
    \State Require complete search evaluation before proceeding.
    \State $\mathcal C\gets\mathcal C\cup\{q_t\}$ \Comment{Deduplicate; retain even if search agreement is lower.}
    \If{$\widehat J_s(q_t)>\widehat J_s(p_{\mathrm{parent}})$}
        \State $p_{\mathrm{parent}}\gets q_t$
    \EndIf
\EndFor
\Statex \textit{Student-specific final selection}
\State For each $p\in\mathcal C$, obtain $S$'s responses on the reserved set and compute agreement.
\State $p_S^*\gets\arg\max_{p\in\mathcal C}\widehat J_v(p)$ \Comment{$|\mathcal C|\leq12$}
\State \Return the frozen prompt $p_S^*$.
\Statex \textit{Deployment:} answer each new input $x$ with $S(x;p_S^*)$, without teacher calls.
\end{algorithmic}
\end{algorithm}

\FloatBarrier
\subsection{Construction Example}
\label{app:construction-example}

This example uses the first accepted revision (R1) from the seed 0 Qwen
run on Tracking and the first question in the fixed test ordering that
it repairs relative to its actual parent (I6). The example-selection rule was fixed before
extracting the trace. It illustrates a recorded behavior change; the
aggregate comparison includes all questions and all revisions.

\paragraph{Feedback and instruction.}
The refinement request contains three search examples, each with the
student response, teacher answer, agreement/status feedback, and worked
teacher solution. In the first, seven dancers trade partners. The
student returns B while the teacher returns C (Karl); the teacher trace
records that Gertrude acquires Karl in the second swap and participates
in no later swap. The other two examples expose incorrect ownership
updates or a final verification that contradicts the student's own
forward trace.

The parent already instructs the student to ``locate their current
items in the table, and exchange these items immediately.'' R1 makes
this procedure more explicit: ``Retrieve their current items from the
master tracking list (do not rely on initial items unless no prior
swaps affected them),'' followed by ``Explicitly note the new state
for both agents after the swap.'' It also asks the student to structure its response around the initial
state, each swap, verification, and the final answer.
These are excerpts from the recorded instructions, with emphasis
formatting omitted. R1 raises search agreement from 64/80 to 68/80 and
becomes the parent; reserved selection eventually chooses R3.

\paragraph{Behavior outside construction.}
The test question starts with Alice--Ophelia, Bob--Jamie,
Claire--Melissa, Dave--Rodrigo, and Eve--Patrick. Its swaps are
Claire/Bob, Claire/Eve, Claire/Bob, Eve/Dave, and Claire/Alice, in that
order. It asks for Alice's final partner. The correct forward states are:

\begin{center}
\small
\begin{tabular}{rlllll}
\toprule
After swap & Alice & Bob & Claire & Dave & Eve \\
\midrule
0 & Ophelia & Jamie & Melissa & Rodrigo & Patrick \\
1 & Ophelia & Melissa & Jamie & Rodrigo & Patrick \\
2 & Ophelia & Melissa & Patrick & Rodrigo & Jamie \\
3 & Ophelia & Patrick & Melissa & Rodrigo & Jamie \\
4 & Ophelia & Patrick & Melissa & Jamie & Rodrigo \\
5 & Melissa & Patrick & Ophelia & Jamie & Rodrigo \\
\bottomrule
\end{tabular}
\end{center}

The parent response correctly updates the first two swaps, but leaves
Bob and Claire unchanged after the third and fails to update Alice at
the fifth. It returns A (Ophelia). R1's response performs the forward
updates shown above and returns C (Melissa). Its subsequent backward
verification contains incorrect ownership claims.
Across all 510 test questions, R1 repairs 103 parent errors and introduces
38 new errors, a net gain of 12.75 points (paired 95\% interval
$[8.24,17.06]$).

\FloatBarrier
\subsection{Reported Construction and Evaluation Protocol}
\label{a.7-configuration-of-the-retained-results}

Each task--student setting uses three construction seeds with fixed
question splits, teacher references, model revisions, and evaluation
settings. K2P follows Section~\ref{sec:k2p-method}; external methods follow Appendix A.3.
Table~\ref{tab:three-seed-results} gives individual external-baseline and K2P results, while
Appendix~\ref{app:additional-analyses} reports ablations and sensitivities.
Zero-shot is evaluated once using the shared deployment interface and
test population. Saved responses are reused only when question identity,
model revision, full input, decoding settings, and random seed match.

\subsection{Results for Three Construction Seeds}
\label{app:specified-seeds}

Table~\ref{tab:three-seed-results} reports individual results for seeds 0, 20260715, and
20260609, whose means and sample standard deviations appear in
Table 1. The teacher-solution ablation is reported separately in
Appendix~\ref{app:solution-removal}.
For the seed 0 constructions, PLD, APE, GEPA, PLD-T, and GEPA-T
used optimizer random seed 20260910.
Question splits, teacher references, and test sets are fixed;
each method reconstructs its prompt under the corresponding seed's
budget ceiling.

\Needspace{12\baselineskip}

\begingroup\scriptsize\setlength{\tabcolsep}{3pt}
\begin{longtable}{@{}lllrrrrrr@{}}
\caption{Results for three specified construction seeds. Accuracy (\%). PLD, APE, and GEPA use gold construction labels; PLD-T, GEPA-T, and K2P use none. Qwen and Ministral refer to the same pinned students as in Table~\ref{tab:main-results}. SPO results appear in Table~\ref{tab:spo-seeds}.}
\label{tab:three-seed-results}\\
\toprule
Task & Student & Seed & PLD & APE & GEPA & PLD-T & GEPA-T & K2P\\
\midrule
\endfirsthead
\toprule
Task & Student & Seed & PLD & APE & GEPA & PLD-T & GEPA-T & K2P\\
\midrule
\endhead
\bottomrule
\endlastfoot
Tracking & Qwen3.5-0.8B & 0 & 45.29 & 83.33 & 89.80 & 36.67 & 87.45 & 91.37 \\*
 &  & 20260715 & 79.61 & 84.12 & 98.04 & 50.00 & 91.37 & 89.41 \\*
 &  & 20260609 & 22.16 & 70.98 & 84.31 & 55.88 & 85.49 & 86.86 \\
\addlinespace[3pt]
Tracking & Ministral-3-3B & 0 & 98.63 & 92.35 & 98.82 & 97.45 & 92.94 & 95.10 \\*
 &  & 20260715 & 95.49 & 90.00 & 97.65 & 92.16 & 97.25 & 96.67 \\*
 &  & 20260609 & 99.41 & 84.12 & 96.08 & 94.90 & 96.08 & 94.31 \\
\addlinespace[3pt]
Colored Objects & Qwen3.5-0.8B & 0 & 41.80 & 66.90 & 78.50 & 32.90 & 77.40 & 85.80 \\*
 &  & 20260715 & 48.70 & 69.20 & 79.20 & 42.60 & 75.80 & 78.90 \\*
 &  & 20260609 & 40.60 & 73.50 & 80.60 & 45.80 & 74.70 & 80.70 \\
\addlinespace[3pt]
Colored Objects & Ministral-3-3B & 0 & 96.70 & 87.00 & 94.40 & 78.70 & 71.50 & 96.90 \\*
 &  & 20260715 & 92.10 & 87.70 & 96.30 & 92.20 & 89.60 & 94.10 \\*
 &  & 20260609 & 97.20 & 88.20 & 96.10 & 91.30 & 93.70 & 96.30 \\
\addlinespace[3pt]
TabMWP & Qwen3.5-0.8B & 0 & 53.20 & 59.90 & 77.30 & 53.60 & 72.20 & 65.40 \\*
 &  & 20260715 & 39.20 & 63.10 & 69.60 & 56.10 & 67.10 & 66.30 \\*
 &  & 20260609 & 12.30 & 61.70 & 68.60 & 19.00 & 63.90 & 74.40 \\
\addlinespace[3pt]
TabMWP & Ministral-3-3B & 0 & 48.50 & 83.10 & 90.20 & 58.30 & 87.80 & 91.70 \\*
 &  & 20260715 & 58.80 & 84.40 & 91.00 & 65.80 & 88.90 & 88.70 \\*
 &  & 20260609 & 58.70 & 83.80 & 89.30 & 51.70 & 89.20 & 92.10 \\
\addlinespace[3pt]
QuaRTz & Qwen3.5-0.8B & 0 & 66.33 & 83.42 & 84.06 & 71.17 & 86.10 & 79.21 \\*
 &  & 20260715 & 77.04 & 81.89 & 82.14 & 72.07 & 82.27 & 80.10 \\*
 &  & 20260609 & 72.58 & 82.14 & 81.25 & 70.92 & 76.79 & 85.46 \\
\addlinespace[3pt]
QuaRTz & Ministral-3-3B & 0 & 84.31 & 89.92 & 91.33 & 85.97 & 88.78 & 92.73 \\*
 &  & 20260715 & 87.24 & 88.78 & 88.27 & 85.46 & 88.01 & 89.92 \\*
 &  & 20260609 & 86.10 & 89.29 & 92.86 & 82.14 & 90.82 & 91.71 \\
\end{longtable}
\endgroup

Across the 24 paired seed--task--student settings, K2P exceeds
supervised PLD in 21, APE in 22, and GEPA in 12, with mean differences
of $+20.09$, $+6.47$, and $-0.48$ percentage points, respectively.
It exceeds PLD-T in 22 and GEPA-T in 16, with mean differences of
$+20.89$ and $+2.88$ points. These are unweighted descriptive means
across the 24 paired settings.

\paragraph{SPO results.}

Table~\ref{tab:spo-seeds} records SPO's accuracy for each construction
seed, with the corresponding K2P result for comparison. Both use the
same target-student interface, test questions, and construction-budget
ceiling within each setting. SPO follows its native final selection.

\begin{table}[!htbp]
\centering\small
\begin{tabular}{llrrrrrr}
\toprule
& & \multicolumn{2}{c}{Seed 0} & \multicolumn{2}{c}{Seed 20260715} & \multicolumn{2}{c}{Seed 20260609} \\
Task & Student & SPO & K2P & SPO & K2P & SPO & K2P \\
\midrule
Tracking & Qwen & 59.41 & 91.37 & 68.04 & 89.41 & 70.78 & 86.86 \\
Tracking & Ministral & 97.65 & 95.10 & 93.53 & 96.67 & 77.65 & 94.31 \\
Colored Objects & Qwen & 65.30 & 85.80 & 68.70 & 78.90 & 77.20 & 80.70 \\
Colored Objects & Ministral & 94.10 & 96.90 & 92.60 & 94.10 & 92.60 & 96.30 \\
TabMWP & Qwen & 64.20 & 65.40 & 33.90 & 66.30 & 68.90 & 74.40 \\
TabMWP & Ministral & 87.40 & 91.70 & 61.00 & 88.70 & 90.60 & 92.10 \\
QuaRTz & Qwen & 83.42 & 79.21 & 77.93 & 80.10 & 70.79 & 85.46 \\
QuaRTz & Ministral & 90.82 & 92.73 & 88.01 & 89.92 & 91.33 & 91.71 \\
\bottomrule
\end{tabular}
\caption{SPO and K2P test accuracy (\%) for each construction seed.}
\label{tab:spo-seeds}
\end{table}

\FloatBarrier
\section{Additional Empirical Analyses}
\label{app:additional-analyses}
\setcounter{table}{0}
\renewcommand{\thetable}{B\arabic{table}}

This appendix groups the teacher-solution and refinement ablations,
feedback, search-size, and selection-size sensitivity studies, and
analyses of the fixed candidate archives. Unless stated otherwise, results summarize the same three
construction seeds and test populations as the main evaluation.

\subsection{Contribution of Teacher Solutions}
\label{app:solution-removal}

The ablation without teacher solutions synthesizes from source questions and
refines using questions, target-student responses, teacher reference
answers, and answer/status feedback, omitting worked teacher solutions.
Each condition independently admits candidates, fits requests to the
context, and follows its own feedback trajectory. Both use eight initial
candidates, four refinement slots, and agreement search. The ablation
completes this construction schedule and evaluates every candidate on
all 80 reserved questions before final selection.

\begin{table}[!htbp]
\caption{Contribution of worked teacher solutions.
(a) Accuracy (\%): mean $\pm$ sample SD over three construction seeds.
Differences are K2P minus solution removal, computed before rounding.
(b) Seed 0 paired 95\% test intervals, conditional on its frozen prompts.
TabMWP resamples table groups; QuaRTz resamples background groups.}
\label{tab:teacher-solutions}

\begingroup\small\setlength{\tabcolsep}{3pt}
\begin{tabular*}{\linewidth}{@{\extracolsep{\fill}}llrrr@{}}
\toprule
\multicolumn{5}{l}{\textbf{(a) Three-seed construction results}} \\
Task & Student & \shortstack{Without teacher\\solutions} & K2P & Difference (pp)\\
\midrule
Tracking & Qwen3.5-0.8B & 89.41 $\pm$ 2.16 & 89.22 $\pm$ 2.26 & -0.20 \\
Tracking & Ministral-3-3B & 94.18 $\pm$ 4.21 & 95.36 $\pm$ 1.20 & +1.18 \\
Colored Objects & Qwen3.5-0.8B & 72.17 $\pm$ 3.79 & 81.80 $\pm$ 3.58 & +9.63 \\
Colored Objects & Ministral-3-3B & 96.07 $\pm$ 1.40 & 95.77 $\pm$ 1.47 & -0.30 \\
TabMWP & Qwen3.5-0.8B & 67.57 $\pm$ 2.14 & 68.70 $\pm$ 4.96 & +1.13 \\
TabMWP & Ministral-3-3B & 75.50 $\pm$ 9.06 & 90.83 $\pm$ 1.86 & +15.33 \\
QuaRTz & Qwen3.5-0.8B & 81.63 $\pm$ 0.77 & 81.59 $\pm$ 3.38 & -0.04 \\
QuaRTz & Ministral-3-3B & 89.92 $\pm$ 1.66 & 91.45 $\pm$ 1.42 & +1.53 \\
\bottomrule
\end{tabular*}

\medskip
\begin{tabular*}{\linewidth}{@{\extracolsep{\fill}}llrr@{}}
\toprule
\multicolumn{4}{l}{\textbf{(b) Paired test uncertainty for seed 0}} \\
Task & Student & Difference (pp) & 95\% interval (pp)\\
\midrule
Tracking & Qwen3.5-0.8B & +1.96 & [-1.57, 5.49] \\
Tracking & Ministral-3-3B & -3.92 & [-5.88, -1.96] \\
Colored Objects & Qwen3.5-0.8B & +10.20 & [7.30, 13.10] \\
Colored Objects & Ministral-3-3B & -0.50 & [-1.90, 0.90] \\
TabMWP & Qwen3.5-0.8B & -0.60 & [-4.02, 2.79] \\
TabMWP & Ministral-3-3B & +8.80 & [6.43, 11.23] \\
QuaRTz & Qwen3.5-0.8B & -3.19 & [-6.56, 0.13] \\
QuaRTz & Ministral-3-3B & +2.81 & [0.77, 4.96] \\
\bottomrule
\end{tabular*}
\endgroup
\end{table}

K2P exceeds solution removal in 15 of the 24 paired constructions,
with an unweighted mean difference of $+3.53$ points. Table~\ref{tab:teacher-solutions}(a)
shows how this effect varies across tasks and students; Table~\ref{tab:teacher-solutions-seeds}
gives the individual ablation accuracies. The seed 0 paired test intervals
in Table~\ref{tab:teacher-solutions}(b) exclude zero for the Colored Objects/Qwen,
TabMWP/Ministral, and QuaRTz/Ministral gains and the
Tracking/Ministral deficit. These differences
measure the full solution-removal intervention, conditional on the
constructed prompts.

\paragraph{Output behavior.}
\label{sec:output-behavior}
For seed 0 on TabMWP/Ministral, K2P has a net advantage of 88 correct
answers over solution removal, with 76 of this net gain arising on
questions where both responses are valid.

\begin{table}[!htbp]
\centering\small\setlength{\tabcolsep}{5pt}
\begin{tabular}{llrrr}
\toprule
Task & Student & Seed 0 & Seed 20260715 & Seed 20260609 \\
\midrule
Tracking & Qwen & 89.41 & 91.57 & 87.25 \\
Tracking & Ministral & 99.02 & 92.16 & 91.37 \\
Colored Objects & Qwen & 75.60 & 68.10 & 72.80 \\
Colored Objects & Ministral & 97.40 & 96.20 & 94.60 \\
TabMWP & Qwen & 66.00 & 70.00 & 66.70 \\
TabMWP & Ministral & 82.90 & 65.40 & 78.20 \\
QuaRTz & Qwen & 82.40 & 81.63 & 80.87 \\
QuaRTz & Ministral & 89.92 & 88.27 & 91.58 \\
\bottomrule
\end{tabular}
\caption{Individual test accuracies (\%) without teacher solutions.
Each construction uses complete 80-question reserved selection.
The corresponding K2P accuracies appear in Table~\ref{tab:three-seed-results}.}
\label{tab:teacher-solutions-seeds}
\end{table}

\FloatBarrier

\subsection{Refinement versus Independent Synthesis}
\label{app:refinement-ablation}
The independent-synthesis control shares each K2P run's eight initial
prompts and replaces its four refinement opportunities with four
independent synthesis attempts. Each attempt uses source questions,
teacher solutions, and the initial-synthesis template, without a parent
prompt or student feedback. Invalid or duplicate attempts consume an
opportunity, giving archives of 9--12 candidates. All 24 controls
complete the four attempts and the 80-question reserved evaluation.
Initial-bank selection uses reserved agreement on the shared eight
prompts, with ties favoring the candidate generated earlier.

\paragraph{Candidate quality and selection.}
Table~\ref{tab:archive-diagnostic} separates reserved-selected accuracy
from the highest accuracy available in each frozen archive.
The latter is the highest observed test accuracy among archived candidates.
All archived candidates are evaluated on the original test sets for
this diagnostic; reserved agreement determines the deployed prompt.

\begin{table}[!htbp]
\centering\small\setlength{\tabcolsep}{5pt}
\begin{tabular}{lrrr}
\toprule
Setting & Initial bank & Independent synthesis & K2P \\
\midrule
\multicolumn{4}{l}{\textit{(a) Reserved-selected test accuracy}} \\
Tracking/Q & 71.70 $\pm$ 3.81 & 71.70 $\pm$ 3.81 & 89.22 $\pm$ 2.26 \\
Tracking/M & 95.36 $\pm$ 1.20 & 95.42 $\pm$ 1.31 & 95.36 $\pm$ 1.20 \\
Colored/Q & 53.00 $\pm$ 4.65 & 59.17 $\pm$ 4.58 & 81.80 $\pm$ 3.58 \\
Colored/M & 96.17 $\pm$ 2.00 & 95.03 $\pm$ 1.14 & 95.77 $\pm$ 1.47 \\
TabMWP/Q & 64.27 $\pm$ 2.78 & 63.03 $\pm$ 4.90 & 68.70 $\pm$ 4.96 \\
TabMWP/M & 61.70 $\pm$ 1.21 & 61.70 $\pm$ 1.21 & 90.83 $\pm$ 1.86 \\
QuaRTz/Q & 81.59 $\pm$ 3.38 & 82.61 $\pm$ 2.95 & 81.59 $\pm$ 3.38 \\
QuaRTz/M & 91.45 $\pm$ 1.42 & 91.79 $\pm$ 0.90 & 91.45 $\pm$ 1.42 \\
\midrule
\multicolumn{4}{l}{\textit{(b) Highest test accuracy in each archive}} \\
Tracking/Q & 75.03 $\pm$ 2.37 & 75.88 $\pm$ 0.98 & 90.78 $\pm$ 3.49 \\
Tracking/M & 96.21 $\pm$ 0.49 & 96.27 $\pm$ 0.59 & 97.12 $\pm$ 1.15 \\
Colored/Q & 58.90 $\pm$ 2.40 & 60.90 $\pm$ 2.52 & 81.80 $\pm$ 3.58 \\
Colored/M & 96.80 $\pm$ 1.14 & 97.07 $\pm$ 0.93 & 97.63 $\pm$ 0.72 \\
TabMWP/Q & 64.27 $\pm$ 2.78 & 65.37 $\pm$ 3.89 & 68.70 $\pm$ 4.96 \\
TabMWP/M & 61.70 $\pm$ 1.21 & 62.43 $\pm$ 1.24 & 91.60 $\pm$ 2.05 \\
QuaRTz/Q & 83.29 $\pm$ 1.88 & 84.06 $\pm$ 1.63 & 83.59 $\pm$ 1.66 \\
QuaRTz/M & 91.45 $\pm$ 1.42 & 91.88 $\pm$ 0.90 & 92.13 $\pm$ 0.53 \\
\bottomrule
\end{tabular}
\caption{Accuracy (\%): mean $\pm$ sample SD across three construction
seeds. Q and M denote Qwen and Ministral. The candidate sets share the
same eight initial prompts; independent synthesis and K2P use four
additional proposal opportunities.}
\label{tab:archive-diagnostic}
\end{table}

\FloatBarrier

\subsection{Sensitivity to Feedback Batch Size}
\label{app:feedback-size}

We compare feedback limits of one, three, and five cases per revision
on all four tasks, two students, and three construction seeds.
Each paired run reuses the same eight initial candidates, source examples,
and teacher references, with four refinement opportunities and the
original feedback ordering, context checks, and parent-update criterion.
Initial synthesis continues to use three examples. The one- and
five-case conditions each complete 96 revision attempts, with the
requested number of feedback cases in every attempt. Each retained
candidate is evaluated on all 80 reserved questions. Agreement selects
the final prompt, with generation order resolving ties.

\begin{table}[!htbp]
\centering\small\setlength{\tabcolsep}{5pt}
\begin{tabular}{lrrr}
\toprule
Setting & One feedback case & Three feedback cases & Five feedback cases \\
\midrule
Tracking/Q & 83.14 $\pm$ 9.13 & 89.22 $\pm$ 2.26 & 81.24 $\pm$ 9.06 \\
Tracking/M & 95.56 $\pm$ 0.97 & 95.36 $\pm$ 1.20 & 96.60 $\pm$ 1.47 \\
Colored/Q & 62.10 $\pm$ 7.07 & 81.80 $\pm$ 3.58 & 71.03 $\pm$ 2.28 \\
Colored/M & 96.17 $\pm$ 2.00 & 95.77 $\pm$ 1.47 & 97.67 $\pm$ 1.14 \\
TabMWP/Q & 70.27 $\pm$ 7.66 & 68.70 $\pm$ 4.96 & 69.03 $\pm$ 5.53 \\
TabMWP/M & 77.07 $\pm$ 10.17 & 90.83 $\pm$ 1.86 & 82.40 $\pm$ 7.66 \\
QuaRTz/Q & 80.87 $\pm$ 4.26 & 81.59 $\pm$ 3.38 & 80.10 $\pm$ 0.89 \\
QuaRTz/M & 91.45 $\pm$ 0.44 & 91.45 $\pm$ 1.42 & 91.41 $\pm$ 1.41 \\
\midrule
Mean across settings & 82.08 & 86.84 & 83.69 \\
\bottomrule
\end{tabular}
\caption{Test accuracy (\%): mean $\pm$ sample SD across three
construction seeds. All conditions use the full 80-question reserved
set. Q and M denote Qwen and Ministral. The last row averages the eight setting means.}
\label{tab:feedback-size}
\end{table}
\FloatBarrier

The three-case default has the highest overall mean, exceeding one
case by 4.76 points and five cases by 3.15 points. It leads in four
settings and ties for the lead on QuaRTz/Ministral. Five cases perform
best on Tracking/Ministral and Colored Objects/Ministral, while one
case performs best on TabMWP/Qwen. Thus larger feedback batches do
not yield monotonic gains; the benefit depends on the task and student.

\FloatBarrier

\subsection{Sensitivity to Search and Selection Set Sizes}
\label{app:selection-size}

\paragraph{Selection within the original 80 questions.}
We subsample the original 80 reserved questions in each of the 24 frozen
K2P archives. For each size $n_v\in\{20,40,60\}$, we draw 1,000 uniform
subsets without replacement and select by teacher--student agreement,
breaking ties in favor of earlier-generated candidates. We score the
selected prompt using its cached test responses. At $n_v=80$, the full set reproduces the original choice
in all 24 archives. Table~\ref{tab:selection-size} averages over draws
within each archive and then equally over archives. Increasing the
selection set from 20 to 80 questions reduces mean selection loss by
2.03 percentage points; the gain in mean accuracy from 60 to 80 is
0.37 points. On QuaRTz, an additional analysis samples whole background
groups uniformly among subsets totaling exactly $n_v$ questions. Across
its six archives, mean selection loss is 2.53, 2.42, 2.33, and 1.34 points
at $n_v=20,40,60,80$, respectively.

\begin{table}[!htbp]
\centering\small\setlength{\tabcolsep}{7pt}
\begin{tabular}{rrrr}
\toprule
Selection questions & Test accuracy (\%) & Selection loss (pp) & Full-set choice (\%) \\
\midrule
20 & 84.81 & 3.11 & 43.65 \\
40 & 86.02 & 1.90 & 60.07 \\
60 & 86.47 & 1.45 & 75.14 \\
80 & 86.84 & 1.08 & 100.00 \\
\bottomrule
\end{tabular}
\caption{Selection-size sensitivity on 24 frozen K2P archives.
Selection loss is archive-best minus selected test accuracy; full-set
choice is the fraction of selections matching the original 80-question
choice. Candidate archives and their responses are fixed throughout.}
\label{tab:selection-size}
\end{table}

\paragraph{Extending search and selection to 120 questions.}
We study TabMWP and QuaRTz with both students and all three construction
seeds, evaluating on the original 1,000 and 784 test questions,
respectively. The search and reserved pools each extend their original 80 questions
with 40 unused training questions. TabMWP additions use new table groups;
QuaRTz additions use unused questions from that role's existing
background groups, preserving separation between roles. Thus the
QuaRTz extension increases questions within backgrounds. Source remains
fixed at 80, and the original generation settings, agreement criterion,
and generation-order tie rule are retained.

\paragraph{Search-set size.}
For each seed, we retain its eight initial candidates and rerun the
four refinement opportunities using 40 or 120 search questions;
80 reuses the original construction. The 40-question set is a
seed-specific whole-group subset of the original 80, while the 80- and
120-question pools are shared across seeds. Consequently, variation at
40 includes both subset and construction variation. Each branch
recomputes search agreement and parent choices, then selects its final
prompt on the same original 80 reserved questions.
Table~\ref{tab:role-size-summary}(a) shows unchanged accuracy from 80 to
120 for both QuaRTz students in every seed, but lower means for both
TabMWP students. TabMWP/Ministral has substantial variation at 40 and
120: seed 20260715 scores 63.80\% at 40 versus 88.70\% at 80, while
seed 20260609 scores 61.00\% at 120 versus 92.10\% at 80
(Table~\ref{tab:role-size-seeds}(a)). Increasing search size therefore
has no consistent accuracy benefit in these settings.

\paragraph{Selection-set size.}
We freeze each original search-80 archive of twelve candidates and
evaluate agreement on the extended 120-question reserved pool.
For each dataset, we generate 1,000 paired, nested subsets
$\mathcal V_{40}\subset\mathcal V_{80}\subset\mathcal V_{120}$
of 40, 80, and 120 questions, retaining whole groups.
Dynamic programming first samples uniformly from 80-question group
subsets that contain a whole-group 40-question subset. It then samples
uniformly from those 40-question subsets within the sampled 80.
This defines a conditional distribution for the 40-question sets.
The same pairs are used across that dataset's students and archives.
At 120, each archive uses the full extended pool once. All choices are
fixed by agreement before looking up the selected prompts' cached test
scores. Table~\ref{tab:role-size-summary}(b) averages the subset results
within each seed and reports mean and sample SD across seeds; the
original fixed 80-question selection is shown separately.

\begin{table}[!htbp]
\centering\small\setlength{\tabcolsep}{6pt}
\begin{tabular}{lrrrr}
\toprule
\multicolumn{5}{l}{\textbf{(a) Search size; reserved selection fixed at 80}} \\
Setting & 40 & Original 80 & 120 & $\Delta_{80\to120}$ (pp) \\
\midrule
TabMWP/Q & $69.83\pm6.91$ & $68.70\pm4.96$ & $65.80\pm0.46$ & -2.90 \\
TabMWP/M & $82.13\pm15.88$ & $90.83\pm1.86$ & $79.23\pm16.08$ & -11.60 \\
QuaRTz/Q & $82.19\pm1.94$ & $81.59\pm3.38$ & $81.59\pm3.38$ & +0.00 \\
QuaRTz/M & $91.41\pm0.99$ & $91.45\pm1.42$ & $91.45\pm1.42$ & +0.00 \\
\midrule
\multicolumn{5}{l}{\textbf{(b) Selection size; original search-80 archive fixed}} \\
Setting & Resampled 40 & Resampled 80 & Full 120 & Original 80 \\
\midrule
TabMWP/Q & $66.92\pm4.92$ & $67.91\pm5.01$ & $68.70\pm4.96$ & $68.70\pm4.96$ \\
TabMWP/M & $90.12\pm2.22$ & $90.36\pm1.85$ & $90.13\pm1.76$ & $90.83\pm1.86$ \\
QuaRTz/Q & $80.19\pm1.02$ & $80.58\pm1.62$ & $81.59\pm3.38$ & $81.59\pm3.38$ \\
QuaRTz/M & $90.79\pm0.80$ & $90.95\pm0.46$ & $90.73\pm0.90$ & $91.45\pm1.42$ \\
\bottomrule
\end{tabular}
\caption{Search and selection size sensitivity on TabMWP and QuaRTz
(Q: Qwen; M: Ministral). Test accuracy (\%), mean $\pm$ sample SD over
three construction seeds. In (b), each seed contributes its mean over
1,000 paired subsets at 40 and 80, and one full-pool choice at 120;
original 80 reports the unchanged main-evaluation selection set.
$\Delta_{80\to120}$ compares paired search constructions.}
\label{tab:role-size-summary}
\end{table}

Increasing selection from resampled 40 to 80 improves mean selected
test accuracy by 0.16--0.99 points in all four settings. Nine of twelve
seed-level means increase; the SD across subsets within each fixed
archive decreases in eleven of twelve
(Table~\ref{tab:role-size-seeds}(b)). From resampled 80 to full 120,
Qwen gains 0.79 and 1.01 points on TabMWP and QuaRTz, whereas
Ministral loses 0.23 and 0.22 points. Relative to the \emph{original}
fixed 80, full 120 leaves Qwen's means unchanged and reduces
Ministral's by 0.70 and 0.72 points. The added questions therefore have
student-dependent effects. Subset variability is conditional on the
frozen archives and extended pool: for TabMWP/Qwen seed 20260715,
the paired 40-to-80 gain averages 1.48 points, with 5th--95th
percentiles of $-7.50$ to $7.50$ across draws. These percentiles
summarize subset variation within that pool. Together, the results
support retaining the existing 80-question configuration while
showing the different effects of search and selection size.

\begin{table}[!htbp]
\centering\small\setlength{\tabcolsep}{6pt}
\begin{tabular}{llrrrr}
\toprule
\multicolumn{6}{l}{\textbf{(a) Search size; selected test accuracy (\%)}} \\
Setting & Seed & 40 & Original 80 & 120 & $\Delta_{80\to120}$ (pp) \\
\midrule
TabMWP/Q & 0 & 65.40 & 65.40 & 65.40 & +0.00 \\
TabMWP/Q & 20260715 & 66.30 & 66.30 & 66.30 & +0.00 \\
TabMWP/Q & 20260609 & 77.80 & 74.40 & 65.70 & -8.70 \\
\addlinespace[2pt]
TabMWP/M & 0 & 91.20 & 91.70 & 91.40 & -0.30 \\
TabMWP/M & 20260715 & 63.80 & 88.70 & 85.30 & -3.40 \\
TabMWP/M & 20260609 & 91.40 & 92.10 & 61.00 & -31.10 \\
\addlinespace[2pt]
QuaRTz/Q & 0 & 83.93 & 79.21 & 79.21 & +0.00 \\
QuaRTz/Q & 20260715 & 80.10 & 80.10 & 80.10 & +0.00 \\
QuaRTz/Q & 20260609 & 82.53 & 85.46 & 85.46 & +0.00 \\
\addlinespace[2pt]
QuaRTz/M & 0 & 90.31 & 92.73 & 92.73 & +0.00 \\
QuaRTz/M & 20260715 & 92.22 & 89.92 & 89.92 & +0.00 \\
QuaRTz/M & 20260609 & 91.71 & 91.71 & 91.71 & +0.00 \\
\midrule
\multicolumn{6}{l}{\textbf{(b) Selection size; conditional subset variation}} \\
Setting & Seed & Resampled 40 & Resampled 80 & Full 120 & Original 80 \\
\midrule
TabMWP/Q & 0 & $65.10\pm1.36$ & $65.40\pm0.00$ & 65.40 & 65.40 \\
TabMWP/Q & 20260715 & $63.18\pm3.76$ & $64.66\pm3.10$ & 66.30 & 66.30 \\
TabMWP/Q & 20260609 & $72.49\pm3.04$ & $73.68\pm2.02$ & 74.40 & 74.40 \\
\addlinespace[2pt]
TabMWP/M & 0 & $90.03\pm1.06$ & $90.08\pm0.90$ & 89.60 & 91.70 \\
TabMWP/M & 20260715 & $87.96\pm2.11$ & $88.67\pm0.43$ & 88.70 & 88.70 \\
TabMWP/M & 20260609 & $92.39\pm0.69$ & $92.33\pm0.56$ & 92.10 & 92.10 \\
\addlinespace[2pt]
QuaRTz/Q & 0 & $79.20\pm2.02$ & $79.16\pm1.47$ & 79.21 & 79.21 \\
QuaRTz/Q & 20260715 & $80.12\pm1.49$ & $80.23\pm1.22$ & 80.10 & 80.10 \\
QuaRTz/Q & 20260609 & $81.24\pm2.54$ & $82.34\pm2.59$ & 85.46 & 85.46 \\
\addlinespace[2pt]
QuaRTz/M & 0 & $91.37\pm0.95$ & $90.98\pm0.68$ & 90.56 & 92.73 \\
QuaRTz/M & 20260715 & $89.88\pm1.45$ & $90.47\pm1.31$ & 89.92 & 89.92 \\
QuaRTz/M & 20260609 & $91.11\pm0.87$ & $91.39\pm0.67$ & 91.71 & 91.71 \\
\bottomrule
\end{tabular}
\caption{Per-seed results for Table~\ref{tab:role-size-summary}.
In (b), $\pm$ denotes the SD of selected test accuracy across the
1,000 subsets \emph{within one frozen archive}; Table~\ref{tab:role-size-summary}
instead reports SD across the three seed-level means. Full 120 and
original 80 each use one fixed selection set.}
\label{tab:role-size-seeds}
\end{table}

\FloatBarrier

\subsection{Archive and Selection Diagnostics}
\label{app:archive-diagnostic}

We retrospectively evaluate all 288 candidates in the 24 K2P archives
from three construction seeds: eight initial prompts and four valid,
search-evaluated revisions per archive. All 96 revisions are compared
with their actual parents, including proposals that did not improve
search agreement. The diagnostic uses the original test sets: 510
Tracking questions, 1,000 Colored Objects questions, 1,000 TabMWP
questions, and 784 QuaRTz questions. Gold answers score the frozen
prompts; construction decisions and deployed prompts remain fixed.

Figure~\ref{fig:selected-vs-best} compares K2P's selected and best
accuracy, and Table~\ref{tab:archive-selection-loss} reports their
within-seed differences before averaging across seeds.

\begin{table}[!htbp]
\centering\small
\begin{tabular}{lrrrr}
\toprule
Setting & Seed 0 & Seed 20260715 & Seed 20260609 & Mean $\pm$ SD \\
\midrule
Tracking/Q & 2.16 & 2.55 & 0.00 & 1.57 $\pm$ 1.37 \\
Tracking/M & 3.33 & 0.00 & 1.96 & 1.76 $\pm$ 1.68 \\
Colored/Q & 0.00 & 0.00 & 0.00 & 0.00 $\pm$ 0.00 \\
Colored/M & 1.20 & 2.70 & 1.70 & 1.87 $\pm$ 0.76 \\
TabMWP/Q & 0.00 & 0.00 & 0.00 & 0.00 $\pm$ 0.00 \\
TabMWP/M & 0.00 & 0.80 & 1.50 & 0.77 $\pm$ 0.75 \\
QuaRTz/Q & 3.06 & 2.93 & 0.00 & 2.00 $\pm$ 1.73 \\
QuaRTz/M & 0.00 & 2.04 & 0.00 & 0.68 $\pm$ 1.18 \\
\bottomrule
\end{tabular}
\caption{K2P selection loss in percentage points: archive-best minus
reserved-selected test accuracy, computed separately within each run.}
\label{tab:archive-selection-loss}
\end{table}

\paragraph{Search progress and test agreement.}
Table~\ref{tab:revision-transitions} pairs every revision's change in
search agreement with its change in test accuracy. The positive part
of the search change, expressed as a fraction, is the accepted gain
$G_t$ in Section~\ref{sec:global-construction}. Of the 27 parent updates,
21 improve test accuracy and six reduce it. Of the 69 proposals that
retain the parent, 17 improve test accuracy, 50 reduce it, and two tie.

Search optimism concerns the final search parent $p_R$, while reserved
selection determines the deployed prompt. Table~\ref{tab:search-test-agreement}
compares that parent's agreement on search and test questions.
Search agreement is higher in 22 of 24 runs. This comparison includes
sampling variation and differences between search and test populations.

\begin{table}[!htbp]
\centering\small\setlength{\tabcolsep}{4pt}
\begin{tabular}{lrrr}
\toprule
Setting & Search agreement & Test agreement & Difference (pp) \\
\midrule
Tracking/Q & 90.83 $\pm$ 6.17 & 86.54 $\pm$ 4.74 & 4.30 $\pm$ 1.53 \\
Tracking/M & 96.67 $\pm$ 0.72 & 93.33 $\pm$ 3.98 & 3.33 $\pm$ 4.56 \\
Colored/Q & 85.00 $\pm$ 3.31 & 80.23 $\pm$ 1.82 & 4.77 $\pm$ 3.64 \\
Colored/M & 98.75 $\pm$ 1.25 & 96.63 $\pm$ 1.77 & 2.12 $\pm$ 1.03 \\
TabMWP/Q & 71.67 $\pm$ 11.27 & 65.23 $\pm$ 7.57 & 6.43 $\pm$ 3.91 \\
TabMWP/M & 89.58 $\pm$ 4.02 & 90.10 $\pm$ 2.10 & -0.52 $\pm$ 2.19 \\
QuaRTz/Q & 89.17 $\pm$ 1.91 & 82.06 $\pm$ 3.75 & 7.11 $\pm$ 1.85 \\
QuaRTz/M & 97.92 $\pm$ 1.44 & 93.24 $\pm$ 1.02 & 4.68 $\pm$ 1.77 \\
\bottomrule
\end{tabular}
\caption{Final search-parent agreement (\%), mean $\pm$ sample SD over
three seeds. Differences are computed within each seed before aggregation.}
\label{tab:search-test-agreement}
\end{table}

\paragraph{Reference quality and ranking.}
All candidates, seeds, and students share the same diagnostic teacher
reference for each test question. The strict task parser is used
throughout. Invalid or truncated answers do not count as agreements;
the denominator remains the full diagnostic population. Tracking has
504 correct and 510 usable teacher references out of 510; Colored
Objects has 993 correct and 1,000 usable out of 1,000; TabMWP has
980 correct and 982 usable out of 1,000; QuaRTz has 738 correct and
784 usable out of 784. Teacher accuracy is therefore 98.82\%, 99.30\%,
98.00\%, and 94.13\%, respectively.

Agreement and accuracy deficits are taken relative to their respective
maxima within each archive. The diagnostic agreement choice maximizes
agreement on these common questions, breaking ties by generation order.
It attains the archive accuracy maximum in 23 of 24 runs; on seed 0
QuaRTz/Ministral it loses 1.02 points. In 22 archives, every agreement
maximizer also maximizes accuracy. Spearman correlations use tied ranks
and range from 0.880 to 1.000 across the 24 archives. These diagnostic
choices are separate from the original 80-question reserved selection.

\paragraph{Paired uncertainty for seed 0.}
For the eight seed 0 archives, we use 10,000 paired bootstrap draws with a fixed random seed. The unit
is the question for Tracking and Colored Objects, the table group for
TabMWP (992 groups), and the background group for QuaRTz (81 groups).
All candidate predictions and teacher references stay paired when a
group is resampled. Percentile intervals compare the test accuracy of prompts selected from
the full archive and from the initial bank, and compare each revision
with its actual parent. These intervals condition on the realized prompts. Figure~\ref{fig:selection-diagnostic}(b)
instead summarizes variation across constructions with sample standard deviations.

Selection loss is the highest test accuracy in the archive minus that
of the selected prompt. For each paired bootstrap draw, we recompute
every candidate's accuracy difference from the selected prompt and
subtract the corresponding observed difference. We take the largest
absolute deviation across candidates. The 95th percentile of these
maxima is the band radius, which we add to and subtract from the observed
selection loss, clipping the lower endpoint at zero. This accounts for
simultaneous candidate comparisons within each setting.
All eight bands include zero.

\begin{table}[H]
\centering
\small
\begin{tabular}{lrrr}
\toprule
Setting & Selected gain [95\% CI], pp & Selection loss [95\% band], pp & Spearman's $\rho$ \\
\midrule
Tracking/Q & 20.59 [15.88, 25.10] & 2.16 [0.00, 8.63] & 0.998 \\
Tracking/M & 0.00 [0.00, 0.00] & 3.33 [0.00, 7.84] & 0.988 \\
Colored/Q & 37.90 [34.40, 41.50] & 0.00 [0.00, 4.90] & 0.998 \\
Colored/M & -1.20 [-2.40, -0.10] & 1.20 [0.00, 3.50] & 0.979 \\
TabMWP/Q & 0.00 [0.00, 0.00] & 0.00 [0.00, 5.19] & 0.993 \\
TabMWP/M & 28.60 [25.47, 31.79] & 0.00 [0.00, 4.38] & 1.000 \\
QuaRTz/Q & 0.00 [0.00, 0.00] & 3.06 [0.00, 8.15] & 0.956 \\
QuaRTz/M & 0.00 [0.00, 0.00] & 0.00 [0.00, 2.70] & 0.916 \\
\bottomrule
\end{tabular}
\caption{Seed 0 conditional uncertainty and agreement--accuracy rank correlation.
Selected gain compares full-archive selection with initial-only selection.}
\label{tab:archive-uncertainty}
\end{table}

\begingroup\footnotesize\setlength{\tabcolsep}{3pt}
\begin{longtable}{llccrrrr}
\caption{All 96 revisions relative to their actual parents. Search and
test columns report agreement and accuracy changes (pp); repairs and
regressions count test questions. ``Update'' denotes a strict search
agreement increase. I and R index initial candidates and revisions.}
\label{tab:revision-transitions}\\
\toprule
Setting & Seed & Parent $\to$ revision & Update & Search & Repairs & Regressions & Test \\
\midrule
\endfirsthead
\toprule
Setting & Seed & Parent $\to$ revision & Update & Search & Repairs & Regressions & Test \\
\midrule
\endhead
Tracking/Q & 0 & I6 $\to$ R1 & Yes & 5.00 & 103 & 38 & 12.75 \\*
 &  & R1 $\to$ R2 & Yes & 3.75 & 46 & 26 & 3.92 \\*
 &  & R2 $\to$ R3 & Yes & 6.25 & 32 & 43 & -2.16 \\*
 &  & R3 $\to$ R4 & No & -2.50 & 38 & 38 & 0.00 \\
\addlinespace[2pt]
Tracking/Q & 20260715 & I7 $\to$ R1 & No & -1.25 & 96 & 78 & 3.53 \\*
 &  & I7 $\to$ R2 & Yes & 15.00 & 121 & 38 & 16.27 \\*
 &  & R2 $\to$ R3 & No & -2.50 & 26 & 29 & -0.59 \\*
 &  & R2 $\to$ R4 & No & -1.25 & 41 & 28 & 2.55 \\
\addlinespace[2pt]
Tracking/Q & 20260609 & I7 $\to$ R1 & No & -17.50 & 73 & 119 & -9.02 \\*
 &  & I7 $\to$ R2 & No & -3.75 & 92 & 106 & -2.75 \\*
 &  & I7 $\to$ R3 & Yes & 6.25 & 117 & 43 & 14.51 \\*
 &  & R3 $\to$ R4 & Yes & 6.25 & 48 & 70 & -4.31 \\
\addlinespace[2pt]
Tracking/M & 0 & I8 $\to$ R1 & No & 0.00 & 22 & 9 & 2.55 \\*
 &  & I8 $\to$ R2 & No & -21.25 & 22 & 87 & -12.75 \\*
 &  & I8 $\to$ R3 & Yes & 1.25 & 24 & 7 & 3.33 \\*
 &  & R3 $\to$ R4 & No & 0.00 & 3 & 5 & -0.39 \\
\addlinespace[2pt]
Tracking/M & 20260715 & I7 $\to$ R1 & No & -3.75 & 27 & 30 & -0.59 \\*
 &  & I7 $\to$ R2 & No & -6.25 & 27 & 28 & -0.20 \\*
 &  & I7 $\to$ R3 & No & -6.25 & 23 & 33 & -1.96 \\*
 &  & I7 $\to$ R4 & No & -6.25 & 27 & 33 & -1.18 \\
\addlinespace[2pt]
Tracking/M & 20260609 & I1 $\to$ R1 & No & -16.25 & 37 & 82 & -8.82 \\*
 &  & I1 $\to$ R2 & No & -2.50 & 41 & 36 & 0.98 \\*
 &  & I1 $\to$ R3 & No & -8.75 & 42 & 40 & 0.39 \\*
 &  & I1 $\to$ R4 & No & -35.00 & 20 & 212 & -37.65 \\
\addlinespace[2pt]
Colored/Q & 0 & I4 $\to$ R1 & Yes & 7.50 & 274 & 77 & 19.70 \\*
 &  & R1 $\to$ R2 & No & -1.25 & 101 & 221 & -12.00 \\*
 &  & R1 $\to$ R3 & Yes & 3.75 & 123 & 115 & 0.80 \\*
 &  & R3 $\to$ R4 & No & -2.50 & 114 & 77 & 3.70 \\
\addlinespace[2pt]
Colored/Q & 20260715 & I2 $\to$ R1 & No & -11.25 & 143 & 130 & 1.30 \\*
 &  & I2 $\to$ R2 & No & 0.00 & 248 & 120 & 12.80 \\*
 &  & I2 $\to$ R3 & No & -20.00 & 148 & 219 & -7.10 \\*
 &  & I2 $\to$ R4 & Yes & 18.75 & 316 & 86 & 23.00 \\
\addlinespace[2pt]
Colored/Q & 20260609 & I6 $\to$ R1 & No & -3.75 & 199 & 185 & 1.40 \\*
 &  & I6 $\to$ R2 & Yes & 5.00 & 289 & 124 & 16.50 \\*
 &  & R2 $\to$ R3 & No & -8.75 & 122 & 250 & -12.80 \\*
 &  & R2 $\to$ R4 & Yes & 16.25 & 181 & 109 & 7.20 \\
\addlinespace[2pt]
Colored/M & 0 & I1 $\to$ R1 & No & -2.50 & 16 & 28 & -1.20 \\*
 &  & I1 $\to$ R2 & No & -6.25 & 11 & 36 & -2.50 \\*
 &  & I1 $\to$ R3 & No & 0.00 & 16 & 36 & -2.00 \\*
 &  & I1 $\to$ R4 & No & -1.25 & 12 & 24 & -1.20 \\
\addlinespace[2pt]
Colored/M & 20260715 & I4 $\to$ R1 & No & -2.50 & 49 & 28 & 2.10 \\*
 &  & I4 $\to$ R2 & No & -3.75 & 41 & 72 & -3.10 \\*
 &  & I4 $\to$ R3 & No & -5.00 & 39 & 54 & -1.50 \\*
 &  & I4 $\to$ R4 & No & -3.75 & 44 & 74 & -3.00 \\
\addlinespace[2pt]
Colored/M & 20260609 & I3 $\to$ R1 & No & 0.00 & 46 & 51 & -0.50 \\*
 &  & I3 $\to$ R2 & Yes & 5.00 & 50 & 20 & 3.00 \\*
 &  & R2 $\to$ R3 & No & -3.75 & 18 & 31 & -1.30 \\*
 &  & R2 $\to$ R4 & No & -12.50 & 12 & 130 & -11.80 \\
\addlinespace[2pt]
TabMWP/Q & 0 & I5 $\to$ R1 & No & -48.75 & 64 & 466 & -40.20 \\*
 &  & I5 $\to$ R2 & No & -47.50 & 51 & 466 & -41.50 \\*
 &  & I5 $\to$ R3 & No & -50.00 & 59 & 494 & -43.50 \\*
 &  & I5 $\to$ R4 & No & -42.50 & 74 & 433 & -35.90 \\
\addlinespace[2pt]
TabMWP/Q & 20260715 & I1 $\to$ R1 & No & -27.50 & 106 & 379 & -27.30 \\*
 &  & I1 $\to$ R2 & No & -30.00 & 90 & 376 & -28.60 \\*
 &  & I1 $\to$ R3 & No & -33.75 & 110 & 343 & -23.30 \\*
 &  & I1 $\to$ R4 & No & -28.75 & 85 & 392 & -30.70 \\
\addlinespace[2pt]
TabMWP/Q & 20260609 & I5 $\to$ R1 & Yes & 8.75 & 215 & 104 & 11.10 \\*
 &  & R1 $\to$ R2 & No & -15.00 & 114 & 245 & -13.10 \\*
 &  & R1 $\to$ R3 & Yes & 13.75 & 141 & 119 & 2.20 \\*
 &  & R3 $\to$ R4 & No & -11.25 & 115 & 192 & -7.70 \\
\addlinespace[2pt]
TabMWP/M & 0 & I3 $\to$ R1 & Yes & 27.50 & 304 & 39 & 26.50 \\*
 &  & R1 $\to$ R2 & No & -1.25 & 61 & 65 & -0.40 \\*
 &  & R1 $\to$ R3 & No & -13.75 & 44 & 154 & -11.00 \\*
 &  & R1 $\to$ R4 & Yes & 2.50 & 67 & 46 & 2.10 \\
\addlinespace[2pt]
TabMWP/M & 20260715 & I1 $\to$ R1 & No & -7.50 & 134 & 90 & 4.40 \\*
 &  & I1 $\to$ R2 & Yes & 16.25 & 280 & 41 & 23.90 \\*
 &  & R2 $\to$ R3 & Yes & 6.25 & 102 & 33 & 6.90 \\*
 &  & R3 $\to$ R4 & No & -3.75 & 39 & 47 & -0.80 \\
\addlinespace[2pt]
TabMWP/M & 20260609 & I1 $\to$ R1 & Yes & 30.00 & 345 & 14 & 33.10 \\*
 &  & R1 $\to$ R2 & Yes & 2.50 & 38 & 34 & 0.40 \\*
 &  & R2 $\to$ R3 & No & -1.25 & 38 & 26 & 1.20 \\*
 &  & R2 $\to$ R4 & Yes & 1.25 & 40 & 25 & 1.50 \\
\addlinespace[2pt]
QuaRTz/Q & 0 & I6 $\to$ R1 & No & -3.75 & 61 & 75 & -1.79 \\*
 &  & I6 $\to$ R2 & No & -11.25 & 75 & 148 & -9.31 \\*
 &  & I6 $\to$ R3 & No & -12.50 & 60 & 117 & -7.27 \\*
 &  & I6 $\to$ R4 & Yes & 2.50 & 63 & 87 & -3.06 \\
\addlinespace[2pt]
QuaRTz/Q & 20260715 & I5 $\to$ R1 & No & -6.25 & 101 & 59 & 5.36 \\*
 &  & I5 $\to$ R2 & No & -6.25 & 85 & 77 & 1.02 \\*
 &  & I5 $\to$ R3 & No & -20.00 & 75 & 132 & -7.27 \\*
 &  & I5 $\to$ R4 & No & -31.25 & 65 & 189 & -15.82 \\
\addlinespace[2pt]
QuaRTz/Q & 20260609 & I6 $\to$ R1 & No & -1.25 & 57 & 93 & -4.59 \\*
 &  & I6 $\to$ R2 & No & -1.25 & 54 & 76 & -2.81 \\*
 &  & I6 $\to$ R3 & Yes & 3.75 & 57 & 68 & -1.40 \\*
 &  & R3 $\to$ R4 & No & -11.25 & 76 & 71 & 0.64 \\
\addlinespace[2pt]
QuaRTz/M & 0 & I2 $\to$ R1 & No & -1.25 & 23 & 40 & -2.17 \\*
 &  & I2 $\to$ R2 & Yes & 2.50 & 22 & 26 & -0.51 \\*
 &  & R2 $\to$ R3 & No & -1.25 & 22 & 25 & -0.38 \\*
 &  & R2 $\to$ R4 & No & -2.50 & 27 & 28 & -0.13 \\
\addlinespace[2pt]
QuaRTz/M & 20260715 & I7 $\to$ R1 & No & -2.50 & 58 & 58 & 0.00 \\*
 &  & I7 $\to$ R2 & No & 0.00 & 58 & 28 & 3.83 \\*
 &  & I7 $\to$ R3 & Yes & 1.25 & 47 & 27 & 2.55 \\*
 &  & R3 $\to$ R4 & No & -1.25 & 16 & 28 & -1.53 \\
\addlinespace[2pt]
QuaRTz/M & 20260609 & I2 $\to$ R1 & No & -1.25 & 26 & 50 & -3.06 \\*
 &  & I2 $\to$ R2 & No & -5.00 & 25 & 47 & -2.81 \\*
 &  & I2 $\to$ R3 & Yes & 1.25 & 22 & 30 & -1.02 \\*
 &  & R3 $\to$ R4 & No & -3.75 & 32 & 31 & 0.13 \\
\bottomrule
\end{longtable}
\endgroup

\begin{figure}[!htbp]
\centering
\includegraphics[width=\linewidth]{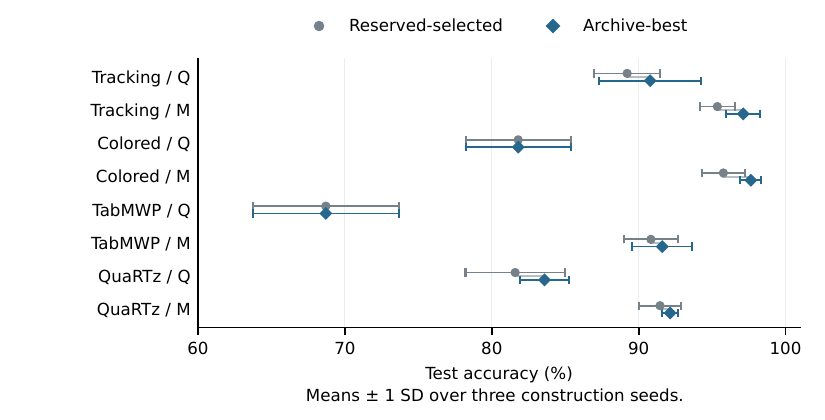}
\caption{K2P reserved-selected and archive-best test accuracy across
three construction seeds (Q: Qwen; M: Ministral). Points and bars show
means and sample standard deviations. Their within-run difference is
the selection loss in Figure~\ref{fig:selection-diagnostic}(b).}
\label{fig:selected-vs-best}
\end{figure}
\FloatBarrier

\section{Theoretical Analysis}
\label{app:theory}

Throughout this appendix, $\gamma$ denotes agreement--accuracy
discrimination, $\kappa$ the empirical progress rate, and $d_t$ the
parent's empirical agreement deficit. The identification margin is
$\Delta_{\rm ag}$, archive contraction uses $\theta$, and the scoring
correction is $\chi(p)$. Search and reserved quantities use subscripts
$s$ and $v$, respectively.

Appendix~\ref{app:robust-global-proof} proves the global agreement
bound and its sampling and progress conditions.
Appendix~\ref{app:discrimination} translates it into accuracy through
the discrimination condition, derives the noisy-reference example,
and states resource conditions for consistency and exact attainment.
Appendix~\ref{app:numerical-bias} gives a complementary comparison using
score offsets. Appendix~\ref{app:finite-repair} gives a finite-rule
specialization with repair and damage.
Appendices~\ref{app:finite-rule} and~\ref{app:finite-mass}
give additional results under rule-preserving proposals.
The subsequent sections analyze relative agreement bias and
population-level archive progress, with complete statements and proofs.

\subsection{Construction Agreement and Accuracy Guarantees}
\label{app:robust-global-proof}

\subsubsection{Evaluation law and averaged progress shortfalls}

Fix the nonempty admissible space $\Pi$ from Equation 1 before
construction data are observed. Every initial candidate and valid,
fully evaluated revision belongs to $\Pi$. Population scores $J(p)$
and $\widetilde J(p)$ use the actual student and the stated task and agreement
policies.

For analysis, define a common joint law of the cached search scores
$\{\widehat J_s(p):p\in\Pi\}$, including scores of prompts that
are never queried. This is a probabilistic representation of fixed
evaluation policies, not a requirement to run additional evaluations.
It can use common question-level randomness across prompts, or
independent evaluation randomness, provided that for every fixed $p$,
$\mathbb E[\widehat J_s(p)]=\widetilde J(p)$. The algorithm observes only queried
scores and reuses them. An adaptive choice of a prompt need not have
an unbiased score. A valid but incompletely evaluated proposal is not
assigned a completed score by the algorithm and has $G_t=0$.
The countable space of length-bounded token strings makes the suprema
measurable; one may also use a finite fixed family.

Let $\mathcal F^{\rm prop}_{t-1}$ contain the construction information
available after the feedback batch is formed and before slot $t$'s
proposal randomness is drawn, excluding reserved observations.
Set
\[
\bar d_{t-1}=\mathbb E[d_{t-1}\mid\mathcal F^{\rm prop}_{t-1}],
\qquad
m_t=\mathbb E[G_t\mid\mathcal F^{\rm prop}_{t-1}],
\qquad
s_t=[\kappa\bar d_{t-1}-m_t]_+.
\]
A canonical choice is
$\varepsilon_t=\mathbb E[s_t]$; any deterministic upper bound on
this expectation is also valid. Conditional expectation is used for
$d_{t-1}$ because unqueried cached-score randomness need not be known
in $\mathcal F^{\rm prop}_{t-1}$. With deterministic scores measurable from the
search inputs, $\bar d_{t-1}=d_{t-1}$.
Taking expectations of $m_t\ge\kappa\bar d_{t-1}-s_t$ gives
\[
\mathbb E[G_t]\ge\kappa\mathbb E[d_{t-1}]-\varepsilon_t.
\]
This construction expresses proposal effectiveness as a shortfall
from the reference rate. The repair--damage conditions below provide
one way to bound its magnitude.

For clarity, the empirical transition quantities for a completed
proposal are
\[
\widehat U_t=\frac1{n_s}\sum_i
(1-\widetilde c_i(p_{t-1}))\widetilde c_i(q_t),\qquad
\widehat L_t=\frac1{n_s}\sum_i
\widetilde c_i(p_{t-1})(1-\widetilde c_i(q_t)).
\]
Their difference is $\widehat J_s(q_t)-\widehat J_s(p_{t-1})$ for
arbitrary binary agreement scores, including invalid or truncated
responses scored as zero. The identity
$G_t=[\widehat U_t-\widehat L_t]_+$ therefore uses no correctness
assumption. Failed, incomplete, or skipped slots have all three
transition quantities zero.

\subsubsection{Proof of the global bound}

\begin{proof}[Proof of Theorem~\ref{thm:robust-global}]
\textbf{Empirical optimization.}
Strict empirical improvement updates the parent, and all other cases
leave it fixed. Thus
$d_t=d_{t-1}-G_t\le d_{t-1}$ pathwise.
The averaged shortfall inequality gives
\[
\mathbb E[d_t]\le
(1-\kappa)\mathbb E[d_{t-1}]+\varepsilon_t.
\]
We take empty sums as zero, empty products as one, and
$(1-\kappa)^0=1$. Iteration and the pathwise bound $d_R\le d_0$ yield
$\mathbb E[d_R]\le B_R$, including $R=0$.
No independence of proposals, repairs, or damages is required.

\textbf{Population agreement of the parent.}
For every fixed $p\in\Pi$,
$\widehat J_{\Pi,s}^*\ge\widehat J_s(p)$.
Unbiasedness and then a supremum over $p$ imply
$\mathbb E[\widehat J_{\Pi,s}^*]\ge \widetilde J_\Pi^*$.
Consequently,
\[
\begin{aligned}
\mathbb E[\widetilde J_\Pi^*-\widetilde J(p_R)]
={}&\widetilde J_\Pi^*-\mathbb E[\widehat J_{\Pi,s}^*]
+\mathbb E[d_R]\\
&+\mathbb E[\widehat J_s(p_R)-\widetilde J(p_R)]
\le B_R+\Gamma_s.
\end{aligned}
\]
The term $\Gamma_s$ accounts for the optimism introduced by
choosing the parent on the reused search observations.

\textbf{Independent reserved selection.}
Condition on the complete construction history, excluding all reserved
observations, so that $\mathcal C$ is fixed. Let $p^{\widetilde J}$ maximize $\widetilde J$
over $\mathcal C$, using a fixed tie rule for this analytical
comparison. Set $Z_p=\widehat J_v(p)-\widetilde J(p)$ using
the reserved score defined in Section~\ref{sec:k2p-agreement}.
Equation 4 gives
\[
\max_{p\in\mathcal C}\widetilde J(p)-\widetilde J(p_S^*)
\le\max_{p\in\mathcal C}Z_p-Z_{p^{\widetilde J}}.
\]
Each $Z_p$ has conditional mean zero. Hoeffding's lemma
(\citealp{hoeffding1963}) gives the conditional moment bound
$\mathbb E[e^{\lambda Z_p}\mid\mathcal F_{\rm con}]
\le e^{\lambda^2/(8n_v)}$ for $\lambda>0$, where
$H_{\rm con}$ denotes the complete construction history and
$\mathcal F_{\rm con}$ its generated sigma-field.
Dependence of different prompts' scores on the same questions and
teacher references is allowed. The log-sum-exp bound yields
\[
\mathbb E[\max_{p\in\mathcal C}Z_p\mid\mathcal F_{\rm con}]
\le\frac{\log K}{\lambda}+\frac{\lambda}{8n_v}.
\]
Minimizing over $\lambda$ gives $\sqrt{\log K/(2n_v)}$; if $K=1$,
selection loss is zero directly.
Since $p_R\in\mathcal C$,
\[
\mathbb E[\widetilde J_\Pi^*-\widetilde J(p_S^*)]
\le B_R+\Gamma_s+\sqrt{\frac{\log K}{2n_v}}.
\]
The best fixed candidate's centered error has expectation zero.
This explains the absence of a factor of two in this expectation
argument, whereas Equation~\eqref{eq:selection-comparison} uses a realized uniform absolute
deviation.

\end{proof}

\subsubsection{Discrimination, noisy references, and consistency}
\label{app:discrimination}

\paragraph{From agreement to accuracy.}
Suppose deterministic constants $\gamma>0$ and $\omega\ge0$ satisfy
\[
\gamma[J_\Pi^*-J(p)]\le \widetilde J_\Pi^*-\widetilde J(p)+\omega
\quad\text{for every }p\in\Pi.
\]
The constants describe the relation between population scores over
the fixed admissible space. In particular, when $\omega=0$, every
agreement maximizer is an accuracy maximizer, and a small agreement
deficit controls the accuracy deficit even if the optima are suprema.

\begin{proof}[Proof of Corollary~\ref{cor:discrimination}]
Apply the condition to the random output $p_S^*\in\Pi$ and take
expectations. Theorem~\ref{thm:robust-global}
gives
\[
\gamma\,\mathbb E[J_\Pi^*-J(p_S^*)]
\le\mathbb E[\widetilde J_\Pi^*-\widetilde J(p_S^*)]+\omega
\le Q_R+\omega.
\]
Divide by $\gamma$ and use $0\le J_\Pi^*-J(p_S^*)\le1$.
For fixed $\gamma>0$, $\omega=0$, and $Q_R\to0$, the expected
accuracy gap therefore tends to zero. For each $a>0$, Markov's
inequality also gives
\[
\Pr\{J_\Pi^*-J(p_S^*)\ge a\}
\le\min\left\{1,\frac{Q_R+\omega}{\gamma a}\right\}.
\]
Thus convergence of the expected gap implies convergence in
probability to the optimal value.
\end{proof}

\paragraph{An affine sufficient condition with residual variation.}
Suppose, for all $p\in\Pi$,
\[
\widetilde J(p)=c+\gamma J(p)+r(p),\qquad \gamma>0,
\qquad \sup_{p\in\Pi}r(p)-\inf_{p\in\Pi}r(p)\le\omega.
\]
Taking a sequence of prompts whose accuracy approaches $J_\Pi^*$
shows that $\widetilde J_\Pi^*\ge c+\gamma J_\Pi^*+\inf_\Pi r$.
Consequently, for every $p$,
\[
\widetilde J_\Pi^*-\widetilde J(p)
\ge\gamma[J_\Pi^*-J(p)]+\inf_\Pi r-r(p)
\ge\gamma[J_\Pi^*-J(p)]-\omega.
\]
This proves the discrimination condition from a relation between the
two scores. For instance, a uniform deviation $|r(p)|\le\tau$
permits $\omega=2\tau$. Both a common offset and positive rescaling
are compatible with zero residual variation.

\paragraph{A binary teacher with independent errors.}
Let $Y\in\{0,1\}$ be the true answer and $S_p\in\{0,1\}$ the
student's actual answer under prompt $p$. All answers in this model
are valid and untruncated, and correctness is $\mathbf1\{S_p=Y\}$.
The student may be stochastic and may err under every admissible
prompt. Let the teacher reference be $\widetilde Y=Y\mathbin\oplus E$.
Assume $\Pr(E=1\mid X,Y)=\eta_T\in[0,1/2)$. For each fixed prompt $p$,
$E$ is independent of $S_p$ conditional on $X,Y$.
Conditioning on $X,Y$ and then taking expectations gives
\[
\begin{aligned}
\Pr(S_p=\widetilde Y\mid X,Y)
&=\eta_T+(1-2\eta_T)\Pr(S_p=Y\mid X,Y),\\
\widetilde J(p)&=\eta_T+(1-2\eta_T)J(p).
\end{aligned}
\]
Positive slope gives
\[
\widetilde J_\Pi^*-\widetilde J(p)=(1-2\eta_T)[J_\Pi^*-J(p)].
\]
Hence Equation~\eqref{eq:discrimination} holds with
$\gamma=1-2\eta_T$ and $\omega=0$. At $\eta_T=0.2$, the slope is $0.6$,
so Equation 8 yields an accuracy-gap bound of $\min\{1,Q_R/0.6\}$.
The reference-noise model supplies the score relation, while the
progress and sampling conditions determine $Q_R$.

\paragraph{Progress and joint resource growth.}
Let $a=1-\kappa\in[0,1)$. For a fixed construction process indexed
by refinement slot, with deterministic shortfall bounds $\varepsilon_t\to0$,
\[
B_R\le a^R D_0+\sum_{t=1}^R a^{R-t}\varepsilon_t\longrightarrow0.
\]
To see this, choose $T$ such that $\varepsilon_t\le\delta$ for
$t>T$. The finitely many terms up to $T$ vanish as $R\to\infty$,
and the remaining sum is at most $\delta/\kappa$. Letting
$\delta\downarrow0$ proves the claim. For $\kappa=1$, the sum is
simply $\varepsilon_R$ once $R\ge1$.

For a sequence of configured runs indexed by $\ell$, with changing
search sample sizes and $R_\ell\to\infty$, one explicit uniform
condition is
\[
\varepsilon_{t,\ell}\le\bar\varepsilon_t+r_\ell,
\qquad \bar\varepsilon_t\to0,\qquad r_\ell\to0,
\qquad 1\le t\le R_\ell,
\]
with a fixed $\kappa\in(0,1]$. Since $D_{0,\ell}\le1$,
\[
B_{R_\ell}\le
a^{R_\ell}+\sum_{t=1}^{R_\ell}a^{R_\ell-t}\bar\varepsilon_t
+\frac{r_\ell}{\kappa}\longrightarrow0.
\]
Together with $\Gamma_{s,\ell}\to0$ and
$\log K_\ell/n_{v,\ell}\to0$, this gives $Q_{R_\ell}\to0$.
Under independent-group evaluation, the reserved condition becomes
$v_{v,\ell}\log K_\ell\to0$. For a fixed finite prompt space of
size $N$, the search bound below tends to zero as $n_{s,\ell}\to\infty$,
or as $v_{s,\ell}\to0$ in the grouped version.

If instead $\varepsilon_t\le\bar\varepsilon_{\rm ref}$, summing
the geometric weights gives
\[
B_R\le\min\left\{D_0,\,
a^R D_0+\frac{\bar\varepsilon_{\rm ref}}{\kappa}(1-a^R)\right\}.
\]
With uniform bounds across runs, fixed discrimination constants,
and vanishing sampling errors, Equation 8 gives
\[
\limsup\mathbb E[J_\Pi^*-J(p_S^*)]
\le\min\left\{1,
\frac{\bar\varepsilon_{\rm ref}/\kappa+\omega}{\gamma}\right\}.
\]
This upper bound shows how persistent progress shortfalls and
residual mismatch near the optimum affect the accuracy guarantee.

\paragraph{Exact attainment in a finite class.}
If $\Pi$ is finite and contains a suboptimal prompt, define the
positive accuracy gap
\[
\Delta_J=\min_{p\in\Pi:\,J(p)<J_\Pi^*}
\{J_\Pi^*-J(p)\}>0.
\]
Every suboptimal output has deficit at least $\Delta_J$, so
\[
\Pr\{J(p_S^*)<J_\Pi^*\}
\le\min\left\{1,\frac{Q_R+\omega}{\gamma\Delta_J}\right\}.
\]
For fixed $\gamma>0$, $\omega=0$, and $Q_R\to0$, the probability
of selecting a globally optimal prompt tends to one. If all prompts
are optimal, this probability is one directly. The benchmark
$J_\Pi^*$ can be strictly below one. For an infinite space without a
positive gap, the preceding expected-gap and approximation-in-probability
statements provide the corresponding conclusion.

\subsubsection{A complementary bound using score offsets}
\label{app:numerical-bias}

Numerical differences between agreement and accuracy give another
way to bound the final accuracy gap. Define
\[
b(p)=J(p)-\widetilde J(p),\qquad
\rho_\Pi=\sup_{p\in\Pi}b(p)-\inf_{p\in\Pi}b(p).
\]
The span $\rho_\Pi$ measures variation in score offsets across
admissible prompts. Under Theorem~\ref{thm:robust-global}, it gives
\[
\mathbb E[J_\Pi^*-J(p_S^*)]
\le\min\{1,Q_R+\rho_\Pi\}.
\]
For every $q\in\Pi$,
\[
\begin{aligned}
J_\Pi^*-J(q)
&\le \widetilde J_\Pi^*-\widetilde J(q)+\sup_{p\in\Pi}b(p)-b(q)\\
&\le \widetilde J_\Pi^*-\widetilde J(q)+\rho_\Pi.
\end{aligned}
\]
The inequalities hold for suprema. Taking expectations at the final
prompt and using Equation 6 proves the bound, with accuracy loss
capped at one. This comparison introduces the bias span only once.

\paragraph{Numerical bias and preservation of the optimum.}
A positive affine relation $\widetilde J=c+\gamma J$ preserves every ranking,
while $b=(1-\gamma)J-c$ can vary across prompts. Thus the numerical
span can be positive even when the agreement optimum identifies the
accuracy optimum. In the binary reference-noise example with
$\eta_T=0.2$, if $\sup_\Pi J=0.85$ and $\inf_\Pi J=0.55$, then
\[
b(p)=0.4J(p)-0.2,\qquad
\rho_\Pi=0.4(0.85-0.55)=0.12.
\]
The additive bound retains $0.12$ as $Q_R\to0$, while
Equation~\eqref{eq:calibrated-global} tends to zero because
$\gamma=0.6$ and $\omega=0$. The discrimination condition uses the
relation between deficits from the two optima to accommodate such
changes of scale. A numerical span bound remains useful when the
available information instead directly controls $J-\widetilde J$.

With shortfalls bounded by $\bar\varepsilon_{\rm ref}$ uniformly
across runs and vanishing sampling errors, the additive comparison
also gives
\[
\limsup\mathbb E[J_\Pi^*-J(p_S^*)]
\le\min\left\{1,\rho_\Pi+
\frac{\bar\varepsilon_{\rm ref}}{\kappa}\right\}.
\]
This is the residual upper bound from the comparison using score offsets.

\paragraph{Proxy error and its interpretation.}
If teacher and student answers are valid, policies are matched, and
agreement with a correct teacher reference equals task correctness,
let $\varepsilon_T=\Pr\{a(T(x))\ne y\}$ be the teacher error probability.
Then $|J(p)-\widetilde J(p)|\le\varepsilon_T$ for every fixed prompt.
This follows pointwise because the two indicators can differ only
on a teacher error. Hence $\rho_\Pi\le2\varepsilon_T$.
With validity or scoring mismatches, use the exact $b=J-\widetilde J$ instead.
A bias shared across all prompts gives $\rho_\Pi=0$ even
when $\widetilde J(p)\ne J(p)$.

Student limitations are included in $J_\Pi^*$ and in every prompt's
actual score. The comparison therefore measures the accuracy
recoverable through prompt choice for that student. The global proxy
span is taken over $\Pi$; the archive analysis below uses the span
over the particular candidates retained in $\mathcal C$.

\subsubsection{Sampling bounds, difficult states, and stopping}

\paragraph{Finite families and independent groups.}
Suppose $N=|\Pi|<\infty$ and search scores average $n_s$
independent, population-matched observations. For
$Z_{s,p}=\widehat J_s(p)-\widetilde J(p)$, the same moment argument gives
\[
\Gamma_s
\le\mathbb E\left[\max_{p\in\Pi}Z_{s,p}\right]
\le\sqrt{\frac{\log N}{2n_s}}.
\]
The expectation of the maximum is nonnegative because every fixed
prompt's score error has mean zero. The maximum ranges over the
space $\Pi$ fixed before search, so this bound depends on its
size $N$. For $N=1$, $\Gamma_s=0$.

More generally, let the search and reserved scores average independent
groups with fixed nonnegative weights $w_{s,g}$ and $w_{v,g}$ summing
to one. Within-group questions may be dependent. Assume the weighted
expected scores equal $\widetilde J(p)$ for each fixed prompt, and reserved
groups and evaluation randomness are independent of construction.
Writing $v_s=\sum_gw_{s,g}^2$ and $v_v=\sum_gw_{v,g}^2$, weighted
Hoeffding gives
\[
\Gamma_s\le\sqrt{\frac{v_s\log N}{2}},
\qquad
\mathbb E[\max_{\mathcal C}\widetilde J-\widetilde J(p_S^*)]
\le\sqrt{\frac{v_v\log K}{2}}.
\]
Use the latter term in Equation 6 and the same weights for empirical
progress. Question-weighted scoring gives group weights proportional
to group sizes. Random sizes require conditioning that preserves
group independence and population matching, as in
Appendix~\ref{app:statistical-units}.

\paragraph{A repair--damage sufficient condition with difficult states.}
Let $\mathcal E_t^{\rm bad}\in\mathcal F^{\rm prop}_{t-1}$ designate difficult feedback or
proposal states. Outside $\mathcal E_t^{\rm bad}$, suppose
\[
\begin{aligned}
\mathbb E[\widehat U_t\mid\mathcal F^{\rm prop}_{t-1}]
&\ge\alpha\bar d_{t-1}-u_t,\\
\mathbb E[\widehat L_t\mid\mathcal F^{\rm prop}_{t-1}]
&\le\beta\bar d_{t-1}+h_t,
\end{aligned}
\qquad
\kappa=\alpha-\beta>0,\quad \alpha\le1,\quad \beta,u_t,h_t\ge0.
\]
There is no progress assumption on $\mathcal E_t^{\rm bad}$.
Because $[x]_+\ge x$, on its complement
$m_t\ge\kappa\bar d_{t-1}-(u_t+h_t)$.
Everywhere, $m_t\ge0$ and $0\le\bar d_{t-1}\le1$.
Thus, with $\nu_t=u_t+h_t$,
\[
\mathbb E[s_t]
\le\nu_t\Pr((\mathcal E_t^{\rm bad})^c)
+\kappa\mathbb E[\bar d_{t-1}\mathbf1_{\mathcal E_t^{\rm bad}}]
\le\nu_t+\kappa\Pr(\mathcal E_t^{\rm bad}).
\]
A bound $\Pr(\mathcal E_t^{\rm bad})\le\pi_t^{\rm bad}$ therefore permits
$\varepsilon_t=\nu_t+\kappa\pi_t^{\rm bad}$.
This combines repair and damage in productive states with the
probability of difficult states, including histories with zero
expected gain. The event is defined from pre-proposal information.

Repeated proposals at fixed feedback states estimate $m_t$ and the
transition moments using label-free scores. In a controlled finite
task with a computable optimum over $\Pi$, they also permit
measurement of $d_{t-1}$ and progress shortfalls. Aggregation over
construction histories then targets $\mathbb E[s_t]$.

\paragraph{No-progress and stopping cases.}
If $G_t=0$ almost surely at every slot, then $d_t=d_0$.
The average progress condition requires
$\varepsilon_t\ge\kappa D_0$ for each $t$. Consequently, the
second entry of the minimum defining $B_R$ is at least
$(1-\kappa)^R D_0+D_0[1-(1-\kappa)^R]=D_0$,
so $B_R=D_0$ for every $R$.
On a stopped history, $m_t=0$ and $s_t=\kappa\bar d_{t-1}$.
The remaining slots retain this shortfall. Invalid proposals and
incomplete evaluations enter in the same way. Almost-sure completion
of initial construction and the prescribed reserved comparison
ensures that the final output is defined under the sampling law
used in the theorem.

For deterministic reference rates $\kappa_t\in[0,1]$ and deterministic
averaged shortfall bounds, the general recursion is
\[
\mathbb E[d_R]\le
D_0\prod_{t=1}^R(1-\kappa_t)
+\sum_{t=1}^R\varepsilon_t\prod_{i=t+1}^R(1-\kappa_i).
\]
It may again be capped by $D_0$. A deterministically skipped slot can
use $\kappa_t=\varepsilon_t=0$. Deterministic rates place the
coefficients outside the expectations in each recursion step.

\subsubsection{A finite-rule specialization with repair and damage}
\label{app:finite-repair}

In this idealized task model, every admissible prompt acts through
a subset of task rules, and every subset has an admissible
representation. Multiple prompt texts may have the same behavior.
Correct references and the execution law below give $\widetilde J=J$, so
Equation~\eqref{eq:discrimination} holds with $\gamma=1,\omega=0$.
The full-rule prompt establishes $J_\Pi^*=1$ and
$\widehat J_{\Pi,s}^*=1$, so $d_t=e_t$.
The conditional moments permit $\varepsilon_t=u+h$, and mode
frequencies provide a structural bound on $\Gamma_s$. These quantities
instantiate Theorem~\ref{thm:robust-global} and
Corollary~\ref{cor:discrimination} in the finite-rule regime.

\paragraph{A finite-rule task model.}
Inputs belong to $d$ disjoint task modes with probabilities $\pi_j>0$,
summing to one. Each mode has an executable rule, and the frozen
student answers inputs in that mode correctly exactly when its rule
is present in the prompt. Generated prompts specify subsets of these
rules. Rules do not interfere, and all $d$ rules fit in an admissible
prompt. The full-rule prompt therefore attains accuracy one, proving
$J_\Pi^*=1$ over $\Pi$. Teacher reference answers are valid and correct;
agreement and task correctness coincide under the model's evaluation
policies. A revision may add rules and remove existing ones. Its useful
and harmful effects enter the analysis separately.

\paragraph{Repair, damage, and the actual parent update.}
Write $\widetilde c_i(p)$ for the agreement indicator in Equation 2 on search
question $i$, and let
$e_t=1-\widehat J_s(p_t)$ be the
current parent's search disagreement rate. For a valid, fully
evaluated proposal $q_t$, define
\[
\widehat U_t=\frac{1}{n_s}\sum_{i=1}^{n_s}
(1-\widetilde c_i(p_{t-1}))\widetilde c_i(q_t),\qquad
\widehat L_t=\frac{1}{n_s}\sum_{i=1}^{n_s}
\widetilde c_i(p_{t-1})(1-\widetilde c_i(q_t)).
\]
These are the fractions of search questions changed from disagreement
to agreement and from agreement to disagreement, respectively. Their
difference is the proposal's change in search agreement. Equation 3
therefore gives the exact identity
\[
e_t=e_{t-1}-[\widehat U_t-\widehat L_t]_+.
\]
A revision can replace the parent while damaging some previously
successful cases, provided its total repair exceeds that damage.
Ties retain the parent. For a failed or skipped slot, set
$\widehat U_t=\widehat L_t=0$; this represents its actual zero change.

Let $\mathcal F^{\rm prop}_{t-1}$ contain the information used in construction
after feedback formation and before the round's proposal randomness,
including the reused search set and current parent, but excluding
reserved observations. Suppose,
at every such history,
\[
\begin{aligned}
\mathbb E[\widehat U_t\mid\mathcal F^{\rm prop}_{t-1}]
&\ge\alpha e_{t-1}-u,\\
\mathbb E[\widehat L_t\mid\mathcal F^{\rm prop}_{t-1}]
&\le\beta e_{t-1}+h,
\end{aligned}
\qquad 0\le\beta<\alpha\le1,\quad u,h\ge0.
\]
The correction condition describes how much of the current disagreement
the feedback and proposal process can repair on average; $u$ permits
persistent shortfalls, including incomplete exposure. The damage
condition permits both a component proportional to the remaining
disagreement and a residual component $h$. These expectations include
unsuccessful slots and allow dependence between correction, damage,
and previous revisions. Solution quality affects these quantities;
the conditions do not require every solution to yield a useful revision.

\begin{proposition}[Finite-rule specialization with repair and damage]
\label{prop:finite-repair}
Under this task model and the conditional moment bounds above, suppose
initial construction and final selection complete almost surely.
Search inputs are $n_s$ independent draws from the task population;
reserved inputs are $n_v$ independent draws, independent of construction.
Let $M$ be the initial bank size, $R$ the refinement-slot limit, and
$K_{\rm beh}=\min\{M+R,2^d\}$ bound the distinct agreement behaviors in the archive.
Set $\kappa=\alpha-\beta$, $\nu=u+h$, and $\bar e_0=\mathbb E[e_0]$.
With $(1-\kappa)^0=1$, define
\[
B_R=\min\!\left\{\bar e_0,
(1-\kappa)^R\bar e_0+
\frac{\nu}{\kappa}\bigl[1-(1-\kappa)^R\bigr]\right\}.
\]
The prompt returned by Equation 4 satisfies
\[
\begin{aligned}
\mathbb E[J_\Pi^*-J(p_S^*)]\le\min\Bigl\{1,\;&
\underbrace{B_R}_{\substack{\text{bound on remaining}\\\text{search disagreement}}}+
\underbrace{\frac12\sqrt{\frac{d-1}{n_s}}}_{\text{search--population discrepancy}}\\
&+\underbrace{\sqrt{\frac{\log K_{\rm beh}}{2n_v}}}_{\text{reserved selection}}\Bigr\}.
\end{aligned}
\]
\end{proposition}

The analysis follows the search-selected parent used by Equation 3.
All quantities $\widehat U_t,\widehat L_t,e_t$ refer to the same
cached search questions and scores. A fully evaluated proposal is
archived whether or not it replaces the parent. A failed, incomplete,
or skipped slot has zero correction and damage and changes neither
parent nor archive. These conventions include stopping in the process
being analyzed.

\begin{proof}[Proof of Proposition~\ref{prop:finite-repair}]
\textbf{Search progress.}
The change in empirical agreement under a proposal equals
$\widehat U_t-\widehat L_t$. The strict update in Equation 3 hence
gives
\[
e_t=e_{t-1}-[\widehat U_t-\widehat L_t]_+\le e_{t-1}.
\]
Because $[x]_+\ge x$, the conditional moment bounds imply
\[
\begin{aligned}
\mathbb E[e_t\mid\mathcal F^{\rm prop}_{t-1}]
&\le e_{t-1}-\mathbb E[\widehat U_t-\widehat L_t
\mid\mathcal F^{\rm prop}_{t-1}]\\
&\le(1-\kappa)e_{t-1}+\nu.
\end{aligned}
\]
Taking expectations and iterating bounds $\mathbb E[e_R]$ by
$(1-\kappa)^R\bar e_0+
\nu[1-(1-\kappa)^R]/\kappa$.
The pathwise inequality $e_R\le e_0$ supplies the other argument
of the minimum in $B_R$. No independence of revisions is used,
and a proposal can lose rules even when it is accepted.

\textbf{From search to population.}
Let $I(p)$ denote a prompt's subset of executable rules. For every
subset, simultaneously,
\[
\left|J(p)-\widehat J_s(p)\right|
=\left|\sum_{j\in I(p)}(\pi_j-\widehat\pi_j)\right|
\le D_s,\qquad
D_s=\frac12\sum_{j=1}^d|\widehat\pi_j-\pi_j|.
\]
Because the inequality holds for every rule subset simultaneously,
it also holds for the adaptively chosen parent.
For independent population draws,
\[
\begin{aligned}
\mathbb E[D_s]
&\le\frac{1}{2\sqrt{n_s}}
\sum_{j=1}^d\sqrt{\pi_j(1-\pi_j)}\\
&\le\frac12\sqrt{\frac{d-1}{n_s}}.
\end{aligned}
\]
The first step uses each empirical mode frequency's variance; the
second is Cauchy--Schwarz applied to $\sqrt{\pi_j}$ and
$\sqrt{1-\pi_j}$. Since $J_\Pi^*=1$ and $p_R\in\mathcal C$,
\[
\Delta_{\rm gen}
\le1-J(p_R)\le e_R+D_s.
\]
The full-rule prompt supplies the global benchmark in this step.

\textbf{Reserved selection.}
Condition on the history used to construct the archive, excluding all
reserved observations. This fixes $\mathcal C$.
Let $p^{\rm best}$ maximize $J$ over this fixed archive,
with a fixed tie rule used only to specify the comparison. Define
$Z_p=\widehat J_v(p)-J(p)$.
Since Equation 4 maximizes reserved agreement,
\[
\Delta_{\rm sel}
\le\max_{p\in\mathcal C}Z_p-Z_{p^{\rm best}}.
\]
Each $Z_p$ has conditional mean zero. Hoeffding's lemma
(\citealp{hoeffding1963}) gives
$\mathbb E[e^{\lambda Z_p}\mid\mathcal C]\le
e^{\lambda^2/(8n_v)}$. The same bound holds conditional on the
full construction history; its use does not require independent
scores across prompts. There are at most $K_{\rm beh}$ distinct score behaviors,
because the archive has at most $M+R$ prompts and the model has at
most $2^d$ rule subsets. For $\lambda>0$, the log-sum-exp bound yields
\[
\mathbb E[\max_{p\in\mathcal C}Z_p\mid\mathcal C]
\le\frac{\log K_{\rm beh}}{\lambda}+\frac{\lambda}{8n_v}.
\]
Optimizing over $\lambda$ gives
$\mathbb E[\Delta_{\rm sel}\mid\mathcal C]
\le\sqrt{\log K_{\rm beh}/(2n_v)}$. For $K_{\rm beh}=1$, selection loss is zero.
The centered error of the fixed best candidate has expectation zero;
this is why this expectation bound needs no factor of two.
Equation~\eqref{eq:selection-comparison} instead bounds a realized selection error by a uniform
absolute deviation and retains that factor.

Take expectations in Equation 5, combine the three bounds, and
use $0\le J_\Pi^*-J(p_S^*)\le1$ to obtain the finite-rule bound above.
Here $\widetilde J=J$, so the agreement analysis of Theorem~\ref{thm:robust-global}
translates to accuracy with $\gamma=1,\omega=0$.
\end{proof}

\paragraph{Independent groups with dependent questions.}
The sampling argument also permits independent groups, each containing
dependent questions, with fixed group weights summing to one. Suppose
the weighted group means match the target population, and reserved
groups remain independent of construction. Let $v_s$ and $v_v$ be the
sums of squared search and reserved group weights. The finite-rule bound above then
holds with $1/n_s$ replaced by $v_s$ and $1/n_v$ by $v_v$.
For search, write the group's random mode-frequency vector as
$Y_g$ with mean $\mu_g$, and put
$q_j=\sum_g w_g^2\mu_{gj}/v_s$. Independence and
$Y_{gj}\in[0,1]$ give
\[
\operatorname{Var}(\widehat\pi_j)
\le\sum_gw_g^2\mu_{gj}(1-\mu_{gj})
\le v_s q_j(1-q_j).
\]
Because $\sum_jq_j=1$, the same Cauchy--Schwarz step gives
$\mathbb E[D_s]\le\frac12\sqrt{(d-1)v_s}$.
Weighted Hoeffding gives the reserved term
$\sqrt{v_v\log K_{\rm beh}/2}$. The empirical correction and damage fractions
use these same weights. Question-weighted scoring corresponds to
group weights proportional to group sizes. Random sizes require
conditioning that preserves independence and population matching,
as discussed in Appendix~\ref{app:statistical-units}. The squared
weights express the concentration supported by independent groups.

\paragraph{From feedback exposure to the conditional moments.}
Fix a pre-proposal history and let $\widehat\pi_j$ denote the observed
search mass of a mode missing from the parent. Let $W_{tj}$ mean that
the actual feedback input contains correct repair evidence for this
mode, defined from its content before generating the proposal. Set
$f_{tj}=\Pr(W_{tj}\mid\mathcal F^{\rm prop}_{t-1})$ and suppose the conditional
probability that a valid, fully evaluated proposal adds rule $j$,
given this evidence, is at least $s_{tj}$. Failed and incomplete
proposals contribute no added rules. Then
\[
\mathbb E[\widehat U_t\mid\mathcal F^{\rm prop}_{t-1}]
\ge\sum_{j\notin I(p_{t-1})}\widehat\pi_j f_{tj}s_{tj}.
\]
Thus a weighted exposure-and-transfer bound of
$\alpha e_{t-1}-u$ implies the correction condition.
Similarly, if the conditional probability of losing an existing
rule $j$ is at most $\ell_{tj}$, then
\[
\mathbb E[\widehat L_t\mid\mathcal F^{\rm prop}_{t-1}]
\le\sum_{j\in I(p_{t-1})}\widehat\pi_j\ell_{tj}.
\]
Bounding this weighted sum by $\beta e_{t-1}+h$ supplies the damage
condition. Linearity permits coupled additions and removals; their
independence is unnecessary. The feedback input is fixed by this pre-proposal history, so
$f_{tj}$ is zero or one. The transfer probabilities still
refer to the actual stochastic proposal and admission law.

These quantities can be measured through disagreements repaired and
agreements lost in repeated proposals at a fixed feedback state.
Search transitions are label-free observations. Interpreting them
as true correction and damage uses the correct-reference and scoring
assumptions; an independent labeled probe tests that interpretation.
Repeated observations test the proposed rate bounds at each examined
state; state sampling determines the construction histories they cover.

\paragraph{Stopping and varying rates.}
For deterministic round-specific parameters
$\kappa_t=\alpha_t-\beta_t\in[0,1]$ and $\nu_t=u_t+h_t$,
the conditional recurrence gives
\[
\mathbb E[e_R]\le
\bar e_0\prod_{t=1}^R(1-\kappa_t)
+\sum_{t=1}^R\nu_t\prod_{i=t+1}^R(1-\kappa_i).
\]
It may again be capped by $\bar e_0$. A known skipped slot has
$\kappa_t=\nu_t=0$. For a history-dependent stop, retain the
conditional recurrence or choose deterministic bounds valid over
all such histories; realized average rates cannot be substituted
into the product. Under constant positive $\alpha$, a stopped
history with $e>0$ satisfies the correction condition only if
$u\ge\alpha e$. Early stopping at $e=0$ satisfies both conditions
without this residual.

\paragraph{Why search control is necessary.}
Consider two modes with population masses $(0.9,0.1)$ but search
frequencies $(0.4,0.6)$. A parent containing only the first rule
has population accuracy $0.9$. A revision containing only the second
repairs $0.6$ of search questions and damages $0.4$, so Equation 3
accepts it although population accuracy falls to $0.1$.
The archive retains the parent. This calculation shows why strict
search improvement requires an explicit search--population term
and a final-selection argument.

\subsection{Exact Attainment under Rule-Preserving Proposals}
\label{app:finite-rule}

In the finite-rule model of Appendix~\ref{app:finite-repair}, inputs fall into
$d$ disjoint modes with probabilities $\pi_j>0$. Generated prompts
specify subsets of executable rules. The student is correct on a mode
exactly when its rule is present; rules do not interfere, and the
full rule set fits the admissible prompt constraints. It attains
$J_\Pi^*=1$. Initial synthesis and search evaluation complete almost
surely. Teacher solutions and references are valid and correct,
and evaluations complete with consistent scoring.

At every history with a rule missing from the parent but represented
in search, feedback exposes an admissible disagreement. Valid proposals
preserve existing parent rules, and an admitted, fully evaluated
proposal adds at least one exposed missing rule with conditional
probability at least $r>0$. Each such addition strictly increases
cached search agreement and is accepted by Equation 3. Construction
has $R$ slots available and may stop early when no search disagreements
remain. An additional resource stop before those slots or before final
selection would require a corresponding change to the analyzed
execution law. These are conditions on K2P's existing operations.

\begin{theorem}[Global attainment in a finite-rule regime]
\label{thm:finite-rule}
Under the finite-rule conditions above, draw $n_s$ search and $n_v$
reserved inputs independently from the task population. Let $k_0$
be the number of rules missing from the initial search-selected
parent $p_0$, and define
$\Phi_{R,r}(k)=\Pr\{\operatorname{Bin}(R,r)<k\}$, with
$\Phi_{R,r}(0)=0$. After $R$ refinement slots, the prompt returned
by Equation 4 satisfies
\[
\begin{aligned}
\Pr\{J(p_S^*)<J_\Pi^*\}\le
\min\Bigl\{1,\;&
\underbrace{\sum_{j=1}^d(1-\pi_j)^{n_s}}_{\text{search coverage}}+
\underbrace{\mathbb E[\Phi_{R,r}(k_0)]}_{\text{repair opportunities}}\\
&+\underbrace{\sum_{j=1}^d(1-\pi_j)^{n_v}}_{\text{reserved coverage}}
\Bigr\}.
\end{aligned}
\tag{C1}\label{eq:finite-rule}
\]
The expectation is over source and search initialization. If $U$
denotes the right-hand side, then also
$\mathbb E[J_\Pi^*-J(p_S^*)]\le U$.
\end{theorem}

\begin{proof}[Proof of Theorem~\ref{thm:finite-rule}]
Let $\mathcal E_s^{\rm cov}$ and $\mathcal E_v^{\rm cov}$ denote the events that the search
and reserved sets, respectively, contain every mode. A union bound
gives
\[
\Pr((\mathcal E_s^{\rm cov})^c)\le\sum_{j=1}^d(1-\pi_j)^{n_s},\qquad
\Pr((\mathcal E_v^{\rm cov})^c)\le\sum_{j=1}^d(1-\pi_j)^{n_v}.
\]
Fix a source/search initialization history $H_{\rm init}$ on $\mathcal E_s^{\rm cov}$.
Every rule missing from the initial parent $p_0$ has an erroneous
search example. Correct teacher references make those examples
disagreements, and feedback exposes a missing rule. Every successful
repair adds at least one such rule without removing existing ones,
so Equation 3 accepts it. Consequently, $k_0$ successful repairs
suffice to reach a full-rule prompt, which remains in the archive.

Conditional on every unfinished history, success has probability at
least $r$. Realize each success using a fresh uniform variable and
its history-dependent success probability. A reference success occurs
when the same uniform variable is at most $r$. Until completion,
each reference success is an actual success. After completion extend
the reference trials to $R$ slots. The probability of failing to
complete is therefore at most
$\Phi_{R,r}(k_0)=\Pr\{\operatorname{Bin}(R,r)<k_0\}$, conditional
on $H_{\rm init}$. This coupling permits dependent, adaptive revisions.
Early stopping on zero search disagreements on $\mathcal E_s^{\rm cov}$
means that all rules are already present.

To separate the two losses in Equation~\ref{eq:two-gaps}, define
\[
F_{\rm gen}=\{\Delta_{\rm gen}>0\},\qquad
F_{\rm sel}=\{\Delta_{\rm gen}=0,\ \Delta_{\rm sel}>0\}.
\]
These events are disjoint and their union is
$\{J(p_S^*)<J_\Pi^*\}$. The preceding argument gives
\[
\Pr(F_{\rm gen})\le
\Pr((\mathcal E_s^{\rm cov})^c)+
\mathbb E[\mathbf1_{\mathcal E_s^{\rm cov}}\Phi_{R,r}(k_0)]
\le\sum_j(1-\pi_j)^{n_s}+\mathbb E[\Phi_{R,r}(k_0)].
\]
On $\mathcal E_v^{\rm cov}$, every incomplete prompt makes at least one
reserved error, while a full-rule prompt has agreement one. Whenever
the archive contains a full-rule prompt, Equation 4 therefore selects
one. Hence $F_{\rm sel}\subseteq(\mathcal E_v^{\rm cov})^c$ and
$\Pr(F_{\rm sel})\le\sum_j(1-\pi_j)^{n_v}$.
Adding the disjoint-event probabilities and taking the minimum with
one proves Equation~\ref{eq:finite-rule}. Independence between these
failure events is unnecessary. Finally,
\[
0\le J_\Pi^*-J(p_S^*)
\le\mathbf1\{J(p_S^*)<J_\Pi^*\},
\]
so taking expectations proves the stated expected-gap consequence.
\end{proof}

For fixed $d$ and $\pi_{\min}>0$, the coverage terms are at most
$d e^{-n_s\pi_{\min}}$ and $d e^{-n_v\pi_{\min}}$.
Since $k_0\le d$, the repair tail is at most
$\Pr\{\operatorname{Bin}(R,r)<d\}$, which tends to zero for a
uniform $r>0$ as $R\to\infty$. This proves consistency as the three
resources grow under the stated task and execution conditions.

\subsection{Finite-Resource Accuracy Guarantees}
\label{app:finite-mass}

Theorem~\ref{thm:finite-rule} controls exact global attainment.
We use the same rule-preserving conditions to measure partial progress
without requiring every mode to appear in search. A mode of small
probability contributes only that probability to accuracy loss when
left unresolved. Let $\mathcal C_0=\mathcal P$ and
$X_0=J_\Pi^*-\max_{p\in\mathcal C_0}J(p)$.

Condition on the source and search history $H_{\rm init}$ after initialization.
Let $\mathcal M_0$ be the rules missing from $p_0$ that occur in
search, with $k_{\rm obs}=|\mathcal M_0|$. Let $u_0$ be the total
population mass of rules missing from $p_0$ and absent from search.
Sort the probabilities of modes in $\mathcal M_0$ as
$w_1\ge\cdots\ge w_{k_{\rm obs}}$. For a nonnegative integer $z$, set
\[
\Psi_z(\mathcal M_0)=
\sum_{i=1}^{(k_{\rm obs}-z)_+}w_i,
\]
where an empty sum is zero. After $z$ successful repairs, at most
$(k_{\rm obs}-z)_+$ exposed rules remain missing. Their mass is largest
when they are the most frequent modes, which explains this ordering.

\begin{theorem}[Remaining error mass after finite refinement]
\label{thm:finite-mass}
Under the finite-rule and execution conditions above, let
$Z\sim\operatorname{Bin}(R,r)$ be an independent reference variable.
Conditional on $H_{\rm init}$,
\[
\mathbb E[\Delta_{\rm gen}\mid H_{\rm init}]
\le\mathbb E_Z\!\left[
\min\{X_0,u_0+\Psi_Z(\mathcal M_0)\}\right].
\tag{C2}\label{eq:finite-mass-bound}
\]
If reserved selection satisfies
$\mathbb E[\Delta_{\rm sel}]\le\zeta_{\rm sel}$, then
\[
\mathbb E[1-J(p_S^*)]\le
\min\!\left\{1,\mathbb E_{H_{\rm init},Z}
\!\left[\min\{X_0,u_0+\Psi_Z(\mathcal M_0)\}\right]
+\zeta_{\rm sel}\right\}.
\tag{C3}\label{eq:finite-mass-global}
\]
\end{theorem}

\begin{proof}[Proof of Theorem~\ref{thm:finite-mass}]
Condition on the source and search history $H_{\rm init}$ after initialization,
fixing the initial bank, parent, and observed modes. Its missing
rules divide into $\mathcal M_0$, the search-visible set, and an
unobserved set of mass $u_0$. Each successful repair adds at least
one distinct member of $\mathcal M_0$ and preserves all previously
acquired parent rules. After $s$ successes, at most
$(k_{\rm obs}-s)_+$ members of $\mathcal M_0$ remain missing.
Their total mass is at most $\Psi_s(\mathcal M_0)$, because the
largest such sum uses the largest mode probabilities.
The parent's remaining error is consequently at most
$u_0+\Psi_s(\mathcal M_0)$. Repairs of unobserved rules only reduce it.

Until all search-visible missing rules are repaired, success
indicators have conditional probabilities at least $r$. Realize
each success with a fresh uniform variable and its history-dependent
conditional success probability. The event that this variable is at
most $r$ defines an independent Bernoulli($r$) reference success.
After completion extend the references to $R$ slots. The actual
number of distinct repaired search-visible rules is at least
$\min\{Z,k_{\rm obs}\}$ under this coupling.
Thus the final parent has error at most
$u_0+\Psi_Z(\mathcal M_0)$. This coupling is a proof device and
does not impose independent model proposals.

The archive contains the final parent and the entire initial bank.
Its best error is therefore bounded by
$\min\{X_0,u_0+\Psi_Z(\mathcal M_0)\}$. Taking conditional
expectations proves Equation~\ref{eq:finite-mass-bound}.
Take expectation over $H_{\rm init}$, add the selection loss from
Equation~\ref{eq:two-gaps}, and bound accuracy loss by one to prove
Equation~\ref{eq:finite-mass-global}.
\end{proof}

If search comprises $n_s$ independent population draws, then
\[
\begin{aligned}
\mathbb E[u_0]
&=\sum_j\pi_j\Pr\{j\text{ absent from search and from }p_0\}\\
&\le\sum_j\pi_j(1-\pi_j)^{n_s}.
\end{aligned}
\]
This step does not require the search-selected parent to be
independent of the search data. A mode being absent from both the search set and the initial parent
implies that it is absent from the search set. The contribution of each unseen mode is weighted by
its frequency, avoiding an all-modes coverage requirement.

For fixed $H_{\rm init}$ and $r>0$, letting $R$ grow makes
$\Psi_Z(\mathcal M_0)$ vanish in expectation; the bound tends to
$\min\{X_0,u_0\}$. More refinement alone cannot guarantee discovery
of modes absent from the fixed feedback population. Under a
finite mode set, increasing independent search information also
makes the expected unseen mass vanish. Accurate final selection
remains a separate requirement.

If $k_0>R$, the repair-count bound for exact attainment is one;
it does not imply impossibility because a single proposal may add
several rules. The mass bound records partial progress even when
$R<k_{\rm obs}$. It remains conservative because it counts only one
rule per successful slot and assigns the largest remaining mode
probabilities to unresolved rules.

\subsection{Statistical Units and Population Matching}
\label{app:statistical-units}

Fix the teacher, target student, prompt interfaces, and decoding and
scoring policies. For each prompt $p$, let $c_p$ be the correctness
indicator defined in Equation 1, and let $\widetilde c_p\in\{0,1\}$ be the
validity-aware agreement indicator in
Equation 2. Expectations include the input and model-decoding
randomness. Thus $J(p)=\mathbb E[c_p]$, $\widetilde J(p)=\mathbb E[\widetilde c_p]$, and
$b(p)=\mathbb E[c_p-\widetilde c_p]$. The population for $\widetilde J$ matches that for $J$.
Candidate responses on the same question may be paired, and candidates
may share a teacher reference.

Theorem~\ref{thm:archive-selection} conditions on the source and search
history $H_{\rm con}$ that produces $\mathcal C$, excluding reserved references
and scores. Conditional
on $H_{\rm con}$, each group's full vector of questions, teacher references, and
student outputs for all candidates is independent of the other groups.
The reserved evaluation law is unaffected by construction. Groups need
not be identically distributed, but their weighted expectation must
equal $\widetilde J(p)$. For question-weighted scoring,
$Z_g(p)=n_g^{-1}\sum_{i\in g}\widetilde c_i(p)$ and $w_g=n_g/n_v$, so the weighted
mean is exactly Equation 2. The theorem assumes fixed weights. If group
sizes are random, conditioning on them is legitimate only if it
preserves group independence and population matching.

Disjoint question identities alone do not establish these statistical
assumptions. The selection bound uses the actual number of candidates
$K_{\mathcal C}$ and the squared-weight sum $v_v$ for the reserved
population. An upper limit of twelve candidates can be used
conservatively. A deterministic benchmark split does not itself
establish independence or population matching. At eighty questions,
the distribution-free bound can be too loose to certify accuracy
losses of a few percentage points, even when empirical selection
loss is small.

\subsection{Relative Proxy Bias and Actual Scoring}
\label{app:proxy-bias}

For $b(p)=J(p)-\widetilde J(p)$, any two prompts satisfy
\[
J(p)-J(q)=\widetilde J(p)-\widetilde J(q)+b(p)-b(q).
\tag{C4}\label{eq:relative-bias}
\]
This identity expresses how candidate-specific score offsets enter
an accuracy comparison. The following bounds relate the difference
in offsets to teacher errors and actual scoring policies.

\begin{proposition}[Where teacher errors affect ranking]
\label{prop:local-bias}
Suppose responses are valid and task correctness is equality with the
normalized gold answer. Let $O$ be the event that the teacher is wrong
and $D_{pq}$ the event that candidates $p$ and $q$ produce different
normalized student answers. Then
\[
|b(p)-b(q)|\le 2\Pr(O\cap D_{pq}).
\]
\end{proposition}
When the teacher is correct, agreement and correctness assign the same
score under these assumptions. When two candidates produce the same
student answer, any teacher error affects them equally. Relative bias
can therefore arise only where teacher errors overlap with candidate
disagreements. The factor of two allows a wrong teacher answer to reward one
candidate's matching error while denying agreement credit to the
other candidate's correct answer. Consequently, teacher accuracy alone does not specify the
distortion of a candidate comparison; the location of its errors also
matters. The extension below includes invalid or truncated responses and the
actual task scorers.

The bias can be written as
\[
b(p)=\Pr(c_p=1,\widetilde c_p=0)-\Pr(c_p=0,\widetilde c_p=1).
\tag{C5}\label{eq:bias-events}
\]
The first term counts correct student answers unrewarded by agreement;
the second counts agreement on incorrect answers. Subtracting
$J(q)=\widetilde J(q)+b(q)$ from $J(p)=\widetilde J(p)+b(p)$ proves
Equation~\ref{eq:relative-bias}.

To include invalid or truncated responses, map them to a symbol
$\bot$, writing $\bar S_p$ and $\bar T$ for the resulting normalized
answers. Define
$c_p^\dagger=\mathbf1\{\bar S_p\ne\bot,\bar S_p=y\}$,
$J^\dagger(p)=\mathbb E[c_p^\dagger]$, and
$\chi(p)=J(p)-J^\dagger(p)$. Here $y$ denotes the normalized gold
answer. The correction $\chi$ accounts for differences between the
actual task scorer and this aligned scoring rule, including whether an
otherwise correct truncated response is credited.

Let $O=\{\bar T\ne\bot,\bar T\ne y\}$,
$I=\{\bar T=\bot\}$, and $D_{pq}=\{\bar S_p\ne\bar S_q\}$. Then
\[
|b(p)-b(q)|\le
2\Pr(O\cap D_{pq})+\Pr(I\cap D_{pq})
+|\chi(p)-\chi(q)|=: \tau_{pq}.
\tag{C6}\label{eq:full-local-bias}
\]
\begin{proof}
Put $d_p=c_p^\dagger-\widetilde c_p$. When the teacher is valid and correct,
$d_p=0$. On $O$, $d_p\in\{-1,0,1\}$, so
$|d_p-d_q|\le2$. On $I$, agreement is zero and
$d_p\in\{0,1\}$, giving $|d_p-d_q|\le1$. If the normalized student
answers coincide, $d_p=d_q$. Hence, pointwise,
\[
|d_p-d_q|\le2\mathbf1_{O\cap D_{pq}}+\mathbf1_{I\cap D_{pq}}.
\]
Take expectations and use $b(p)=\mathbb E[d_p]+\chi(p)$.
Proposition~\ref{prop:local-bias} follows when $I$ is empty and
$\chi=0$.
\end{proof}

The task scorers for Tracking, Colored Objects, and QuaRTz compare
extracted answers without a separate truncation exclusion, whereas agreement
requires untruncated responses. TabMWP's explicit-final scorer and
agreement both exclude truncated responses. Equation~\ref{eq:full-local-bias}
covers these policies without changing them. Its probabilities refer to
the actual joint decoding law; candidate disagreement can therefore
also reflect decoding randomness. The archive selection theorem uses
$b=J-\widetilde J$ directly and does not require the simplified scoring assumptions
of Proposition~\ref{prop:local-bias}.

\subsection{Selection from a Frozen Archive}
\label{app:selection-proof}

\paragraph{Realized proxy and estimation errors.}
For a frozen archive, let $\widehat J_v$ denote reserved agreement and set
$\rho_{\mathcal C}=\max_{p\in\mathcal C}b(p)-\min_{p\in\mathcal C}b(p)$
and $\epsilon_{\mathcal C}=\max_{p\in\mathcal C}|\widehat J_v(p)-\widetilde J(p)|$.
Equation~\ref{eq:relative-bias} and empirical maximization in Equation 4 give
\[
\Delta_{\rm sel}\le
\underbrace{\rho_{\mathcal C}}_{\text{relative proxy bias}}+
\underbrace{2\epsilon_{\mathcal C}}_{\text{agreement estimation}}.
\tag{C7}\label{eq:selection-comparison}
\]

One candidate can be underestimated while the selected candidate is
overestimated, giving the factor of two. Conditioning on the source and
search history $H_{\rm con}$ fixes the adaptively generated archive;
reserved references and scores remain outside this history. The following
theorem bounds estimation error using independent groups, allowing
dependence within a table or background.

\Needspace{11\baselineskip}
\begin{theorem}[Selection from a frozen archive]
\label{thm:archive-selection}
Condition on the construction history $H_{\rm con}$ and its archive $\mathcal C$,
with $K_{\mathcal C}=|\mathcal C|\ge1$. Suppose reserved evaluation consists of $n_{\rm grp}$
independent groups, allowing arbitrary dependence within each group.
Let $Z_g(p)\in[0,1]$ be group $g$'s mean agreement and
$\widehat J_v(p)=\sum_gw_g Z_g(p)$, with fixed weights $w_g\ge0$
summing to one. Assume $\mathbb E[\widehat J_v(p)\mid H_{\rm con}]=\widetilde J(p)$ for
every candidate.
Write $v_v=\sum_g w_g^2$ and
$\rho_{\mathcal C}=\max_{p\in\mathcal C}b(p)-\min_{p\in\mathcal C}b(p)$.
For any $\delta\in(0,1)$, Equation 4 satisfies, with conditional
probability at least $1-\delta$,
\[
\max_{p\in\mathcal C}J(p)-J(p_S^*)
\le \rho_{\mathcal C}+\sqrt{2v_v\log(2K_{\mathcal C}/\delta)}.
\tag{C8}\label{eq:archive-bound}
\]
\end{theorem}

The left-hand side measures loss in the true accuracy objective of
Equation 1, relative to the best archived prompt. The first term,
$\rho_{\mathcal C}$, measures how differently the agreement proxy
biases the candidates. The second term accounts for selecting from
finite reserved evaluations. A more informative reserved set can reduce
the second term; the first depends on the teacher's errors and the
candidate responses. These two effects can therefore vary separately.

Hoeffding's inequality (\citealp{hoeffding1963}) bounds
agreement-estimation errors simultaneously over the frozen archive.
Combining this bound with Equation~\ref{eq:selection-comparison} gives
Equation~\ref{eq:archive-bound}. This is why
the selection bound depends on the frozen archive size $K_{\mathcal C}\le M+R=12$ in our configuration;
candidate generation contributes through the other term in
Equation~\ref{eq:two-gaps}.

For question-weighted evaluation, a group of $n_g$ questions receives
weight $w_g=n_g/n_v$, so the grouped average is exactly Equation 2.
Independent questions are singleton groups and give $v_v=1/n_v$.
Shared QuaRTz backgrounds instead define groups, whose weights reflect
their numbers of questions. The quantity $1/v_v$ expresses the effective
sample size in this bound: unequal weights and possible dependence
within groups reduce the information certified by the bound.
Independence between groups and the expectation-matching condition
remain statistical assumptions beyond disjoint question identities.
The probability statement concerns repeated reserved evaluation under
these assumptions, conditional on the constructed archive.

\begin{proof}[Proof of Theorem~\ref{thm:archive-selection}]
Conditional on $H_{\rm con}$, the terms $w_gZ_g(p)$ are independent and have
range lengths $w_g$. Hoeffding's inequality (\citealp{hoeffding1963})
gives, for any fixed candidate and $t>0$,
\[
\Pr\bigl(|\widehat J_v(p)-\widetilde J(p)|>t\mid H_{\rm con}\bigr)
\le2\exp(-2t^2/v_v).
\tag{C9}\label{eq:weighted-hoeffding}
\]
A union bound over $K_{\mathcal C}$ candidates yields, with probability at least
$1-\delta$, simultaneous errors at most
$\epsilon_v=\sqrt{v_v\log(2K_{\mathcal C}/\delta)/2}$.
Let $p_J\in\arg\max_{p\in\mathcal C}J(p)$. On this event,
\begin{align*}
J(p_J)-J(p_S^*)
={}&[\widetilde J(p_J)-\widehat J_v(p_J)]
+[\widehat J_v(p_J)-\widehat J_v(p_S^*)]\\
&+[\widehat J_v(p_S^*)-\widetilde J(p_S^*)]+b(p_J)-b(p_S^*)\\
\le{}&2\epsilon_v+\rho_{\mathcal C}.
\end{align*}
The middle empirical difference is nonpositive by Equation 4.
Substituting $\epsilon_v$ proves Equation~\ref{eq:archive-bound}.
If the assumptions hold for each construction history, integrating over
$H_{\rm con}$ gives the same unconditional probability guarantee, with each
history's archive and bias range. Singleton groups with $w_g=1/n_v$
recover the independent-question special case.
\end{proof}

\paragraph{Controlling reserved selection.}
Under Theorem~\ref{thm:archive-selection}, with fixed reserved weight
sum $v_v$, fixed initial bank size $M$, and at most $M+R$ archived
prompts, for any $\delta_v\in(0,1)$ a valid bound is
\[
\zeta_{\rm sel}=\min\left\{1,\mathbb E[\rho_{\mathcal C}]
+\sqrt{2v_v\log(2(M+R)/\delta_v)}+\delta_v\right\}.
\]
This follows by integrating the conditional guarantee and bounding the
gap by one on its failure event. If the weights are history-dependent,
take the corresponding expectation under legitimate conditional
evaluation assumptions. None of these bounds require the search and
reserved success events to be independent.

\subsection{Best-Candidate Identification and Population Shift}
\label{app:identification-proof}

Equation~\ref{eq:archive-bound} controls accuracy loss. Identifying the
best candidate itself additionally requires that it be distinguishable
by population agreement, as the following corollary makes explicit.

\begin{corollary}[Identifying the best archived prompt]
\label{cor:identification}
Under Theorem~\ref{thm:archive-selection}, suppose $K_{\mathcal C}\ge2$ and the unique
accuracy maximizer $p^\dagger$ also uniquely maximizes $\widetilde J$, with margin
$\Delta_{\rm ag}=\min_{q\ne p^\dagger}[\widetilde J(p^\dagger)-\widetilde J(q)]>0$. Then
\[
\Pr(p_S^*\ne p^\dagger\mid H_{\rm con})
\le\min\{1,(K_{\mathcal C}-1)e^{-\Delta_{\rm ag}^2/(2v_v)}\}.
\tag{C10}\label{eq:identification}
\]
\end{corollary}
\begin{proof}[Proof of Corollary~\ref{cor:identification}]
For any $q\ne p^\dagger$, the independent group differences
$w_g[Z_g(q)-Z_g(p^\dagger)]$ have range lengths $2w_g$ and total
expectation at most $-\Delta_{\rm ag}$. One-sided Hoeffding therefore gives
\[
\Pr\bigl(\widehat J_v(q)\ge\widehat J_v(p^\dagger)\mid H_{\rm con}\bigr)
\le e^{-\Delta_{\rm ag}^2/(2v_v)}.
\]
An incorrect selection requires at least one of the $K_{\mathcal C}-1$ competitors
to tie or exceed $p^\dagger$. A union bound and the trivial probability
bound of one prove Equation~\ref{eq:identification}, regardless of the
fixed tie rule. For $K_{\mathcal C}=1$, selection error is zero.
\end{proof}

The margin $\Delta_{\rm ag}$ is the smallest population-agreement lead of the
candidate with the highest population accuracy over a competitor. For a fixed archive and
matched population, increasing independent evaluation information so
that $v_v\to0$ makes the probability of selecting another candidate
vanish. The condition $\max_gw_g\to0$ is sufficient because
$v_v\le\max_gw_g$. Agreement and accuracy share a unique best candidate
whenever $J(p^\dagger)-J(q)>\tau_{p^\dagger q}$ for every competitor:
Equation~\ref{eq:relative-bias} then implies
$\widetilde J(p^\dagger)-\widetilde J(q)>0$. This condition allows nonzero proxy bias;
the population bias bounds and margins are not directly observed
during label-free construction. If the proxy instead changes the
identity of the best candidate, more evaluation data make its ranking
more precise without correcting the underlying bias.

If reserved and deployment populations differ, define
$b_{\rm sel}(p)=J_{\rm sel}(p)-\widetilde J_{\rm sel}(p)$ and
$d(p)=J_{\rm dep}(p)-J_{\rm sel}(p)$. Under the assumptions of
Theorem~\ref{thm:archive-selection} applied to the selection population,
the same argument gives the following bound with conditional probability
at least $1-\delta$, given $H_{\rm con}$:
\[
\max_{p\in\mathcal C}J_{\rm dep}(p)-J_{\rm dep}(p_S^*)
\le \operatorname{osc}_{\mathcal C}(b_{\rm sel}+d)
+\sqrt{2v_v\log(2K_{\mathcal C}/\delta)},
\tag{C11}\label{eq:population-shift}
\]
where $\operatorname{osc}_{\mathcal C}(f)=\max_{\mathcal C}f-
\min_{\mathcal C}f$. Population shift thus contributes through its
relative effect on the candidates.

Equation~\ref{eq:two-gaps} separates the global gap into generation
and selection. Appendices~\ref{app:robust-global-proof}, \ref{app:finite-rule}, \ref{app:finite-mass}, and~\ref{app:archive-proof}
analyze generation and its connection to this selection bound.

\subsection{A Distillation Risk Bound}
\label{app:distillation-risk}

The selection analysis also yields a bound relative to the teacher.
Let $c_T\in\{0,1\}$ denote teacher-reference correctness under the same
task scorer as $c_p$, and define
$\mathcal R_T=\mathbb E[1-c_T]$ and $\mathcal R_S(p)=1-J(p)$.
Assume scoring is compatible with agreement: for every candidate,
$\widetilde c_p=1$ implies $c_p=c_T$. This holds when agreeing valid normalized
answers receive the same task correctness score. Invalid or truncated
responses retain $\widetilde c_p=0$ under Equation 2 and remain included in
$1-\widetilde J(p)$.

\begin{corollary}[Distillation risk]
\label{cor:distillation-risk}
Under Theorem~\ref{thm:archive-selection} and the scoring compatibility
above, for any $\delta\in(0,1)$, with conditional probability at least
$1-\delta$,
\[
\mathcal R_S(p_S^*)
\le \mathcal R_T+\min_{p\in\mathcal C}\bigl[1-\widetilde J(p)\bigr]
+\sqrt{2v_v\log(2K_{\mathcal C}/\delta)}.
\tag{C12}\label{eq:distillation-risk}
\]
\end{corollary}

\begin{proof}
Scoring compatibility gives the pointwise inequality
$1-c_p\le(1-c_T)+(1-\widetilde c_p)$: when the teacher is correct and agreement
holds, the student is also correct. Taking expectations yields
$\mathcal R_S(p)\le\mathcal R_T+1-\widetilde J(p)$ for every candidate.
On the uniform concentration event in Appendix~\ref{app:selection-proof},
empirical maximization in Equation 4 gives
$\widetilde J(p_S^*)\ge\max_{p\in\mathcal C}\widetilde J(p)-2\epsilon_v$, where
$\epsilon_v=\sqrt{v_v\log(2K_{\mathcal C}/\delta)/2}$.
Combining these inequalities proves Equation~\ref{eq:distillation-risk}.
\end{proof}

The three terms separate teacher task error, the smallest agreement loss
attainable within the archive, and reserved-selection error. Here
$1-\widetilde J(p)$ serves as an answer-level distillation loss; when all responses
are valid and untruncated, it is the probability of different normalized
teacher and student answers. Synthesis and refinement determine which
agreement losses are attainable in the archive, while reserved
selection controls the final estimation term.

\subsection{General Archive Progress and Conditional Contraction}
\label{app:archive-proof}

This analysis permits imperfect solutions and reference answers. Its
population-level progress conditions complement the empirical
agreement analysis of Theorem~\ref{thm:robust-global}. They also
describe occasional archive improvements when proposals do not improve
their parents on average.

\paragraph{Source synthesis sets the starting quality.}
Let $\mathcal C_t$ be the archive after slot $t$, with
$\mathcal C_0=\mathcal P$. Write
\[
X_t=J_\Pi^*-\max_{p\in\mathcal C_t}J(p),\qquad
\mathbb E[X_0]\le d_{\rm src}.
\]
Here $X_R=\Delta_{\rm gen}$.
Thus $d_{\rm src}$ bounds the initial bank's expected distance from
the global optimum. It concerns the best initial candidate in true
accuracy; the search-agreement winner $p_0$ is the parent used for
refinement and can be a different candidate.

Source solutions can help at two distinct steps: they expose reusable
task knowledge, and synthesis expresses that knowledge as instructions
the student can execute on new inputs. Appendix~\ref{app:source-witness}
formalizes these steps through an evidence-exposure probability and a
conditional transfer rate, deriving a sufficient bound for
$d_{\rm src}$. Exposure concerns the examples supplied to an admitted
candidate, including the effects of retries and uniqueness checks.
Transfer is a property of the teacher--prompt--student interaction;
the presence of a worked solution alone does not establish it.

\paragraph{A revision's gain and its contribution to the archive.}
For a parent $p$ and proposal $q$, let $c_p,c_q$ be correctness
indicators on a common population input and decoding probability space
with the required marginal laws. Equal prompts use the same indicator.
Then
\[
J(q)-J(p)=
\Pr(c_p=0,c_q=1)-\Pr(c_p=1,c_q=0).
\tag{C13}\label{eq:repair-regression}
\]
The two terms are corrected errors and newly introduced errors.
Their difference describes the proposal's change relative to its
parent. To improve the archive, that change must also overcome any
shortfall of the parent relative to its best existing candidate.
Define
\[
\Delta J_t=J(q_t)-J(p_{t-1}),\qquad
\delta_{{\rm par},t-1}
=\max_{p\in\mathcal C_{t-1}}J(p)-J(p_{t-1}).
\]
A failed, skipped, or incomplete slot is represented analytically by
$q_t=p_{t-1}$. This convention adds no model call or candidate.
With $[z]_+=\max\{z,0\}$, retaining every fully evaluated valid
proposal gives the following exact relation.

\begin{proposition}[Progress of the retained archive]
\label{prop:archive-progress}
For every construction trajectory,
\[
X_t=X_{t-1}
-\underbrace{[\Delta J_t-\delta_{{\rm par},t-1}]_+}_{\text{new archive gain}}.
\tag{C14}\label{eq:archive-progress}
\]
In particular, $0\le X_t\le X_{t-1}$.
\end{proposition}

The identity follows because the best accuracy after a slot is the
maximum of the previous best accuracy and the proposal's accuracy.
It separates improving the parent from improving the archive. A
proposal can repair its parent yet remain below another saved
candidate, giving no archive gain. A poor proposal cannot erase a good
candidate that is already available for selection. Consequently,
occasional strong proposals can improve the archive even when the
mean accuracy of the proposal distribution is below the parent's.
Final selection can still lose some of this retained quality.

\paragraph{When useful proposals reduce the global gap.}
Let $\mathcal F^{\rm fb}_{t-1}$ be the construction history before
forming the next feedback batch, excluding reserved observations.
It precedes the pre-proposal history $\mathcal F^{\rm prop}_{t-1}$
of Appendix~\ref{app:robust-global-proof}; expectations here average
over both feedback formation and proposal randomness. The following result controls
the part of the global gap above a residual level $\varepsilon$.

\begin{theorem}[Generation through effective archive improvements]
\label{thm:generation-gap}
Assume a nonempty initial bank is completed almost surely. Let
$\varepsilon\in[0,1]$ and $r,\theta\in(0,1]$. Suppose that at every
pre-feedback history with $X_{t-1}>\varepsilon$,
\[
\Pr\!\left\{
\Delta J_t-\delta_{{\rm par},t-1}
\ge\theta(X_{t-1}-\varepsilon)
\;\middle|\;\mathcal F^{\rm fb}_{t-1}
\right\}\ge r.
\]
Then, for $R$ refinement slots,
\[
\mathbb E[(X_R-\varepsilon)_+]
\le(1-r\theta)^R\mathbb E[(X_0-\varepsilon)_+].
\tag{C15}\label{eq:generation-bound}
\]
\end{theorem}

On the stated event, the excess gap contracts by at least a fraction
$\theta$; archive retention makes it nonincreasing otherwise.
The success probability includes admission and evaluation failures.
Unexposed errors and early stopping enter through the residual or the
round-specific progress probabilities below.

This condition can be connected to local properties of feedback.
Consider a specified population of parent errors. A useful witness is
an item of feedback that displays correct repair evidence for those
errors, identified from its content before a proposal is generated.
Appendix~\ref{app:witness-progress} derives a lower bound on the
probability of effective archive improvement from three quantities:
the probability of exposing that witness, the expected error mass
corrected conditional on exposure, and the expected new error mass
among proposals achieving a specified amount of correction. The
resulting gain must exceed the parent deficit. Proposals outside
this useful subset may be harmful without damaging the retained
archive, including when the feedback batch itself is deterministic.

Appendix~\ref{app:global-proof} gives a complementary bound under
mean-repair conditions, accounting for parent selection on a reused
search set.

\begin{proof}[Proof of Proposition~\ref{prop:archive-progress}]
Let $V_{t-1}=\max_{p\in\mathcal C_{t-1}}J(p)$. For a valid,
fully evaluated proposal,
$\max_{p\in\mathcal C_t}J(p)=\max\{V_{t-1},J(q_t)\}$.
Deduplication does not change this maximum. A failed or skipped slot
has $q_t=p_{t-1}$ analytically, so the same equality holds without
adding a candidate. Therefore
\[
\max_{p\in\mathcal C_t}J(p)-V_{t-1}
=[J(q_t)-V_{t-1}]_+
=[\Delta J_t-\delta_{{\rm par},t-1}]_+.
\]
Subtract from $J_\Pi^*$ to obtain Equation~\ref{eq:archive-progress}.
Summing over slots gives Equation~\ref{eq:complete-identity}.
\end{proof}

\begin{proof}[Proof of Theorem~\ref{thm:generation-gap}]
Put $Y_t=(X_t-\varepsilon)_+$. If $Y_{t-1}=0$, archive
monotonicity gives $Y_t=0$. Otherwise the stated improvement event
gives $Y_t\le(1-\theta)Y_{t-1}$; off it, $Y_t\le Y_{t-1}$.
Consequently
\[
\mathbb E[Y_t\mid\mathcal F^{\rm fb}_{t-1}]
\le(1-r\theta)Y_{t-1}.
\]
Iterated conditional expectation proves
Equation~\ref{eq:generation-bound}. No independence of proposal
outcomes is used. Failed slots are included in the conditional
probability, rather than excluded from the analysis.
\end{proof}

For deterministic round-specific lower bounds $r_t\theta_t\in[0,1]$,
the same argument replaces $(1-r\theta)^R$ by
$\prod_{t=1}^R(1-r_t\theta_t)$. A known stopped slot has rate zero.
For history-dependent rates the conditional recurrence remains valid,
but their realized average cannot be substituted into this product.
The requirement applies only while the gap exceeds $\varepsilon$.
An unobservable error mode can therefore impose a nonzero residual
without contradicting the theorem.

To connect archive progress to the deployed prompt, telescope Equation~\ref{eq:archive-progress} and substitute
into Equation~\ref{eq:two-gaps} to obtain
\[
\begin{aligned}
J_\Pi^*-J(p_S^*)={}&
\underbrace{X_0}_{\text{initial global gap}}\\
&-\underbrace{\sum_{t=1}^R
[\Delta J_t-\delta_{{\rm par},t-1}]_+}_{\text{accumulated archive gains}}
+\underbrace{\Delta_{\rm sel}}_{\text{selection loss}}.
\end{aligned}
\tag{C16}\label{eq:complete-identity}
\]
Source synthesis sets the starting quality. Revisions improve it only
when they add a better candidate to the archive. Agreement-based
selection determines how much of that quality reaches deployment.
This is an identity for the original global objective, with no
assumption that the search itself finds a global maximizer.

\begin{corollary}[Complete construction gap]
\label{cor:global-gap}
Under Theorem~\ref{thm:generation-gap}, let
$d_{\rm src}\ge\mathbb E[X_0]$ and define
\[
B_{\rm arch}=\min\!\left\{1,d_{\rm src},
\varepsilon+(1-r\theta)^R\mathbb E[(X_0-\varepsilon)_+]\right\}.
\]
Suppose final selection completes almost surely and
$\mathbb E[\Delta_{\rm sel}]\le\zeta_{\rm sel}$. Then
\[
\mathbb E[J_\Pi^*-J(p_S^*)]
\le\min\{1,B_{\rm arch}+\zeta_{\rm sel}\}.
\tag{C17}\label{eq:global-repair}
\]
\end{corollary}

The selection term can be supplied by
Theorem~\ref{thm:archive-selection} under its sampling assumptions;
Appendix~\ref{app:selection-proof} gives the expectation bound. Dependence
between generation and selection is allowed. With fixed valid
progress rates, additional rounds reduce the excess generation gap
above $\varepsilon$; reducing selection error additionally requires
informative reserved evaluation and limited relative proxy bias.
The expectation of the initial excess gap is retained explicitly:
it cannot in general be replaced by \mbox{$[d_{\rm src}-\varepsilon]_+$}.

\begin{proof}[Proof of Corollary~\ref{cor:global-gap}]
Theorem~\ref{thm:generation-gap} and
$X_R\le\varepsilon+(X_R-\varepsilon)_+$ bound the expected
generation gap by
$\varepsilon+(1-r\theta)^R\mathbb E[(X_0-\varepsilon)_+]$.
Archive monotonicity also gives
$\mathbb E[X_R]\le\mathbb E[X_0]\le d_{\rm src}$, and
$X_R\le1$. Thus $\mathbb E[\Delta_{\rm gen}]\le B_{\rm arch}$.
Take expectations in Equation~\ref{eq:two-gaps}, add
$\mathbb E[\Delta_{\rm sel}]\le\zeta_{\rm sel}$, and use the
trivial upper bound of one to prove Equation~\ref{eq:global-repair}.
Linearity of expectation permits dependence between generation and
selection.
\end{proof}

\subsection{From Feedback Witnesses to Archive Improvements}
\label{app:witness-progress}

Fix a pre-feedback history $\mathcal F^{\rm fb}_{t-1}$ and its parent $p$.
Choose a measurable target set $\mathcal A_t\subseteq\{c_p=0\}$ on the
population probability space, with mass $\mu_t>0$. The set is fixed
before the proposal; it is an analytical description of errors,
not information supplied by ground-truth labels to construction.
For a proposed prompt $q_t$, write
\[
U_t=\Pr(\mathcal A_t,c_{q_t}=1),\qquad
L_t=\Pr(c_p=1,c_{q_t}=0).
\]
Thus $0\le U_t\le \mu_t$, and Equation~\ref{eq:repair-regression}
gives $\Delta J_t\ge U_t-L_t$.

Let $W_t$ be the event that the actual feedback input contains
correct repair evidence for this target set. Its criterion is
specified from the evidence content before proposal generation.
Suppose
\[
\Pr(W_t\mid\mathcal F^{\rm fb}_{t-1})\ge f_t,\qquad
\mathbb E[U_t\mid\mathcal F^{\rm fb}_{t-1},W_t]\ge \underline U_t.
\]
Dummy proposals use $q_t=p$, so they have $U_t=L_t=0$.
The correction expectation includes proposal admission and evaluation
failures. For a zero-probability witness the success lower bound is
zero and no conditional moments are invoked. The following regression
condition concerns proposals that achieve a specified amount of
target correction; other proposals can have arbitrary damage.

\begin{proposition}[Witness-supported archive improvement]
\label{prop:witness-progress}
Choose thresholds $0<\tau_t<\underline U_t\le \mu_t$ and $\ell_t>0$, and suppose
\[
\mathbb E[L_t\mid\mathcal F^{\rm fb}_{t-1},W_t,U_t\ge \tau_t]\le \overline L_t.
\]
Then, conditional on $\mathcal F^{\rm fb}_{t-1}$,
\[
\begin{aligned}
&\Pr\{\Delta J_t-\delta_{{\rm par},t-1}
\ge \tau_t-\ell_t-\delta_{{\rm par},t-1}
\mid\mathcal F^{\rm fb}_{t-1}\}\\
&\hspace{1em}\ge
f_t\frac{\underline U_t-\tau_t}{\mu_t-\tau_t}
\left[1-\frac{\overline L_t}{\ell_t}\right]_+ .
\end{aligned}
\tag{C18}\label{eq:witness-success}
\]
\end{proposition}

\begin{proof}
Condition on the history and $W_t$. Boundedness of $U_t$ implies
\[
\underline U_t\le\mathbb E[U_t\mid W_t,\mathcal F^{\rm fb}_{t-1}]
\le \tau_t+(\mu_t-\tau_t)\Pr(U_t\ge \tau_t\mid W_t,\mathcal F^{\rm fb}_{t-1}).
\]
Condition further on $U_t\ge \tau_t$, which has positive conditional
probability by the preceding bound. Markov's inequality gives
$\Pr(L_t>\ell_t\mid W_t,\mathcal F^{\rm fb}_{t-1},U_t\ge \tau_t)\le \overline L_t/\ell_t$.
Multiplying the lower bound on $U_t\ge \tau_t$ by the resulting
conditional probability of $L_t\le\ell_t$ lower-bounds their
intersection. On that intersection $\Delta J_t\ge \tau_t-\ell_t$.
Multiplying also by
$\Pr(W_t\mid\mathcal F^{\rm fb}_{t-1})\ge f_t$ proves the claim.
The regression condition describes their dependence; independence
between correction and regression is unnecessary.
\end{proof}

If the threshold net gain is at least
$\theta(X_{t-1}-\varepsilon)$ and the right-hand side is at least
$r$ whenever $X_{t-1}>\varepsilon$, the conditions of
Theorem~\ref{thm:generation-gap} follow. This implication derives a
tail guarantee from conditional repair moments and evidence exposure,
rather than assuming a tail probability separately. If only
$\mathbb E[L_t\mid\mathcal F^{\rm fb}_{t-1},W_t]\le \overline L_t^{\rm all}$
is available, without the additional condition $U_t\ge\tau_t$,
a union bound instead gives the sufficient probability
$f_t[(\underline U_t-\tau_t)/(\mu_t-\tau_t)-\overline L_t^{\rm all}/\ell_t]_+$.

A more explicit sufficient regime uses the comparison prompt and
avoidable-error set of Appendix~\ref{app:repair-conditions}.
Take $\mathcal A_t=\{c_{p^\circ}=1,c_p=0\}$, so
$\mu_t\ge X_{t-1}+\delta_{{\rm par},t-1}-\varepsilon_\circ$.
Suppose exposure has probability at least $f$, and conditional
corrected mass given exposure is at least $s \mu_t$. Choose
$0<b<a<s\le1$, and assume expected new error mass conditional on
both exposure and $U_t\ge a \mu_t$ is at most $\lambda \mu_t$. Set
$\theta=a-b$, and assume the pathwise parent deficit is at most
$\bar\delta$. Then one can take
\[
r=f\frac{s-a}{1-a}\left[1-\frac{\lambda}{b}\right]_+,\qquad
\varepsilon=\varepsilon_\circ+
\frac{1-\theta}{\theta}\bar\delta,
\tag{C19}\label{eq:witness-global-rates}
\]
provided $r>0$ and $\varepsilon<1$. Indeed, choose
$\tau_t=a \mu_t$ and $\ell_t=b \mu_t$. Their net gain above the
archive best is at least
\[
\theta \mu_t-\delta_{{\rm par},t-1}
\ge\theta(X_{t-1}-\varepsilon_\circ)
-(1-\theta)\bar\delta
=\theta(X_{t-1}-\varepsilon).
\]
Here the pathwise parent assumption is stronger than a bound on its
expectation; the two cannot be interchanged. Witnesses may cover
different predefined error modes, but the combined target set and
moment bounds must describe the actual feedback input.

The exposure, transfer, and regression conditions describe the given
synthesizer and student. Even a deterministic feedback batch can yield
mostly harmful proposals while a useful subset improves the archive;
its conditional witness probability is zero or one.

\subsection{Source Witnesses and Initial Candidate Quality}
\label{app:source-witness}

Fix a comparison prompt $p^\circ$ independently of construction, with
$\varepsilon_\circ=J_\Pi^*-J(p^\circ)$, and a measurable partition
$E_1,\ldots,E_d$ of the task-input population. For stochastic decoding,
put correctness indicators for all prompts on a fixed common probability
space with their correct marginal laws; equal prompts use the same
indicator. This coupling is an analytical device and does not change
the decoding procedure. Let $\mu_j=\Pr(E_j,c_{p^\circ}=1)$.

For an admitted initial prompt $p_i$, let $W^{\rm src}_{ij}$ be the
event that its synthesis input includes a teacher solution displaying
the rule for mode $j$. The event is defined by the source evidence,
without referring to the success of the resulting prompt. Under the
actual law of admitted candidates, let $\mathcal H^{\rm src}_{i-1}$
denote the synthesis history before drawing the evidence for candidate
$i$, including earlier attempts and admitted prompts but excluding
reserved observations. Suppose
$\Pr(W^{\rm src}_{ij}\mid \mathcal H^{\rm src}_{i-1})\ge e_j$ and
\[
\mathbb E\!\left[\Pr(c_{p_i}=1\mid E_j,c_{p^\circ}=1)
\mid \mathcal H^{\rm src}_{i-1},W^{\rm src}_{ij}\right]\ge s_j.
\]
The inner probability is on a fresh population input and decoding
draw; the outer expectation is over synthesis. Zero-mass modes can be
omitted. The exposure and transfer bounds may be required only for one
specified admitted candidate; requiring them throughout the initial
bank is a stronger sufficient condition.

\begin{proposition}[Initialization from source witnesses]
\label{prop:source-witness}
Under these conditions,
\[
\mathbb E\!\left[J_\Pi^*-\max_{p\in\mathcal P}J(p)\right]
\le\min\!\left\{1,\varepsilon_\circ+
\sum_{j=1}^d\mu_j(1-e_js_j)\right\}.
\tag{C20}\label{eq:source-bound}
\]
\end{proposition}
\begin{proof}
For the specified candidate, conditioning first on exposure shows that
its expected probability of being correct on
$E_j\cap\{c_{p^\circ}=1\}$ is at least $\mu_je_js_j$.
Thus its expected mass of errors on that set is at most
$\mu_j(1-e_js_j)$. Pointwise in the candidate,
\[
J_\Pi^*-J(p_i)
=\varepsilon_\circ+J(p^\circ)-J(p_i)
\le\varepsilon_\circ+\Pr(c_{p^\circ}=1,c_{p_i}=0).
\]
Sum over the partition, take expectations, and use that the best
initial candidate has no larger gap. Accuracy gaps are at most one.
No independence among generated candidates is used.
\end{proof}

This bound separates availability of relevant solution information
($e_j$) from its transfer into student behavior ($s_j$). For a fixed
source bank of $N_{\rm src}$ examples containing $n_j$ witnesses, a raw uniform
draw of $m$ distinct examples has exposure probability
$1-\binom{N_{\rm src}-n_j}{m}/\binom{N_{\rm src}}m$. For $m$ independent draws with witness
probability $w_j$, it is $1-(1-w_j)^m$. These formulas describe the
synthesis input, which in our configuration contains three examples;
the total source-bank size is not the per-prompt exposure count.
Admission can reweight these probabilities. If $V$ denotes admission
in a single-attempt law, then
$\Pr(W_j\mid V)=\Pr(W_j\cap V)/\Pr(V)$.
With repeated attempts and uniqueness checks, exposure must be computed
under the distribution of admitted candidates conditional on the synthesis
history. The exposure probability for an unfiltered sample of source
examples equals this quantity only when admission preserves that
probability.

Equation~\ref{eq:source-bound} makes no automatic claim that increasing
the number of candidates improves the bound: its gain depends on
additional diversity or tail information. It also retains
$\varepsilon_\circ$ explicitly. A good comparison prompt within a
restricted family does not by itself identify the optimum over $\Pi$.

\subsection{Feedback Exposure, Repair, and Regression}
\label{app:repair-conditions}

Let $\mathcal F^{\rm fb}_{t-1}$ contain all construction information used before
forming round $t$'s feedback, fixing parent $p=p_{t-1}$ but leaving the
round's unobserved proposal randomness unconditioned. Source and cached
search data may be reused arbitrarily. Define
\[
D_j(p)=E_j\cap\{c_{p^\circ}=1,c_p=0\},\qquad
d_j(p)=\Pr(D_j(p)),\qquad D_+(p)=\sum_jd_j(p).
\]
The teacher need not observe $D_j$ or ground-truth correctness.
These are population quantities for analyzing transfer.

Let $W_{tj}$ be the event that the actual feedback batch exposes an
admissible solution and student-response comparison displaying a
correct repair rule for mode $j$. It is determined before generating
the revision and cannot be defined by a successful revision afterward.
If batching is deterministic given history, its conditional probability
$f_{tj}$ is simply zero or one. Define
$R_{tj}(q)=\Pr(c_q=1\mid D_j(p))$ for $d_j(p)>0$.
The following are sufficient local conditions, required for every
pre-feedback history under consideration:
\begin{enumerate}
\item \textbf{Exposure and transfer.}
$f_{tj}=\Pr(W_{tj}\mid\mathcal F^{\rm fb}_{t-1})$, and
$\mathbb E[R_{tj}(q_t)\mid\mathcal F^{\rm fb}_{t-1},W_{tj}]\ge r_{tj}$.
Moreover,
$\sum_jd_j(p)f_{tj}r_{tj}\ge\alpha D_+(p)-u$.
For a zero-probability witness, take $r_{tj}=0$.
\item \textbf{Regression control.}
$\mathbb E[\Pr(c_p=1,c_{q_t}=0)\mid\mathcal F^{\rm fb}_{t-1}]
\le\beta D_+(p)+h$.
\end{enumerate}
Here $0\le\beta<\alpha\le1$, $u,h\ge0$, and failed or incomplete
proposals use $q_t=p$. Thus transfer rates include the probability of
actually obtaining an admitted, fully evaluated revision. They do not
condition away failures. Shared witnesses, coupled repairs across
modes, and dependence across rounds are allowed.

The error mass $u$ permits incomplete exposure of the comparison
prompt's solvable cases. The term $h$ permits collateral mistakes not
proportional to that mass. An event such as ``the teacher disagrees''
does not alone certify a useful witness: an incorrect reference may
provide misleading feedback. This effect can reduce $f_{tj}r_{tj}$ and
increase regression; candidate-dependent proxy bias affects parent and
reserved selection separately.

\begin{proposition}[One-step repair bound]
\label{prop:repair-step}
With $\eta=\alpha-\beta$ and $\nu=u+h$, these conditions imply
\[
\mathbb E[J_\Pi^*-J(q_t)\mid\mathcal F^{\rm fb}_{t-1}]
\le(1-\eta)[J_\Pi^*-J(p)]+\eta\varepsilon_\circ+\nu.
\tag{C21}\label{eq:repair-step}
\]
\end{proposition}
\begin{proof}
The disjoint sets $D_j(p)$ are all subsets of $\{c_p=0\}$. Nonnegative
correction probabilities and conditioning on $W_{tj}$ give
\[
\mathbb E[\Pr(c_p=0,c_{q_t}=1)\mid\mathcal F^{\rm fb}_{t-1}]
\ge\sum_jd_j(p)f_{tj}r_{tj}\ge\alpha D_+(p)-u.
\]
Subtract the regression bound and apply
Equation~\ref{eq:repair-regression}. This yields
$\mathbb E[J(q_t)-J(p)\mid\mathcal F^{\rm fb}_{t-1}]
\ge\eta D_+(p)-\nu$.
Since $D_+(p)\ge J(p^\circ)-J(p)
=J_\Pi^*-J(p)-\varepsilon_\circ$, rearrangement proves the claim.
This last inequality also holds when the parent is better than the
comparison prompt. No claim of pathwise accuracy improvement is used.
\end{proof}

\subsection{Adaptive Search Control and the Complete Construction Bound}
\label{app:global-proof}

\paragraph{Archive and parent invariants.}
The archive after slot $t$ is
$\mathcal C_t=\mathcal P\cup\{q_i:i\le t,\ q_i\text{ is valid
and fully evaluated}\}$. A failed, skipped, or
unevaluated proposal is represented by $q_t=p_{t-1}$; it adds no
candidate or observation. The initial parent maximizes cached search
agreement over $\mathcal P$. Induction on Equation 3 gives
\[
\widehat J_s(p_t)
=\max_{p\in\mathcal C_t}\widehat J_s(p).
\]
For this sufficient bound, each distinct prompt has a single canonical
cached search score, including any repeated occurrence. Strict
replacement and content deduplication then preserve this identity;
the archive keeps every distinct candidate at its earliest occurrence.
Expectations below are over the analyzed construction law, with
completion assumptions as stated in the respective results.

\paragraph{Controlling the search parent.}
Let $\Pi_0\subseteq\Pi$ be a finite prompt family fixed independently
of search evaluation and containing every candidate that can be
generated in the analyzed construction. Suppose the full vector of
cached agreement observations for $\Pi_0$ has independent search
groups, fixed weights, and population-matched means. For
$\delta_s\in(0,1)$ and $\epsilon_s=\sqrt{v_s\log(2|\Pi_0|/\delta_s)/2}$, weighted
Hoeffding and a union bound imply the event
\[
\mathcal E_s^{\rm conc}=\left\{\sup_{p\in\Pi_0}
|\widehat J_s(p)-\widetilde J(p)|\le\epsilon_s\right\},
\qquad\Pr((\mathcal E_s^{\rm conc})^c)\le\delta_s.
\]
Write $\rho_{\Pi_0}=\max_{p\in\Pi_0}b(p)-\min_{p\in\Pi_0}b(p)$.
For every history on $\mathcal E_s^{\rm conc}$, the cached empirical-maximizer
property gives
\[
0\le\max_{p\in\mathcal C_t}J(p)-J(p_t)
\le\rho_{\Pi_0}+2\epsilon_s.
\tag{C22}\label{eq:parent-gap}
\]
Off this event the gap is at most one. Hence one may take
\[
\zeta_{\rm par}=\min\{1,\rho_{\Pi_0}+2\epsilon_s+\delta_s\}.
\]
The family may be the full length-bounded $\Pi$; using a
smaller family requires actual candidate containment. Its optimum
never replaces $J_\Pi^*$ in the result. A very large family can make
this sufficient bound loose. A sharper adaptive generalization
argument can replace it, but a union bound over only the realized
adaptive candidates cannot.

\begin{theorem}[Generation under conditional mean repair]
\label{thm:mean-repair}
Assume initial admission completes almost surely. At every
pre-feedback history, including unsuccessful slots, suppose the
conditions of Proposition~\ref{prop:repair-step} hold with
$\eta=\alpha-\beta>0$ and $\nu=u+h$.
Let $\mathbb E[X_0]\le d_{\rm src}$ and
$\mathbb E[\delta_{{\rm par},t-1}]\le\zeta_{\rm par}$ for each slot.
With $(1-\eta)^0=1$, define
\[
\begin{aligned}
B_{\rm mean}={}&(1-\eta)^R d_{\rm src}\\
&+[1-(1-\eta)^R]\left[
\varepsilon_\circ+\frac{\nu+(1-\eta)\zeta_{\rm par}}{\eta}\right].
\end{aligned}
\]
Then
\[
\mathbb E[\Delta_{\rm gen}]
\le\min\{1,d_{\rm src},B_{\rm mean}\}.
\tag{C23}\label{eq:mean-generation}
\]
\end{theorem}

\begin{proof}[Proof of Theorem~\ref{thm:mean-repair}]
Define the archive's accuracy gap after slot $t$ and the previous
parent's deficit relative to the best candidate in $\mathcal C_{t-1}$:
\[
X_t=J_\Pi^*-\max_{p\in\mathcal C_t}J(p),\qquad
\delta_{{\rm par},t-1}=\max_{p\in\mathcal C_{t-1}}J(p)-J(p_{t-1}).
\]
They are nonnegative, and
$J_\Pi^*-J(p_{t-1})=X_{t-1}+\delta_{{\rm par},t-1}$.
An admitted $q_t$ enters $\mathcal C_t$; a dummy proposal equals the
parent already in it. Therefore $X_t\le J_\Pi^*-J(q_t)$ on every
trajectory. Apply Proposition~\ref{prop:repair-step} conditionally:
\[
\mathbb E[X_t\mid\mathcal F^{\rm fb}_{t-1}]
\le(1-\eta)(X_{t-1}+\delta_{{\rm par},t-1})
+\eta\varepsilon_\circ+\nu.
\]
Take unconditional expectations and use
$\mathbb E[\delta_{{\rm par},t-1}]\le\zeta_{\rm par}$. Iterating gives
\[
\mathbb E[X_R]\le (1-\eta)^R d_{\rm src}
+\sum_{i=0}^{R-1}(1-\eta)^i
\left[\eta\varepsilon_\circ+\nu+(1-\eta)\zeta_{\rm par}\right].
\]
Use the geometric sum and $X_R=\Delta_{\rm gen}$ to obtain
the bound $B_{\rm mean}$. Archive monotonicity also gives
$X_R\le X_0$, proving Equation~\ref{eq:mean-generation}.
For $R=0$, the sum is zero and
the bound reduces to $d_{\rm src}$. For $\eta=1$ and $R\ge1$,
it reduces to $\varepsilon_\circ+\nu$.
The unconditional parent-gap bound includes the failure probability
of $\mathcal E_s^{\rm conc}$, allowing the repair recursion to use the original
proposal law.
\end{proof}

\paragraph{Varying rates and stopping.}
Constant positive $\eta$ requires the local conditions at every
history covered by Theorem~\ref{thm:mean-repair}. A stopped or empty-feedback slot
cannot be assigned positive repair capacity without a residual term
large enough to cover its unaddressed error. For deterministic
round-specific rates $\eta_t\in[0,1]$ and residuals $\nu_t$, the same
proof gives
\[
\begin{aligned}
\mathbb E[X_R]\le{}&
\prod_{t=1}^R(1-\eta_t)d_{\rm src}\\
&+\sum_{t=1}^R\prod_{i=t+1}^R(1-\eta_i)
\left[\eta_t\varepsilon_\circ+\nu_t
+(1-\eta_t)\zeta_{{\rm par},t-1}\right].
\end{aligned}
\tag{C24}\label{eq:varying-repair}
\]
A known dummy slot satisfies $X_t=X_{t-1}$ exactly and can be omitted
from the recursion. History-dependent random rates require retaining
the conditional recursion; they cannot be replaced by their observed
average in the product above. These statements do not assume fresh
independent feedback each round.

\subsection{Connections to the Empirical Analysis}
\label{app:mechanism-tests}

The empirical comparisons examine three mechanisms in the construction
analysis: transfer from worked solutions, feedback-guided candidate
generation, and finite-sample selection. The solution-removal control
(Appendix~\ref{app:solution-removal}) measures the contribution of teacher
solutions across synthesis and refinement. The independent-synthesis
control (Appendix~\ref{app:refinement-ablation}) holds the initial bank and
number of proposal opportunities fixed, comparing feedback-guided
revisions with further synthesis from source examples. These comparisons
connect source information and feedback to the quality of generated
candidates.

The 96 actual revisions in Appendix~\ref{app:archive-diagnostic} expose
the repair--damage tradeoff: incorrect-to-correct and correct-to-incorrect
transitions measure each proposal's effects relative to its parent.
Search-score changes and parent updates show how this tradeoff interacts
with the acceptance rule. Keeping all valid, fully evaluated proposals
in the archive also permits comparison of trajectory progress with the
best candidate generated during construction.

Finally, archive-best and reserved-selected test accuracy separate
empirical candidate quality from selection loss within each realized
archive. The selection-size analysis in
Table~\ref{tab:selection-size} isolates the effect of the number of
reserved questions while holding candidate prompts and responses fixed.
Together with the common-question agreement diagnostic, these
measurements connect finite-sample ranking to the selection term in the
theory. The global guarantees above characterize construction under
their stated coverage, transfer, and progress conditions; the empirical
comparisons quantify these mechanisms in the benchmark runs.

\clearpage
\section{Extended Related Work}
\label{app:extended-related-work}

\textbf{Prompting and parameter adaptation.}
Prompting guides frozen models through context and instructions
\citep{brown2020,liu2023prompt}, including intermediate reasoning
\citep{cot,zerocot} and sequential decomposition \citep{leasttomost};
prefix and prompt tuning learn continuous task vectors
\citep{prefix,prompttuning}. K2P constructs a natural-language instruction
while keeping student parameters fixed.

\textbf{Teacher knowledge transfer.}
Knowledge distillation trains student parameters on teacher behavior
\citep{kd}, output sequences \citep{seqkd}, or reasoning
\citep{reasoningteachers,stepbystep}. Prompt-level methods instead
transfer reusable instructions: PLD \citep{pld} uses gold-conditioned
reasoning and student correctness; reasoning distillation by prompt
optimization \citep{rdpo} uses judged reasoning consistency. K2P uses
worked solutions for synthesis and refinement, and extracted answers
for label-free evaluation of a frozen student.

\textbf{Instruction optimization.}
Discrete search uses gradient-guided tokens \citep{autoprompt},
gradient-free edits \citep{grips}, or generated instructions
\citep{ape}. Textual gradients \citep{apo}, scored histories
\citep{opro}, and evolutionary populations \citep{evoprompt,promptbreeder}
guide proposals. DSPy and MIPRO optimize instructions and demonstrations
in LM programs \citep{dspy,mipro}; GEPA reflects on execution trajectories
to evolve prompts \citep{gepa}. K2P couples solution-based proposals
with student execution and agreement-guided search, followed by reserved
archive selection.

\textbf{Reasoning refinement and experience reuse.}
Self-Refine revises outputs with model feedback \citep{selfrefine},
although intrinsic self-correction can fail without external feedback
\citep{selfcorrect}. Reflexion stores verbal reflections \citep{reflexion};
ReasoningBank extracts memories from successful and failed experiences
\citep{reasoningbank}. Self-Consistency aggregates reasoning paths
\citep{sc}, and Tree of Thoughts searches within-problem reasoning states
\citep{tot}. K2P concentrates adaptation in construction and deploys
one fixed instruction across questions.

\end{document}